\documentclass{article}

\usepackage{microtype}
\usepackage{graphicx}
\usepackage{subcaption}
\usepackage{booktabs} 

\usepackage{hyperref}

\usepackage[preprint]{icml2026}

\usepackage{amsmath}
\usepackage{amssymb}
\usepackage{mathtools}
\usepackage{amsthm}
\usepackage{amsfonts} 
\usepackage{thmtools}
\usepackage{bm}
\usepackage{enumitem}
\usepackage{tcolorbox} 

\usepackage[capitalize,noabbrev]{cleveref}

\theoremstyle{plain}
\newtheorem{theorem}{Theorem}[section]
\newtheorem{proposition}[theorem]{Proposition}
\newtheorem{lemma}[theorem]{Lemma}
\newtheorem{corollary}[theorem]{Corollary}
\theoremstyle{definition}
\newtheorem{definition}[theorem]{Definition}
\newtheorem{assumption}[theorem]{Assumption}
\theoremstyle{remark}
\newtheorem{remark}[theorem]{Remark}





\usepackage[textsize=tiny]{todonotes}

\icmltitlerunning{Uniform Spectral Growth and Convergence of Muon in LoRA-Style Matrix Factorization}

\newcommand{\bA}{\mathbf{A}}
\newcommand{\bB}{\mathbf{B}}
\newcommand{\bC}{\mathbf{C}}
\newcommand{\bD}{\mathbf{D}}
\newcommand{\bE}{\mathbf{E}}

\newcommand{\bG}{\mathbf{G}}

\newcommand{\bI}{\mathbf{I}}

\newcommand{\bK}{\mathbf{K}}
\newcommand{\bL}{\mathbf{L}}
\newcommand{\bM}{\mathbf{M}}
\newcommand{\bN}{\mathbf{N}}

\newcommand{\bQ}{\mathbf{Q}}
\newcommand{\bR}{\mathbf{R}}
\newcommand{\bS}{\mathbf{S}}

\newcommand{\bU}{\mathbf{U}}
\newcommand{\bV}{\mathbf{V}}
\newcommand{\bW}{\mathbf{W}}
\newcommand{\bX}{\mathbf{X}}
\newcommand{\bY}{\mathbf{Y}}
\newcommand{\bZ}{\mathbf{Z}}

\newcommand{\bc}{\mathbf{c}}
\newcommand{\bd}{\mathbf{d}}

\newcommand{\bp}{\mathbf{p}}

\newcommand{\bu}{\mathbf{u}}
\newcommand{\bv}{\mathbf{v}}
\newcommand{\bw}{\mathbf{w}}
\newcommand{\bx}{\mathbf{x}}
\newcommand{\by}{\mathbf{y}}
\newcommand{\bz}{\mathbf{z}}

\newcommand{\cE}{\mathcal{E}} 

\newcommand{\cH}{\mathcal{H}}
\newcommand{\cL}{\mathcal{L}}

\newcommand{\cM}{\mathcal{M}}
\newcommand{\cK}{\mathcal{K}}

\newcommand{\bbR}{\mathbb{R}} 

\newcommand{\rd}{\mathrm{d}} 
\newcommand{\cC}{\mathcal{C}}

\newcommand{\cI}{\mathcal{I}}
\newcommand{\cN}{\mathcal{N}}

\newcommand{\cS}{\mathcal{S}}
\newcommand{\cT}{\mathcal{T}}
\newcommand{\cV}{\mathcal{V}}

\newcommand{\sR}{\mathbb{R}}

\DeclarePairedDelimiter{\abs}{\lvert}{\rvert}
\DeclarePairedDelimiter{\norm}{\lVert}{\rVert}

\newcommand{\roundb}[1]{\left(#1\right)}
\newcommand{\curlyb}[1]{\left\{#1\right\}}
\newcommand{\squarb}[1]{\left[#1\right]}
\newcommand{\angleb}[1]{\left\langle#1\right\rangle}
\newcommand{\normF}[1]{\norm*{#1}_\mathrm{F}}
\newcommand{\Off}{\mathrm{Off}}

\newcommand{\ddt}{\frac{\mathrm{d}}{\mathrm{d}t}}

\newcommand{\Diag}{\mathrm{Diag}}
\newcommand{\diag}{\mathrm{diag}} 
\newcommand{\bSigma}{\bm \Sigma} 

\newtcolorbox{highlight}[1][]{
    colback=gray!10,
    colframe=gray!30,
    boxrule=0.5pt,
    arc=2pt,
    leftrule=3pt,
    rightrule=3pt,
    toprule=1pt,
    bottomrule=1pt,
    #1
}

\begin{document}

\twocolumn[
  \icmltitle{Uniform Spectral Growth and Convergence of Muon \\ in LoRA-Style Matrix Factorization}



  \icmlsetsymbol{equal}{*}

  \begin{icmlauthorlist}
    \icmlauthor{Changmin Kang}{equal,kaist}
    \icmlauthor{Jihun Yun}{equal,krafton}
    \icmlauthor{Baekrok Shin}{kaist}
    \icmlauthor{Yeseul Cho}{kaist}
    \icmlauthor{Chulhee Yun}{kaist}
  \end{icmlauthorlist}

  \icmlaffiliation{kaist}{Kim Jaechul Graduate School of Artificial Intelligence, KAIST, Seoul, South Korea}
  \icmlaffiliation{krafton}{KRAFTON, Seoul, South Korea}

  \icmlcorrespondingauthor{Chulhee Yun}{chulhee.yun@kaist.ac.kr}


  \vskip 0.3in
]



\printAffiliationsAndNotice{\icmlEqualContribution}

\begin{abstract}
    Spectral gradient descent (SpecGD) orthogonalizes the matrix parameter updates and has inspired practical optimizers such as Muon. 
    They often perform well in large language model (LLM) training, but their dynamics remain poorly understood. In the low-rank adaptation (LoRA) setting, where weight updates are parameterized as a product of two low-rank factors, we find a distinctive spectral phenomenon under Muon in LoRA fine-tuning of LLMs: singular values of the LoRA product show near-uniform growth across the spectrum, despite orthogonalization being performed on the two factors separately. Motivated by this observation, we analyze spectral gradient flow (SpecGF)---a continuous-time analogue of SpecGD---in a simplified LoRA-style matrix factorization setting and prove ``equal-rate'' dynamics: all singular values grow at equal rates up to small deviations.
    Consequently, smaller singular values attain their target values earlier than larger ones, sharply contrasting with the largest-first stepwise learning observed in standard gradient flow. Moreover, we prove that SpecGF in our setting converges to global minima from almost all initializations, provided the factor norms remain bounded; with $\ell_2$ regularization, we obtain global convergence. Lastly, we corroborate our theory with experiments in the same setting.
\end{abstract}

\section{Introduction}

Classical optimization algorithms for training neural networks, such as stochastic gradient descent~\citep{robbins1951stochastic} and Adam(W)~\citep{kingma2015adam,loshchilov2018decoupled} update network parameters in a coordinate-wise manner, ignoring the structure of parameters. In contrast, several modern approaches in deep learning exploit the structure of matrix parameters, focusing on \emph{how} or \emph{what} to update during training.

One line of work follows the former direction, and the Muon optimizer~\citep{jordan2024muon} is one of representative examples. At each iteration $t$, Muon minimizes a loss function $f(\bW)$ in the weight matrix $\bW$ as
{
\setlength{\abovedisplayskip}{6pt}
\setlength{\belowdisplayskip}{6pt}
\begin{align}\label{eq:muon_update}
\begin{split}
    \bM_t &= \nabla f(\bW_t) + \mu \bM_{t-1},\\
    \bW_{t+1} &= \bW_t - \eta_t \cT(\bM_t),
\end{split}
\end{align}
}\noindent 
where $\mu \in[0, 1]$ tunes the momentum, $\eta_t$ is the step size, and $\cT$ orthogonalizes the momentum (see~\cref{eq:T_def}). In practice, Newton-Schulz iteration approximates $\cT$ by mapping most singular values of normalized $\bM_t$ to an interval $[0.7, 1.3]$. If $\mu = 0$, the update follows spectral gradient descent (SpecGD,~\citet{bernstein2024old,ravi2024implicit}).
Orthogonalized updates induce the spectrum of matrices to be more isotropic~\citep{wang2025muon,vasudeva2025muon}. By leveraging such geometric structure, Muon has gained significant attention for training large language models (LLM), often outperforming Adam(W)~\citep{liu2025muon,wang2025muon}.

Another line of work targets what parameters are updated. This category includes Low-rank adaptation (LoRA,~\citealp{hu2022lora}), which freezes the pretrained model and injects a low-rank update. LoRA parameterizes a weight update $\bW \in \sR^{m\times n}$ as a product of low-rank factor matrices $\bA\bB$ where $\bA \in \sR^{m\times r}$ and $\bB \in \sR^{r\times n}$ with $r \ll \min\{m, n\}$. This substantially reduces the number of trainable parameters while preserving strong fine-tuning performance.

Combining these two intriguing approaches modifies both the parameterization and optimizer, potentially leading to distinct training dynamics and implicit biases. While recent studies combining Muon-style optimizers with LoRA have also reported promising empirical results~\citep{ahn2025dion,anonymous2026taming}, theoretical explanations about their dynamics and implicit bias remain absent.

In \Cref{sec:empirical_observation}, we fine-tune language models using the Muon optimizer together with the LoRA parameterization to examine the dynamics, and we observe the Muon's unique feature persists even after applying LoRA. Although each factor receives its own orthogonalized updates, all spectral components of the product $\bA\bB$ evolve nearly uniformly. On the contrary, the dynamics of linear models with multiple layers trained under gradient descent exhibit incremental learning caused by saddle-to-saddle dynamics---the model learns singular values in a stepwise manner, from largest to smallest~\citep{arora2019implicit,gidel2019implicit,li2020towards}.

Motivated by this empirical observation, in subsequent sections, we deliver rigorous analyses of the dynamics induced from the orthogonalized updates under the LoRA setting. To this end, we model the matrix factorization objective induced by LoRA fine-tuning through the lens of \emph{spectral gradient flow (SpecGF)}, a continuous-time model of SpecGD.

We summarize our contributions as follows:

\noindent\textbf{Uniform Spectral Growth.} Even though SpecGF orthogonalizes the gradients of the LoRA components $\bA$ and $\bB$ rather than their product, we show that the singular values of $\bA\bB$ still evolve at a uniform rate under SpecGF~(\Cref{thm:uniform_growth_main}). The dynamics make smaller singular values converge before larger ones. This sharply contrasts with standard gradient flow which exhibits largest-first stepwise learning dynamics~\citep{arora2019implicit,gidel2019implicit}.

\noindent\textbf{Convergence Guarantee.} We prove that if SpecGF converges, it almost surely converges to global minima provided that the factor norms remain bounded~(\Cref{thm:main_specgf_conv}); with $\ell_2$ regularization, SpecGF globally converges. In addition, every global minimum is Lyapunov stable, and SpecGF converges to global minima exponentially fast~(\Cref{prop:main_conv_rate}).

\noindent\textbf{Empirical Validation.} On matrix factorization, we confirm that SpecGF exhibits uniform growth and smallest-first convergence order, in contrast to the largest-first behavior of standard gradient flow.
\section{Related Works}

\noindent\textbf{Orthogonalized Optimizers.} A line of research on the optimizers for matrix parameters, rather than vectorized ones, has been spotlighted for its empirical success. This includes Muon~\citep{jordan2024muon} and its variants, such as Dion~\citep{ahn2025dion} and MuonBP~\citep{khaled2025muonbp} that enhanced the efficiency, as well as preconditioner-equipped ones like Shampoo~\citep{gupta2018shampoo} and SOAP~\citep{vyas2025soap}. Muon normalizes the singular values of the matrix-valued updates, leading to a faster convergence and higher performance~\citep{liu2025muon,wang2025muon} compared to Adam~\citep{kingma2015adam}.

Recent studies have begun to analyze orthogonalized gradients through the lens of spectral gradient descent (SpecGD,~\citet{bernstein2024old}), which can be viewed as Muon without momentum. Specifically, \citet{fan2025implicit} study implicit bias and max-margin behavior under spectral descent, while \citet{vasudeva2025muon} analyze generalization benefits in imbalanced classification. These results are informative but limited to data assumptions.

\noindent\textbf{Low-Rank Adaptation (LoRA).} Low-rank adaptation (LoRA) is an efficient approach for fine-tuning large language models (LLM). This approach injects trainable low-rank factors into weight updates while keeping the base model fixed~\citep{hu2022lora}, motivated by the observation that fine-tuning updates often lie in a low intrinsic-dimensional subspace~\citep{aghajanyan2021intrinsic}. LoRA reduces trainable parameters from this low-rank structure without sacrificing much downstream performance.

Several recent studies report improved performance when applying Muon-style optimizers in LoRA fine-tuning~\citep{ahn2025dion,anonymous2026taming}. However, theoretical guarantees for convergence and training dynamics in this combined setting also remain largely unexplored. 
\looseness=-1

\noindent\textbf{Training Dynamics.} For deep linear models with standard gradient flow, large singular values are typically learned earlier than smaller ones~\citep{arora2019implicit,gidel2019implicit}, aligning with broader notions of spectral bias~\citep{cao2021towards,xu2019frequency}. Gradient-based methods can also exhibit saddle-to-saddle trajectories in matrix factorization problems~\citep{jacot2021saddle}.

Orthogonalized optimizers such as Muon or SpecGD induce different dynamics. \citet{vasudeva2025muon} prove that, for linear and bilinear models under classification tasks with MSE loss, every singular value of the parameters increases at the same rate until saturation. Such a uniform growth across all spectral components leads to a better generalization of SpecGD compared to gradient descent under data imbalance. However, their analysis is limited to a strong initialization such as spectral initialization~\citep{zhang2025lora}. \citet{qiu2025reparameterized} demonstrate that Muon encourages a more uniform spectrum than AdamW in the trainable matrices.
\looseness=-1
\begin{figure*}[t]
	\centering
    \begin{subfigure}[b]{0.24\textwidth}
        \includegraphics[width=\linewidth]{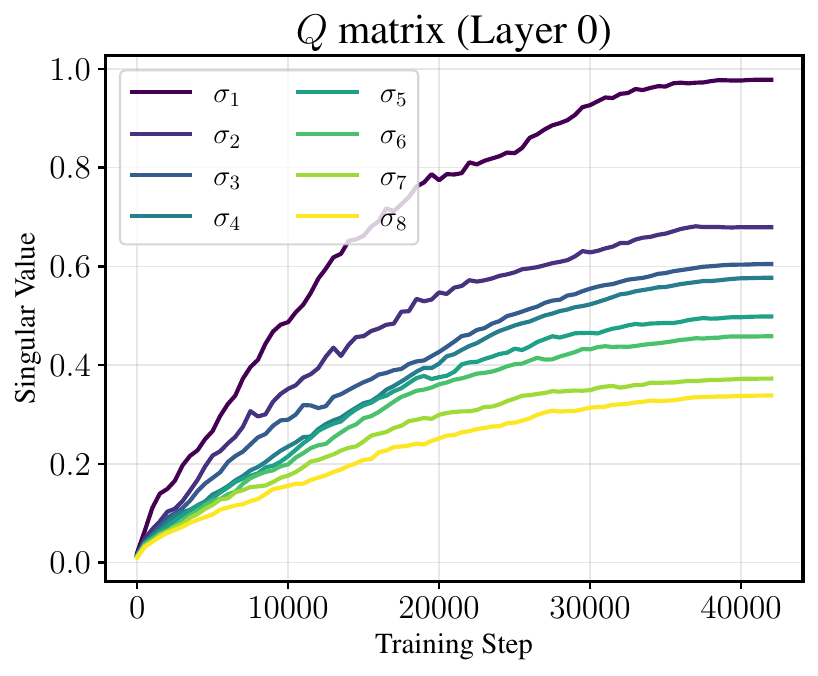}
        \caption{RoBERTa-Base with Muon}\label{fig:rbert_muon}
	\end{subfigure}
    \begin{subfigure}[b]{0.24\textwidth}
        \includegraphics[width=\linewidth]{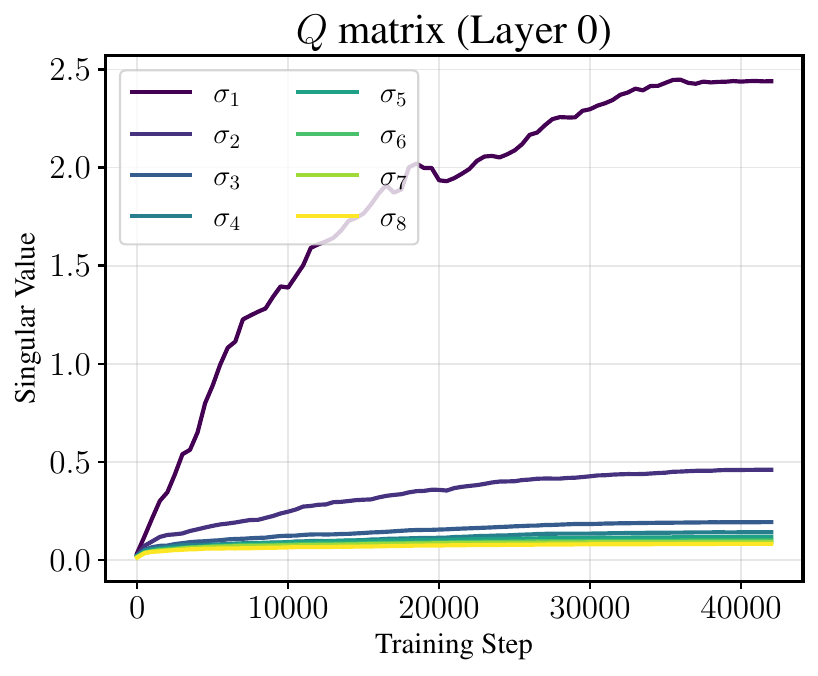}
        \caption{RoBERTa-Base with AdamW}\label{fig:rbert_adam}
    \end{subfigure}
	\begin{subfigure}[b]{0.24\textwidth}
        \includegraphics[width=\linewidth]{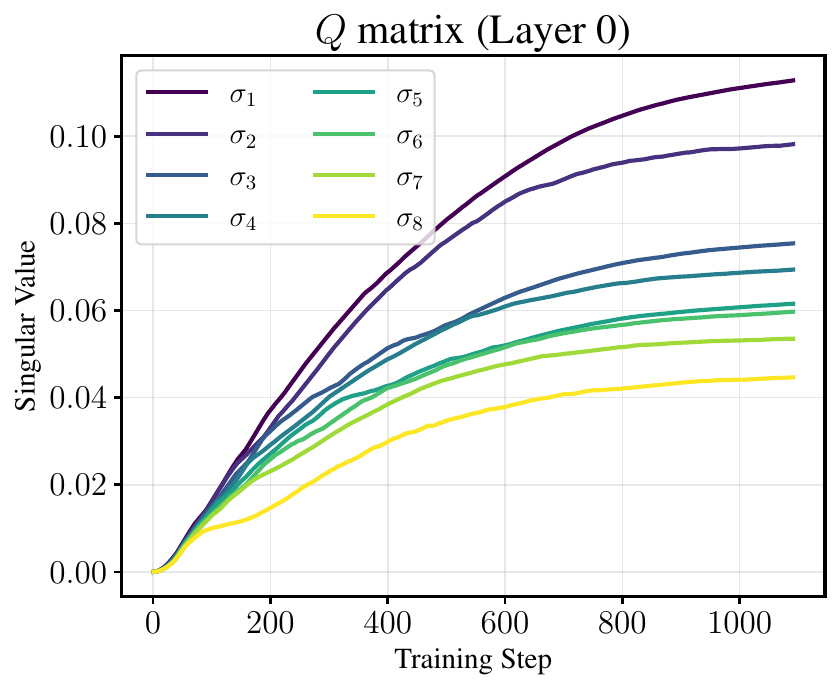}
        \caption{LLaMA-3.2-1B with Muon}\label{fig:llama_muon}
	\end{subfigure}
    \begin{subfigure}[b]{0.24\textwidth}
        \includegraphics[width=\linewidth]{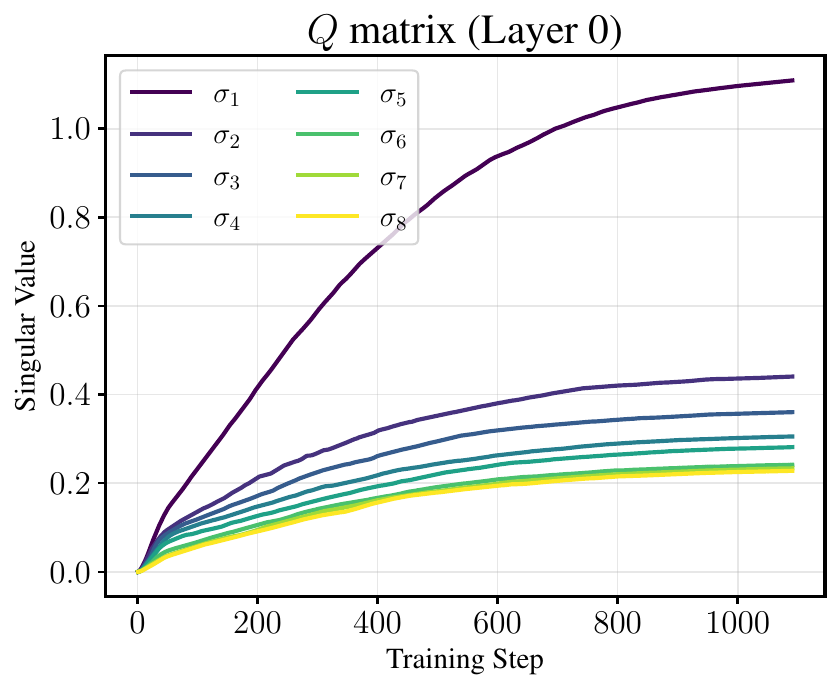}
        \caption{LLaMA-3.2-1B with AdamW}\label{fig:llama_adam}
    \end{subfigure}
    \caption{Evolution of the singular values of the LoRA $\bA\bB$ adapter applied to the query matrix in the first self-attention layer.}
	\label{fig:llm_observation}
    \vspace{-0.2in}
\end{figure*}

\section{Empirical Observation: Uniform Growth in LoRA with Muon}
\label{sec:empirical_observation}

We begin with an empirical observation that motivates our theoretical analysis. 
Given pretrained weight $\bW_0 \in \mathbb{R}^{m \times n}$, we consider the fine-tuning problem
{
\setlength{\abovedisplayskip}{6pt}
\setlength{\belowdisplayskip}{6pt}
\begin{equation}\label{eq:FT_setting}
    \min_{\bW\in\sR^{m\times n}} \ell(\bW_0 + \bW),
\end{equation}
}\noindent
where $\ell: \mathbb{R}^{m \times n} \to \mathbb{R}$ denotes the loss function and $\bW$ represents the trainable weight update.
The LoRA approach reparameterizes $\bW$ as $\bA\bB$ with low-rank factors $\bA \in \sR^{m\times r}$ and $\bB \in \sR^{r\times n}$, where $r \leq \min\{m, n\}$. The factors are initialized as $\bA(0) = \bm 0$, and $\bB(0)$ is a small random matrix~\citep{hu2022lora,hayou2024impact}.

When training LoRA adapters with Muon, we observe that the singular values of $\bA\bB$ grow considerably more uniformly across the spectrum.
This is a non-trivial phenomenon because Muon orthogonalizes the updates of $\bA$ and $\bB$ separately. 
Such a spectral growth in the product $\bA\bB$ suggests some form of alignment between the matrices $\bA$ and $\bB$, established during the training process.

\vspace{-0.1cm}
\paragraph{Experimental Setup.}

We fine-tune RoBERTa-Base~\citep{liu2019robertarobustlyoptimizedbert} on SST-2 datasets from the GLUE benchmark~\citep{wang2018glue}. For a larger-scale experiment, we train LLaMA-3.2-1B~\citep{grattafiori2024llama} on the Alpaca dataset~\citep{taori2023stanford}. For both cases, LoRA is applied to the query and value matrices with rank $r = 8$.
We compare Muon and AdamW optimizers, tracking the singular values of the LoRA product throughout training. Additional training details are described in \Cref{sec:LLM_expm}.

\vspace{-0.1cm} 
\paragraph{LLM Fine-tuning Results.} \Cref{fig:llm_observation} illustrates the evolution of the singular values of the LoRA $\bA\bB$ adapter applied to the query matrix in the first self-attention layer. \Cref{fig:llama_muon,fig:rbert_muon} show results with the Muon optimizer: all singular values exhibit near-uniform evolution, maintaining parallel trajectories throughout training.
In contrast, \Cref{fig:llama_adam,fig:rbert_adam} show that the AdamW optimizer does not necessarily attend to every spectral component equally; it may focus on larger singular values, aligning with the ``largest-first'' learning dynamics observed in standard gradient descent.

The complete singular value trajectories and the effective rank~\citep{roy2007effective} for all matrices are provided in \Cref{sec:LLM_expm}. Across all layers, Muon promotes all singular values of LoRA $\bA\bB$ adapter to evolve near-uniformly; thus, their effective rank is consistently close to the LoRA rank $r=8$. Such dynamics are not always observed for AdamW across all matrices; the consistency is degraded. This empirical evidence confirms qualitatively different learning behavior of Muon from conventional optimizers.

\vspace{-0.1cm}
\paragraph{Modeling for Theory.} The near-uniform evolution of singular values for all LoRA $\bA\bB$ adapters stems from Muon, but such interactions are highly intractable due to the intricate interactions among the components of LLM. To rigorously analyze the dynamics from Muon with LoRA, we construct and study a simplified setting that captures the LoRA approach and orthogonalized matrix updates.

\begin{figure}[t]
	\centering
    \begin{subfigure}[b]{0.23\textwidth}
        \includegraphics[width=\linewidth]{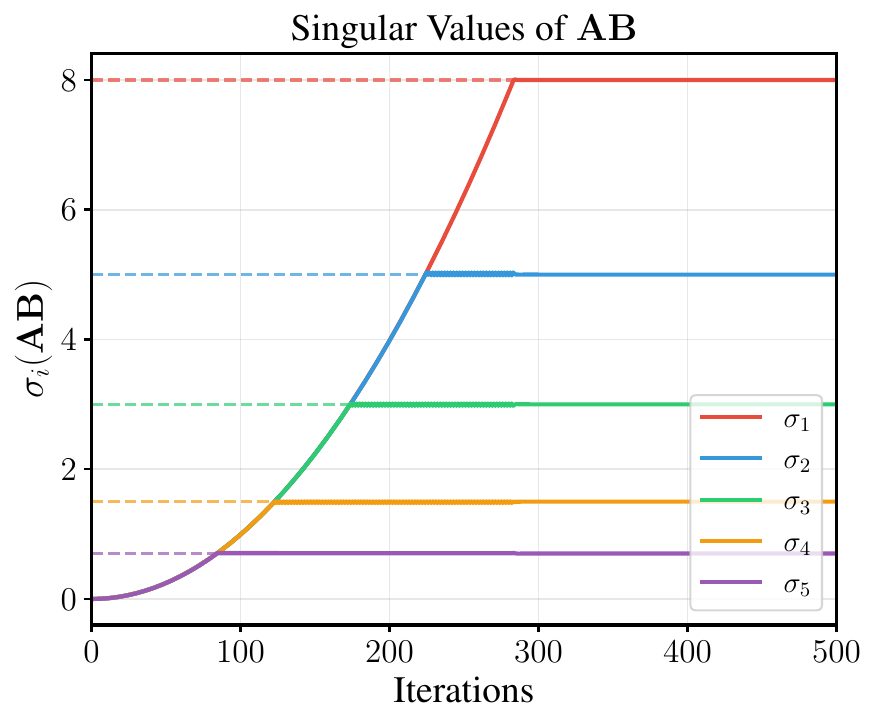}
        \caption{SpecGF}\label{fig:toy_specgf}
	\end{subfigure}
    \begin{subfigure}[b]{0.23\textwidth}
        \includegraphics[width=\linewidth]{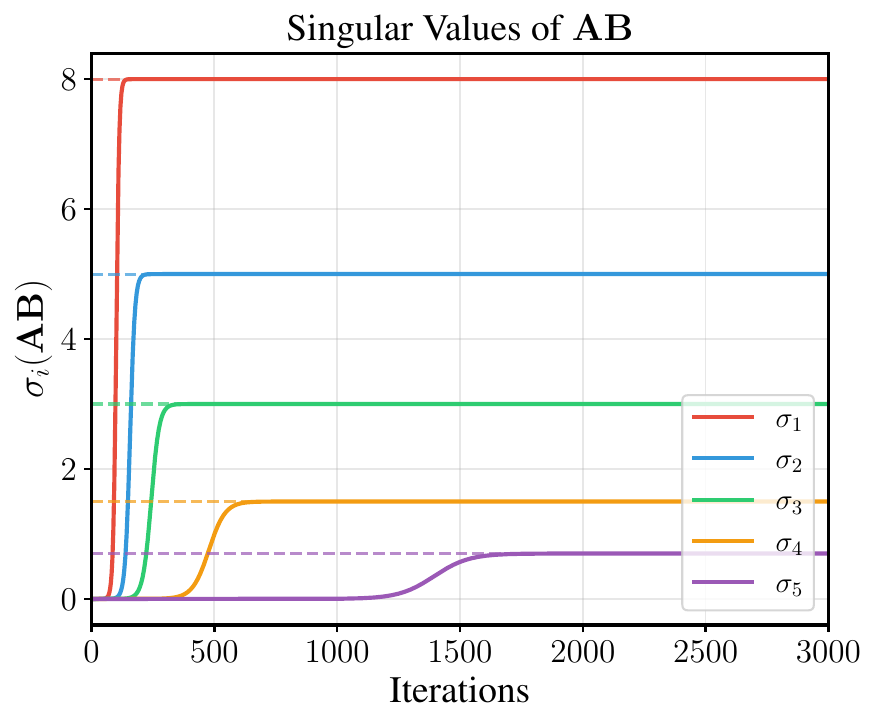}
        \caption{Vanilla GF}\label{fig:toy_gf}
	\end{subfigure}
    \caption{Comparison of singular value evolutions. While SpecGF induces uniform growth of spectrum of $\bA\bB$, GF induces the largest-first dynamics.}
	\label{fig:toy_observation}
    \vspace{-1em}
\end{figure}

Recall that \Cref{eq:FT_setting} aims to find the optimal update $\bW^\star$ during fine-tuning. A second-order Taylor expansion of $\ell$ around $\bW_0 + \bW^\star$ yields:
{
\setlength{\abovedisplayskip}{6pt}
\setlength{\belowdisplayskip}{6pt}
\begin{align*}
&\phantom{=}\ell(\bW_0 + \bW) - \ell(\bW_0 + \bW^\star)\\
&\approx \frac{1}{2}\angleb{\bW - \bW^\star,
\nabla^2\ell(\bW_0 + \bW^\star)(\bW - \bW^\star)}.
\end{align*}
}\noindent 
Therefore, the fine-tuning objective can be approximated as the minimization of the quadratic term. To theoretically understand the dynamics of LoRA, we constrain the problem by relaxing the Hessian $\nabla^2\ell(\bW_0 + \bW^\star)$ to an identity:
{
\setlength{\abovedisplayskip}{6pt}
\setlength{\belowdisplayskip}{6pt}
\begin{equation}\label{eq:obj}
    \min_{\bA, \bB} \Big\{\cL(\bA, \bB) \coloneqq \frac12 \normF{\bA\bB - \bY}^2\Big\},
\end{equation}
}\noindent
where $\bA\bB$ and $\bY\in \sR^{m\times n}$ respectively take the place of $\bW$ and $\bW^\star$. 

We train $(\bA, \bB)$ with SpecGD~(\Cref{eq:muon_update} with $\mu = 0$) with loss function in \Cref{eq:obj}. We set $(m,n,r) = (60,70,5)$ and then construct $\bY = \bU_r \bSigma \bV_r^\top$ where $\bU_r \in \mathbb{R}^{m \times r}$ and $\bV_r \in \mathbb{R}^{n \times r}$ are random orthonormal matrices, and $\bSigma = \Diag(8, 5, 3, 1.5, 0.7)$. We initialize $\bA(0) = \mathbf{0}$ and $\bB(0)$ with i.i.d. Gaussian entries of scale $\gamma = 10^{-3}$. We run SpecGD with a learning rate $\eta = 0.01$.

Notably, the singular values of $\bA\bB$ grow at a \textbf{uniform} rate; see \Cref{fig:toy_observation}. On the contrary, plain gradient descent exhibits incremental learning.
It is expected that the singular values of a single matrix uniformly grow under orthogonalized updates. However, the surprising finding is that such uniform growth is also observed for the \emph{product} $\bA\bB$, even though $\bA$ and $\bB$ are updated via orthogonalization separately.

Our observation naturally leads to the following question: 

\begin{highlight}
    \paragraph{Key Question.}
    \emph{Why do the singular values of $\bA\bB$ evolve uniformly although the gradient of $\bA$ and $\bB$ are orthogonalized separately?}
\end{highlight}

To address this question, we study \Cref{eq:obj} with spectral gradient methods. Our theoretical framework, developed in the following section, explains the observed uniform growth and the smallest-first convergence order.
\section{Theoretical Setup}

\paragraph{Notation.}
We use bold uppercase letters (e.g., $\bX$) for real matrices, bold lowercase letters (e.g., $\bx$) for real vectors, and lowercase letters (e.g., $x$) for real scalars.
Let $\mathbf{1}\in\bbR^{r}$ denote the all-ones vector and $\bI_d$ the $d\times d$ identity matrix.
We write $\mathsf{GL}(r)$ for the set of invertible $r\times r$ real matrices, and $\Diag:\sR^{d}\to\sR^{d\times d}$ for the diagonal operator. We denote the Frobenius, spectral, and nuclear norms by $\normF{\cdot}$, $\norm{\cdot}_2$, and $\norm{\cdot}_*$, respectively. For matrices $\bM$ and $\bN$, we write $\angleb{\bM,\bN}\coloneqq\mathrm{Tr}(\bM\bN^\top)$ and denote the pseudoinverse of $\bM$ by $\bM^\dagger$.
The $k$-th largest singular value of $\bM$ is denoted by $\sigma_k(\bM)$, and $[k]\coloneqq\{1,\dots,k\}$.

\subsection{Problem Setup}

To align with the LoRA setting, we assume $r \leq \mathrm{rank}(\bY)$ to solve \Cref{eq:obj}. If $\bA$ and $\bB$ depend on time $t$, we may write $\cL(t) \coloneqq \cL(\bA(t), \bB(t))$. 
We initialize each factor as $\bA(0)=\bm{0}$ and $\bB(0)$ at random, which coincides with the standard LoRA initialization where each entry of $\bB(0)$ is sampled from a Gaussian distribution~\citep{hu2022lora,hayou2024impact}.

We analyze the continuous-time analog of SpecGD, \textit{spectral gradient flow (SpecGF)}, defined by:
{
\setlength{\abovedisplayskip}{6pt}
\setlength{\belowdisplayskip}{6pt}
\begin{align}\label{eq:AB_update}
\begin{split}
    \dot{\bA}(t) &= -\cT\roundb{\nabla_\bA \cL(t)}\\ &= -\cT\roundb{(\bA(t)\bB(t) - \bY)\bB(t)^\top},\\
    \dot{\bB}(t) &= -\cT\roundb{\nabla_\bB \cL(t)}\\ &= -\cT\roundb{\bA(t)^\top(\bA(t)\bB(t) - \bY)},
\end{split}
\end{align}
}\noindent 
where $\cT$ is an orthogonalization operator applied to the matrix gradients. For a matrix $\bM \in \sR^{m\times n}$ with compact SVD $\bU_\bM\bS_\bM\bV_\bM^\top$ and $\bS_\bM \coloneqq \Diag\{\sigma_1, \ldots, \sigma_k\} \in \bbR^{k \times k}$, the exact orthogonalization is written as
\begin{equation}\label{eq:T_def}
    \cT(\bM) \coloneqq \big((\bM\bM^\top)^\dagger\big)^{1/2}\!\bM = \bU_\bM\bV_\bM^\top.
\end{equation}
Notice that $\cT$ normalizes all the nonzero singular values to $1$. However, $\cT$ is not analytic at rank-deficient points since it involves discontinuous singular-vector selections. 
To obtain an analytic vector field, we introduce a smoothed operator $\cT_\beta$ with $\beta > 0$:
{
\setlength{\abovedisplayskip}{6pt}
\setlength{\belowdisplayskip}{6pt}
\begin{equation}\label{eq:Tb_def}
    \cT_\beta(\bM) \coloneqq \roundb{\bM\bM^\top + \beta \bI_m}^{-1/2}\!\bM.
\end{equation}
}\noindent 

Adding $\beta \bI$ makes the inverse square root well-defined, and makes SpecGF analytic everywhere; see \Cref{cor:T_analytic} for details. We can write $\cT_\beta(\bM)$ as
\begin{equation*}
    \cT_\beta(\bM) = \bU_\bM\bS_\beta\bV_\bM^\top, \; \bS_\beta \coloneqq \Diag\left\{ \frac{\sigma_i}{\sqrt{\sigma_i^2+\beta}} \right\}_{i\in[k]}.
\end{equation*}
Thus, $\cT_\beta$ maps all positive inputs to values near $1$, similar to NS iterations, and boils down to $\cT$ as $\beta \to 0$.
Throughout the paper, ``SpecGF with $\cN$'' refers to the flow where both $\dot\bA$ and $\dot\bB$ use the operator $\cN \in \{\cT, \cT_{\beta}\}$ for orthogonalization.

We use the compact SVD of $\bY$. Let $r^\star = \mathrm{rank}(\bY)$. Then $\bY = \bU_{r^\star}\bm\Sigma\bV_{r^\star}^\top$,
where $\bm \Sigma \in \bbR^{r^\star \times r^\star}$ contains the nonzero singular values in non-increasing order and $\bU_{r^\star}$, $\bV_{r^\star}$ consist of the corresponding singular vectors.
\section{Uniform Growth of Singular Values}\label{sec:uniform_growth_sng}

We analyze how the singular values of $\bW(t) = \bA(t)\bB(t)$ evolve from SpecGF with $\cT_\beta$ under LoRA initialization. A key finding is that, all active singular values grow at the same rate, a phenomenon we call equal-rate dynamics. This synchronized behavior implies that the smallest singular value reaches its target first.

\subsection{Alignment Yields Decoupled Dynamics}

We start with the simple case---rank-1 $\bY$ and $r=1$---that leads to the decoupling of the dynamics of the product $\bA(t)\bB(t)$ into those of scalars. The matrices reduce to
\begin{equation*}
    \bY = \sigma \bu\bv^\top, \quad \bA \in \sR^{m\times 1}, \quad \bB \in \sR^{1\times n},
\end{equation*}
where $\bu$ and $\bv$ are unit vectors in $\sR^m$ and $\sR^n$, respectively. Given a random unit vector $\bw$ in $\sR^n$ and small $\gamma > 0$, LoRA training starts at $\bA(0) = \bm 0, \bB(0) = \gamma \bw^\top.$
Notice that $\angleb{\bv, \bw} \ne 0$ almost surely. This simple case reduces the dynamics of two vectors to those of scalars:

\begin{theorem}[Informal]\label{thm:main_sync_rk1}
    Assume $\angleb{\bv, \bw} \ne 0$ and decompose $\bw$ into $\bv$ and $\bz$ with $\angleb{\bv, \bz} = 0$. Then, for all $t \geq 0$,
    \begin{equation*}
        \bA(t) = a(t)\bu, \quad \bB(t) = b(t)\bv^\top + c(t)\bz^\top.
    \end{equation*}
    If $\gamma > 0$ is sufficiently small, then under SpecGF with $\cT_\beta$,
    \begin{equation*}
        a(t)b(t) \to \sigma, \quad c(t) \to 0.
    \end{equation*}
    Moreover, as $\gamma \to 0$, both $|a(t) - b(t)|$ and $\abs{\dot a(t) - \dot b(t)}$ vanish to 0 at any $t \geq 0$.
\end{theorem}

For the proof of \Cref{thm:main_sync_rk1}, see \Cref{sec:sng_evo_rk1}. If the initialization scale is small, the decoupling from $\bA(t)\bB(t)$ to $a(t), b(t)$, and $c(t)$ results in both the convergence to global minima and a similar rate increment in \emph{both} factors.

This decoupling extends to general matrices under SpecGF with spectral initialization~\citep{zhang2025lora}, which aligns the singular spaces of $\bA(0)$ and $\bB(0)$ with that of $\bY$. Writing $\bY=\bU_{r^\star}\bSigma\bV_{r^\star}^\top$ and letting $\bU_r,\bV_r$ be the first $r$ columns of $\bU_{r^\star},\bV_{r^\star}$, the initialization takes $\bA(0) = \bU_r\bm\Sigma_{\bA}\bQ^\top$ and $\bB(0) = \bQ\bm\Sigma_{\bB}\bV_r^\top$, where $\bm\Sigma_{\bA},\bm\Sigma_{\bB}$ are nonnegative diagonal and $\bQ\in\mathbb{R}^{r\times r}$ is orthogonal. The dynamics then decouple into $r$ independent pairs of singular values $(\sigma_i(\bA(t)),\sigma_i(\bB(t)))$. If $\bm\Sigma_{\bA}$ and $\bm\Sigma_{\bB}$ are sufficiently small, SpecGF with $\mathcal{T}_\beta$ keeps both $|\sigma_i(\bA(t))-\sigma_i(\bB(t))|$ and $|\dot\sigma_i(\bA(t))-\dot\sigma_i(\bB(t))|$ small for all $i$; see \Cref{sec:spec_init}.

\subsection{General case: Approximate as Near-Diagonal}\label{sec:general_approx}

The analyses in the decoupled case utilize alignments between LoRA factors, $\bA(t)$ and $\bB(t)$. However, such alignments are not applicable to \Cref{eq:obj} for LoRA initialization.
In this section, we consider the general case without alignments at initialization. Despite the lack of alignment, we show that under small initialization, all singular values of $\bA(t)\bB(t)$ tend to grow at nearly the same rate, even though $\bA$ and $\bB$ are updated via separately orthogonalized gradients. 

Our approach introduces a coordinate system aligned with the target $\bY$, and we argue that orthogonalization largely separates the growth speed from the current singular-value scale, leading to near-uniform growth across singular values. 

\paragraph{Core Variables.} We focus on the case $r = r_\star$ for simplicity; the case $r < r^\star$ is discussed in Appendix \Cref{sec:underparameterized}. Assume $\bSigma = \Diag(\sigma_1, \ldots, \sigma_r)$ and $\sigma_1 > \cdots > \sigma_r > 0$ with $\sigma_i = O(1)$.
Let $\bV_\perp \in \sR^{n \times (n-r)}$ denote the orthogonal complement of $\bV$ (full right singular vectors), so that $[\bV_r \mid \bV_\perp]$ is orthonormal. We define the \emph{core variables}:
\begin{align*}
    \big(\bX(t), \bZ(t), \bZ_\perp(t)\big) \coloneqq \big(\bU_r^\top \bA(t), \bB(t)\bV_r, \bB(t)\bV_\perp\big).
\end{align*}
The core product is $\bG(t) \coloneqq \bX(t)\bZ(t)$, and under $\bA(0) = \mathbf{0}$, we have $\bA(t) = \bU_r\bX(t)$ for all $t \ge 0$. We denote the diagonal entries of $\bG(t)$ by $d_i(t) := [\bG(t)]_{ii}$ for all $i \in [r]$. Further, we let $d_{\min}(t) \coloneqq \min_{i\in[r]}d_i(t)$. For a square matrix $\bM$, we write $\Off(\bM) \coloneqq \bM \odot (\mathbf{1} \mathbf{1}^\top - \bI_r)$ where $\odot$ denotes the Hadamard (element-wise) product.

\paragraph{Key Concepts.} Under the reparametrization, we introduce two tolerances: (i) the \emph{alignment tolerance} $\delta > 0$; and (ii) the \emph{target tolerance} $\varepsilon > 0$. 
\begin{itemize}[leftmargin=4mm,itemsep=2pt,topsep=0pt,partopsep=0pt]
    \item \textbf{Alignment tolerance $\delta > 0$:} We define $\delta$-alignment by 
    \begin{align}\label{eq:alignment}
        \normF{\Off(\bG(t))} \le \delta \text{ and } \normF{\bX(t)\bZ_\perp(t)} \le \delta,
    \end{align}
    i.e., $\bG(t)$ is \emph{near-diagonal} for a sufficiently small $\delta > 0$. 
    \item \textbf{Target tolerance $\varepsilon > 0$:} We define $d_i(t) := [\bG(t)]_{ii}$, $e_i(t) := d_i(t) - \sigma_i$. The active set is defined by
    \begin{align}\label{eq:active_set}
        \cI_\varepsilon(t) := \{i \in [r] : |e_i(t)| > \varepsilon\}.
    \end{align}
\end{itemize}
We call each index $i \in [r]$ a \emph{mode}, corresponding to the $i$-th singular value of $\bA\bB$. A mode $i$ is \emph{active} at time $t$ if $i \in \cI_\varepsilon(t)$, meaning it has not yet converged to its target $\sigma_i$. Note that, given a convergence to a global minima~(\Cref{sec:conv_analysis}), for any $\delta > 0$ there exists a finite $t_0 > 0$ such that $\delta$-alignment holds for all $t \ge t_0$ since $\cL = \tfrac{1}{2}(\|\bG(t) - \bSigma\|_F^2 + \|\bX\bZ_\perp\|_F^2)$. 

\paragraph{Why Near-Diagonal?}  
When $\bG(t)$ is diagonal, its diagonal entries $d_i(t)$ exactly correspond to the singular values of $\bA(t)\bB(t)$, greatly simplifying the analysis. Although $\bG(t)$ is not exactly diagonal in general, under $\delta$-alignment \eqref{eq:alignment} the diagonal entries can still approximate the singular value well. This is guaranteed by the following lemma.

\begin{lemma}[Diagonal approximation]\label{lem:diagonal_approx_main}
Under $\delta$-alignment \eqref{eq:alignment}, it follows that $|\sigma_{i}(\bA(t)\bB(t)) - d_i(t)| = O(\delta)$, for all $i \in [r]$, where $\sigma_i(\cdot)$ denotes the $i$-th singular value.
\end{lemma}

In other words, $d_i(t)$ approximates the singular value of $\sigma_i(\bA(t)\bB(t))$ up to $O(\delta)$ error; see \Cref{lem:diagonal_approx_app} for the proof of \Cref{lem:diagonal_approx_main}. Hence, it suffices to analyze the dynamics of $d_i(t)$ to understand the singular values of $\bA(t)\bB(t)$. However, under LoRA initialization $\bA(0) = \bm 0$, the gradient $\nabla_\bB \cL$ is very close to zero, thus $\cT_\beta$ operates in its near-zero regime, where the update is dominated by the $\beta$-regularization. We therefore introduce a short time $\tau > 0$ such that the diagonal entries become positive enough. 

\begin{lemma}[Initial growth]\label{lem:initial_growth_main}
Let $\tau \asymp \frac{\sigma_r}{\sigma_1 \sqrt{r}}$ and $\tau \gamma < \sigma_r - \varepsilon$ holds for a small $\varepsilon$ such that $\varepsilon \ll \sigma_r$. With probability at least $1 - e^{-cr}$ for some absolute constant $c > 0$, it follows that $d_{\min}(\tau) \ge d_0 \asymp \tau \gamma > 0$. Moreover, all modes are active at $\tau$: $\cI_\varepsilon(\tau) = [r]$ and $d_i(\tau) < \sigma_i$ for all $i$.
\end{lemma}

By \Cref{lem:initial_growth_main}, we are now able to analyze the training dynamics starting at $t = \tau$, where the active set equals $[r]$. Since this regime corresponds to the early training, focusing on $t \ge \tau$ incurs no loss of generality in characterizing the overall training behaviors; see \Cref{lem:initial_growth_app} for details. 

Similar to time $\tau$, we define a termination time as $T := \inf\{t \ge \tau : \cI_\varepsilon(t) = \emptyset\}$. Note that if SpecGF converges to a global minimum (cf. \Cref{prop:main_conv_rate}), the termination time $T$ is necessarily finite.

We prove uniform growth through the following steps:
\begin{itemize}[leftmargin=4mm,itemsep=2pt,topsep=0pt,partopsep=0pt]
    \item \textbf{Alignment from Small Initialization} (\Cref{prop:alignment_main}). Under Gaussian initialization with small $\gamma$, the $\delta$-alignment (near-diagonal) holds throughout training for $\delta = O(\gamma)$.
    \item \textbf{Square-Root Dynamics} (\Cref{lem:bootstrap_main} and \Cref{lem:sqrt_dynamics_main}). In the near-diagonal regime with $d_{\min}(t) > 0$, the square-root coordinates $s_i(t) = \sqrt{d_i(t)}$ evolve at approximately unit rate: $\dot{s}_i (t) \approx 1$.
    \item \textbf{Uniform Growth} (\Cref{thm:uniform_growth_main}). All active modes grow at the same rate in $\sqrt{\cdot}$-coordinates, leading to the ``equal-rate'' learning.
\end{itemize}

\subsubsection{Alignment from Small Initialization}

A key insight is that under small initialization, the trajectory satisfies $\delta$-alignment throughout training with $\delta = O(\gamma)$.

\begin{proposition}[High-probability alignment]\label{prop:alignment_main}
Under Gaussian initialization $\bA(0) = \mathbf{0}$, $\bB(0) = \gamma \bN$ with $\bN_{ij} \stackrel{\mathrm{i.i.d.}}{\sim} \mathcal{N}(0,1)$, the following holds on $[0, T]$ with probability at least $1 - e^{-cr^2} - e^{-crn}$ for some absolute constant $c > 0$,
{
\setlength{\abovedisplayskip}{6pt}
\setlength{\belowdisplayskip}{4pt}
\begin{align*}
    \sup_{t \le T} \normF{\Off(\bG(t))} = O(\gamma), ~~
    \sup_{t \le T} \normF{\bX(t)\bZ_\perp(t)} = O(\gamma).
\end{align*}
}
\end{proposition}

This proposition mitigates the effect of off-diagonal interactions, which allows us to simplify the subsequent analysis. In particular, given a sufficiently small initialization scale $\gamma > 0$, we can directly use the diagonal entries $d_i(t)$ as surrogates of the singular values of $\bA(t)\bB(t)$ by \Cref{lem:diagonal_approx_main}. See \Cref{prop:alignment_O_gamma_app} for the proof of \Cref{prop:alignment_main}.

\subsubsection{Square-Root Dynamics}

Our analysis of uniform growth requires that $d_{\min}(t)$ remain bounded away from $0$ along the trajectory. The following lemma guarantees the lower bound for all $t \ge \tau$.

\begin{lemma}[Persistence of Non-degeneracy]\label{lem:bootstrap_main}
For a sufficiently small tolerance such that $\varepsilon \ll \sigma_r$, $d_{\min}(t) \ge d_0 > 0$ holds for all $t \in [\tau, T]$.
\end{lemma}

See \Cref{lem:bootstrap_persistence_app} for details. Given $d_{\min}(t) > 0$, we can analyze the evolution of $s_i(t) = \sqrt{d_i(t)}$.

\begin{lemma}[Square-root dynamics]\label{lem:sqrt_dynamics_main}
Suppose $d_{\min}(t) \ge d_0 > 0$ and $\beta = O(\varepsilon^3)$. Then for each active mode $i \in \cI_\varepsilon(t)$ with $e_i(t) \le -\varepsilon$, the square-root coordinate $s_i(t) = \sqrt{d_i(t)}$ satisfies $\dot{s}_i(t) = 1 + O(\varepsilon^{1/4})$.
\end{lemma}

$\cT_\beta$ decouples the speed from the current value $d_i$, unlike plain GF where $\dot{d}_i \propto e_i$; see \Cref{lem:sqrt_dynamics_app} for details. 

\subsubsection{Main Result: Uniform Growth}

Putting these ingredients together, we show that all active modes grow uniformly.

\begin{theorem}[Uniform growth]\label{thm:uniform_growth_main}
For a target tolerance such that $\varepsilon \ll \sigma_r$, consider SpecGF with $\cT_\beta$ where $\beta = O(\varepsilon^3)$ under initialization $\bA(0) = \mathbf{0}$, $\bB(0) = \gamma \bN$ where $\gamma = O(\varepsilon^{1/2})$ and $\bN_{ij} \overset{\text{i.i.d}}{\sim} \mathcal{N}(0, 1)$.
Then for all $t \in [\tau, T]$ and all active modes $i \in \cI_\varepsilon(t)$ with $e_i(t) \le -\varepsilon$, the following holds with probability at least $1 - e^{-cr}$ for some absolute constant $c > 0$: 
\begin{enumerate}[label=\textnormal{(\roman*)},leftmargin=4mm,itemsep=2pt,topsep=0pt,partopsep=0pt]
    \item \textbf{Approximate unit speed:} $\displaystyle \frac{\rd}{\rd t}\sqrt{d_i(t)} = 1 + O(\varepsilon^{1/4})$.
    \item \textbf{Uniform growth:} For any two active modes $i, j \in \cI_\varepsilon(t)$,
    \begin{align*}
    \abs*{\frac{\rd}{\rd t}\sqrt{d_i(t)} - \frac{\rd}{\rd t}\sqrt{d_j(t)}} = O(\varepsilon^{1/4}).
    \end{align*}
\end{enumerate}
\end{theorem}

\begin{corollary}[Smallest singular value learns first]\label{cor:smallest_first_main}
All $\sqrt{d_i(t)}$ trajectories are nearly parallel. Since all modes start near zero and $\sigma_r$ requires the least growth, the smallest singular value reaches its target first.
\end{corollary}

Our analysis naturally extends to the underparameterized setting $r < r^\star$; the proofs largely follow those of \Cref{thm:uniform_growth_main}. We refer to \Cref{sec:uniform_growth} and \Cref{sec:underparameterized} for details.

\paragraph{Implication.} The uniform growth of all spectral components from SpecGF could lead to a faster convergence. Standard GF learns the spectral components one at a time by sequentially visiting the vicinity of saddle points of increasing rank~\citep{jacot2021saddle,zhang2025saddle}. Moreover, the smaller the initialization scale is, the longer it takes to escape each saddle. On the contrary, SpecGF promotes the evolution of every component. Hence, the time to learn all components essentially equals to that to learn the component with the largest singular value. See \Cref{fig:sval}.

However, learning all components uniformly is not always beneficial. When the underlying solution that generalizes well has rank smaller than the LoRA rank, i.e. $r^* < r$, uniform growth across all $r$ singular values may be suboptimal, as it can amplify unnecessary components that are potentially harmful to generalization.

We additionally state that the spectral initialization is subsumed in the setting of \Cref{sec:general_approx}, as the core product $\bG$ is always diagonal, thus the dynamics is under 0-alignment.
\section{Convergence Analysis of SpecGF}\label{sec:conv_analysis}

In this section, we guarantee convergence for SpecGF with the smoothed operator $\cT_\beta$. Toward this, we begin by establishing the well-posedness/uniqueness of SpecGF solutions in the analytic regime; see \Cref{prop:specgf_analytic} for details.
\begin{proposition}[Analyticity of $\cT$ and $\cT_\beta$]
    SpecGF with $\cT_\beta$ is analytic everywhere. SpecGF with $\cT$ is analytic on the set of full-rank matrices. Moreover, for both cases, the solution uniquely exists for all $t \geq 0$ where they are analytic.
\end{proposition}
Since the vector field is analytic (hence locally Lipschitz) on its domain, the solution under SpecGF with $\cT_\beta$ is unique. Moreover, because $\normF{\cT(\cdot)} \le \sqrt{r}$ and $\normF{\cT_\beta(\cdot)} \le \sqrt{r}$, the dynamics have bounded speed, implying that solutions do not blow up in finite time and globally exist. 

\begin{theorem}\label{thm:main_specgf_conv}
    For almost all initial values $\bA(0)$ and $\bB(0)$, SpecGF with $\cT_\beta$ either converges to a global minimum of $\cL(\bA, \bB)$ or $\normF{\bA(t)}^2 + \normF{\bB(t)}^2$ diverges to infinity.
\end{theorem}

Whether SpecGF does not diverge to infinity in norm is crucial to claim the convergence of SpecGF. This is true for standard GF, as each factor is uniformly bounded from a time-invariant quantity, $\bA(t)^\top\bA(t) - \bB(t)\bB(t)^\top$~\citep{arora2019implicit,bah2022learning}. However, such invariance for SpecGF are unknown---$\bA(t)^\top\bA(t) - \bB(t)\bB(t)^\top$ could fail to be invariant for SpecGF; see \Cref{prop:AAT-BTB_not_inv}---thus whether SpecGF globally converges to a critical point is \textit{inconclusive}. We conjecture that there are no polynomial invariants in terms of $\bA(t)$ and $\bB(t)$ for SpecGF; if there were, then it should translate to the scalar case, $m=n=1$. However, every invariant polynomial in terms of $\bA(t)$ and $\bB(t)$ for the scalar case is identically constant; see \Cref{prop:no_poly_invar_Tb,prop:no_poly_invar_T} for details. Nonetheless, assuming convergence, we characterize which critical points SpecGF essentially evolves to.

Our convergence result relies on the strict-saddle geometry of the squared-loss factorization objective: a critical point $x$ of a twice continuously differentiable function $f$ is a \textit{strict saddle point} if the Hessian of $f$ at $x$ has a negative eigenvalue.
According to~\citet{kawaguchi2016deep,bah2022learning}, every critical point of $\cL$ is either a global minimum or a strict saddle point. Moreover, every global minimum corresponds to the best rank-$r$ approximation of $\bY$. With the fact that the set of initial points where SpecGF with $\cT_\beta$ converges to a strict saddle point has Lebesgue measure zero (\Cref{thm:SpecGF_avoids_SSP}), SpecGF with $\cT_\beta$ almost surely converges to the \textbf{global} minima or diverges to infinity in norm.

We prove \Cref{thm:main_specgf_conv} by showing that SpecGF cannot wander around critical points and it avoids every ``bad'' critical points almost surely. We begin our argument with monotonic decrease of the loss along SpecGF trajectories. 
\begin{lemma}\label{lem:SpecGF_dec}
    For SpecGF with $\cT$ and SpecGF with $\cT_\beta$, 
    \begin{align*}
        \frac{d \cL(t)}{\mathrm{d}t} \le 0.
    \end{align*}
    The equality holds if and only if $\nabla_\bA \cL(t) = \nabla_\bB \cL(t) = 0$.
\end{lemma}

See \Cref{lem:SpecGF_T_dec,lem:SpecGF_Tb_dec} for the proof of \Cref{lem:SpecGF_dec}. Using this monotoncity of the loss, \Cref{prop:main_infdiv_conv} states that SpecGF should either converge or diverge to infinity in norm; see \Cref{prop:infdiv_conv} for detailed proof.

\begin{proposition}\label{prop:main_infdiv_conv}
    Under SpecGF with $\cT_\beta$, we have
    {
    \setlength{\abovedisplayskip}{6pt}
    \setlength{\belowdisplayskip}{6pt}
    \begin{equation}\label{eq:main_F_dec}
        \frac{\mathrm{d}\cL(t)}{\mathrm{d}t} \leq -c(r)\|\nabla \cL(t)\| \|(\dot\bA(t), \dot\bB(t))\|,
    \end{equation}}
    where $c(r)$ is a constant depending on $r$. Moreover, $(\bA(t), \bB(t))$ either converges or diverges to infinity in norm.
\end{proposition}

\begin{proof}[Proof sketch of \Cref{prop:main_infdiv_conv}]

    The computation of $c(r)$ is straightforward; for SpecGF with $\cT_\beta$, $c(r) = \frac{1}{2r}$. We focus on the part that SpecGF either converges or diverges to infinity in norm by using \L{}ojasiewicz inequality, which controls the rate of convergence for gradient-based dynamics under analytic objectives~\citep{lojasiewicz1965ensembles}.

    \begin{lemma}[\L{}ojasiewicz]\label{lem:main_Loja_inequ}
        Let $F:U \to \sR$ be a real analytic function on an open set $U\subset\sR^d$. Then, for every critical point $\bp\in U$ of $F$, there are a possibly smaller neighborhood $W$ of $\bp$, constants $C > 0$ and $b \in [\frac12, 1)$ such that for all $\bx \in W$, $|F(\bx) - F(\bp)|^b \leq C\|\nabla F(\bx)\|$.
    \end{lemma}
    
    If the SpecGF iteration does not diverge to infinity in norm, it should have an accumulation point $(\bar\bA, \bar\bB)$; we show $(\bA(t), \bB(t)) \to (\bar\bA, \bar\bB)$. By \Cref{eq:main_F_dec} and continuity of $\cL$, $\cL(t)$ decreases to $\cL(\bar\bA, \bar\bB)$. Once $(\bA(t), \bB(t))$ is arbitrarily close to $(\bar\bA, \bar\bB)$, it remains in that vicinity forever by \Cref{eq:main_F_dec} and \Cref{lem:main_Loja_inequ}.
\end{proof}

\paragraph{Stability of Global Minima.} We state that every global minimum is Lyapunov stable under SpecGF with $\cT_\beta$. Lyapunov stability is standard for describing whether trajectories that start near an equilibrium remain near it over time.

\begin{definition}[Lyapunov stability]\label{def:main_Lyap_stab}
    A critical point $\bar \bx \in \bR^d$ is \textit{Lyapunov stable} if for every neighborhood $U_\epsilon \subset \bR^d$ of $\bar\bx$, there exists a neighborhood $U_\delta \subset \bR^d$ of $\bar\bx$ such that for all $\bx(t)$ with $\bx(0) \in U_\delta$, $\bx(t) \in U_\epsilon$ for all $t \geq 0$.
\end{definition}
If then, any starting point near some global minima cannot go far away from that global minimum. By \Cref{prop:main_infdiv_conv}, bounded SpecGF iterations should converge to a critical point. The stability of the global minima is proved by \Cref{lem:main_Loja_inequ}; see \Cref{thm:SpecGF_Tb_stable} for details. Therefore, the \textbf{local} convergence of SpecGF with $\cT_\beta$ is always guaranteed.

\paragraph{Convergence Rate.} The convergence rate of SpecGF with $\cT_\beta$ readily follows by \Cref{lem:main_Loja_inequ}. Assume that SpecGF converges to a global minimum. We show that the loss function is essentially Morse-Bott, whose optimal exponent in the \L{}ojasiewicz inequality is $\frac{1}{2}$, leading to exponential convergence;
see \Cref{sec:conv_rate_specgf} for detailed proof.
\begin{proposition}[Informal]\label{prop:main_conv_rate}
    Assume that the singular values of $\bY$ are mutually distinct. If SpecGF with $\cT_\beta$ converges to a global minimum of $\cL$, then it exponentially converges in terms of $\cL(t)$.
\end{proposition}
The assumption of distinct singular values of the target matrix is also considered in previous studies~\citep{gidel2019implicit,jin2023understanding}.

\begin{remark}\label{rem:global_conv}
Assume that both $\nabla_\bA \cL(t)$ and $\nabla_\bB \cL(t)$ are full-rank throughout SpecGF with $\cT$. Then SpecGF with $\cT$ is analytic, and the same arguments for SpecGF with $\cT_\beta$ hold. Moreover, if SpecGF with $\cT$ converges, then it should converge \textit{in finite time}; see \Cref{prop:specgf_t_finite_conv} for details.
\end{remark}

\paragraph{Global Convergence via Regularization.} Adding a $\ell_2$ regularizer guarantees global convergence of SpecGF with $\cT_\beta$.
All theoretical claims for SpecGF with $\cT_\beta$ also apply to the regularized SpecGF with $\cT_\beta$. Moreover, boundedness of $\bA$ and $\bB$ readily follows from the regularization; see \Cref{sec:L2reg} for details. Therefore, regularized SpecGF with $\cT_\beta$ \textbf{globally} converges to the global minima almost surely.
\section{Empirical Validation of Our Theory}

In this section, we validate our theoretical results. In particular, we confirm the uniform growth of singular values under SpecGF and the equal-rate learning behavior.

\begin{figure}[t]
	\centering
	\begin{subfigure}[b]{0.26\textwidth}
        \includegraphics[width=\linewidth]{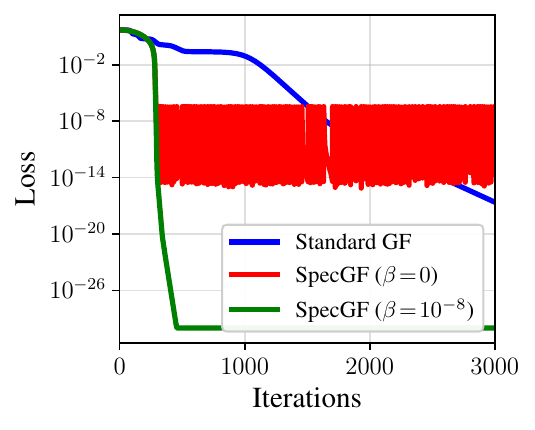}
	\end{subfigure}
    \caption{Loss comparison.}
	\label{fig:loss_comparison}
    \vspace{-0.15in}
\end{figure}

\paragraph{Setup.} We use the same setup as in \Cref{fig:toy_observation}. We set $\beta = 10^{-8}$ (for our theory) or $\beta = 0$ (for exact orthonormalization). To analyze the dynamics, we track core variables $\bX$ and $\bZ$ with the diagonal entries $d_i = [\bG]_{ii}$ and $\sigma_i(\bA\bB)$.

\paragraph{Results.} We highlight several key observations:
\begin{enumerate}[leftmargin=6mm,itemsep=2pt,topsep=0pt,partopsep=0pt]
    \item \textbf{Convergence.} In \Cref{fig:loss_comparison}, both methods successfully minimize the loss, confirming the convergence guarantee of \Cref{thm:main_specgf_conv}. Notably, SpecGF ($\beta = 10^{-8}$) converges significantly faster than vanilla GF. 
    \item \textbf{Diagonal approximation.} In \Cref{fig:sval}, the left and right panels show that $d_i(t) = [\bX(t)\bZ(t)]_{ii}$ closely tracks $\sigma_i(\bA(t)\bB(t))$, validating \Cref{lem:diagonal_approx_main}. This confirms that the diagonal entries of the core product serve as accurate surrogates for the singular values of $\bA(t)\bB(t)$.
    \item \textbf{Uniform Growth.} Under SpecGF, the curves $\sqrt{d_i(t)}$ in \Cref{fig:sval} exhibit nearly identical slopes, demonstrating our \Cref{thm:uniform_growth_main}. In contrast, vanilla GF shows non-uniform growth. 
\end{enumerate}

We present further experimental results in \Cref{sec:theory_expm} by modifying the training settings of \Cref{fig:toy_observation}. Across various settings, the uniform growth of singular values of the product of LoRA factors persists.

\begin{figure}[t]
	\centering
    \begin{subfigure}[b]{0.48\textwidth}
        \includegraphics[width=\linewidth]{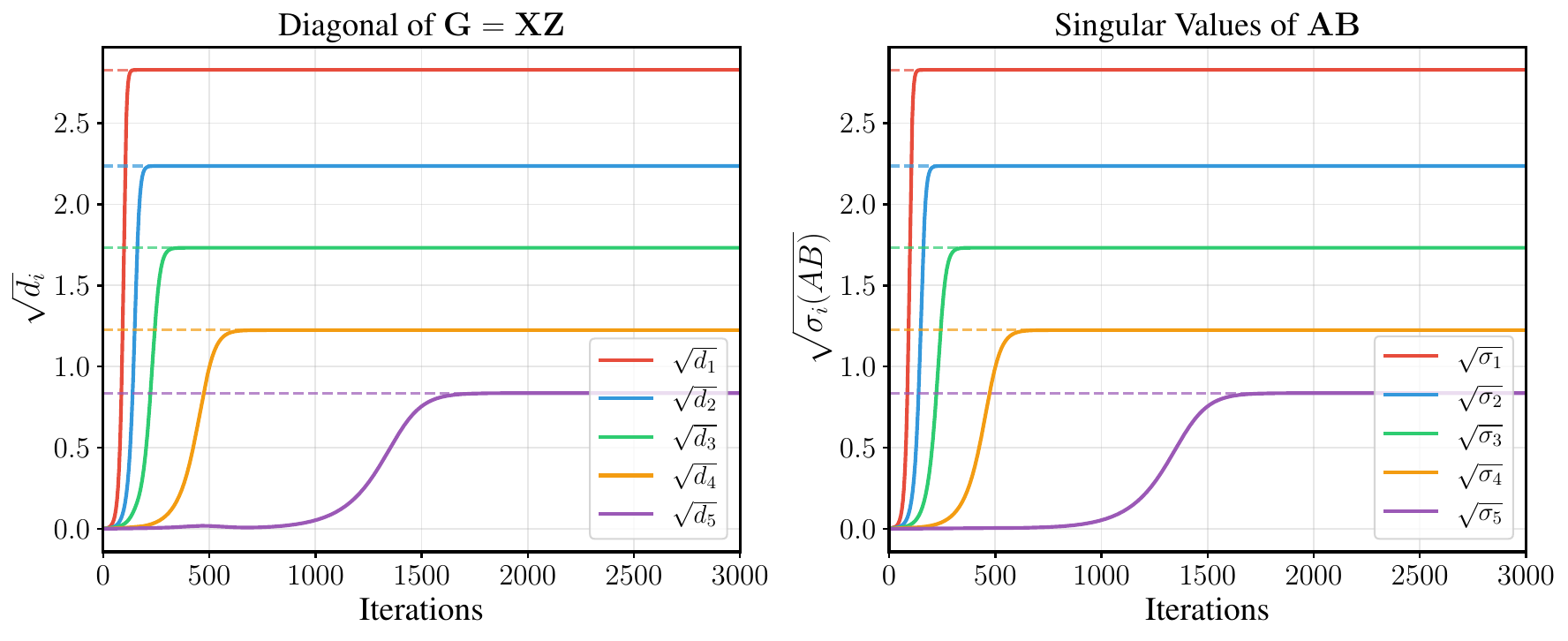}
        \caption{Vanilla GF}
	\end{subfigure}
	\begin{subfigure}[b]{0.48\textwidth}
        \includegraphics[width=\linewidth]{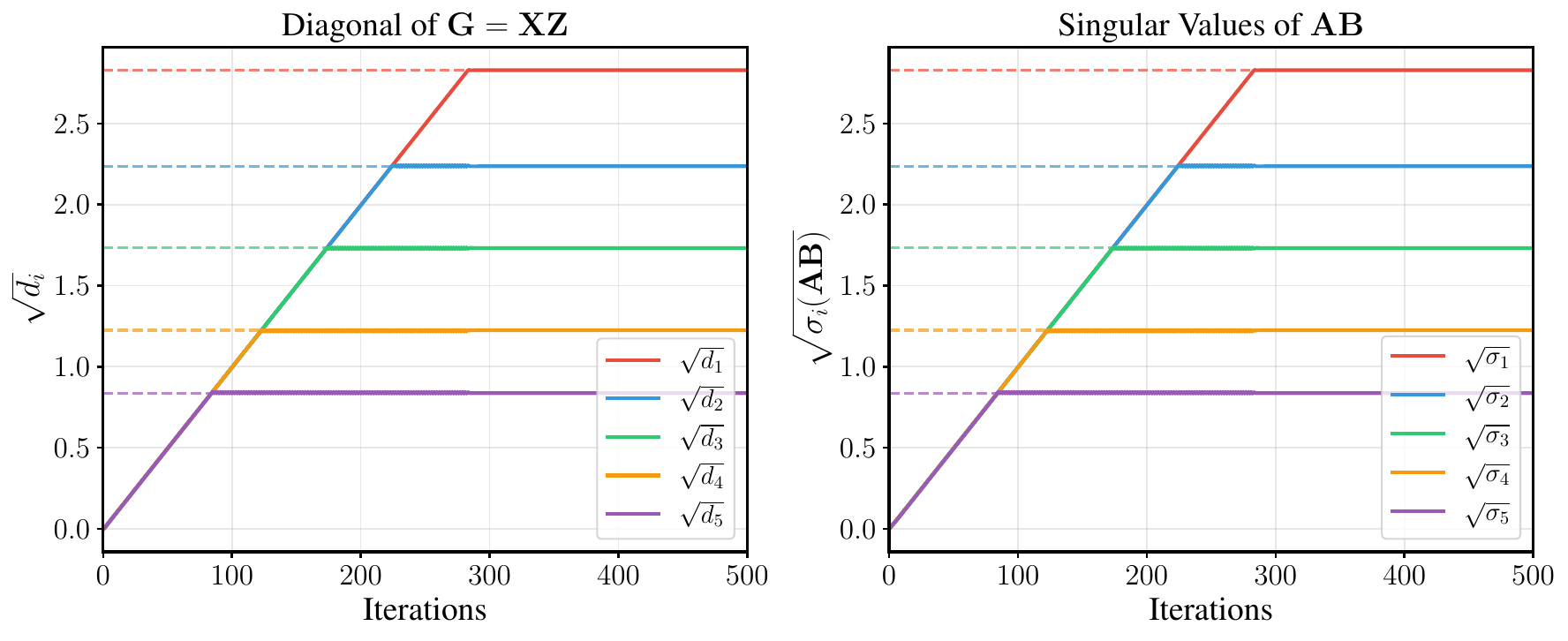}
        \caption{SpecGF ($\beta = 10^{-8}$)}
	\end{subfigure}
    \begin{subfigure}[b]{0.48\textwidth}
        \includegraphics[width=\linewidth]{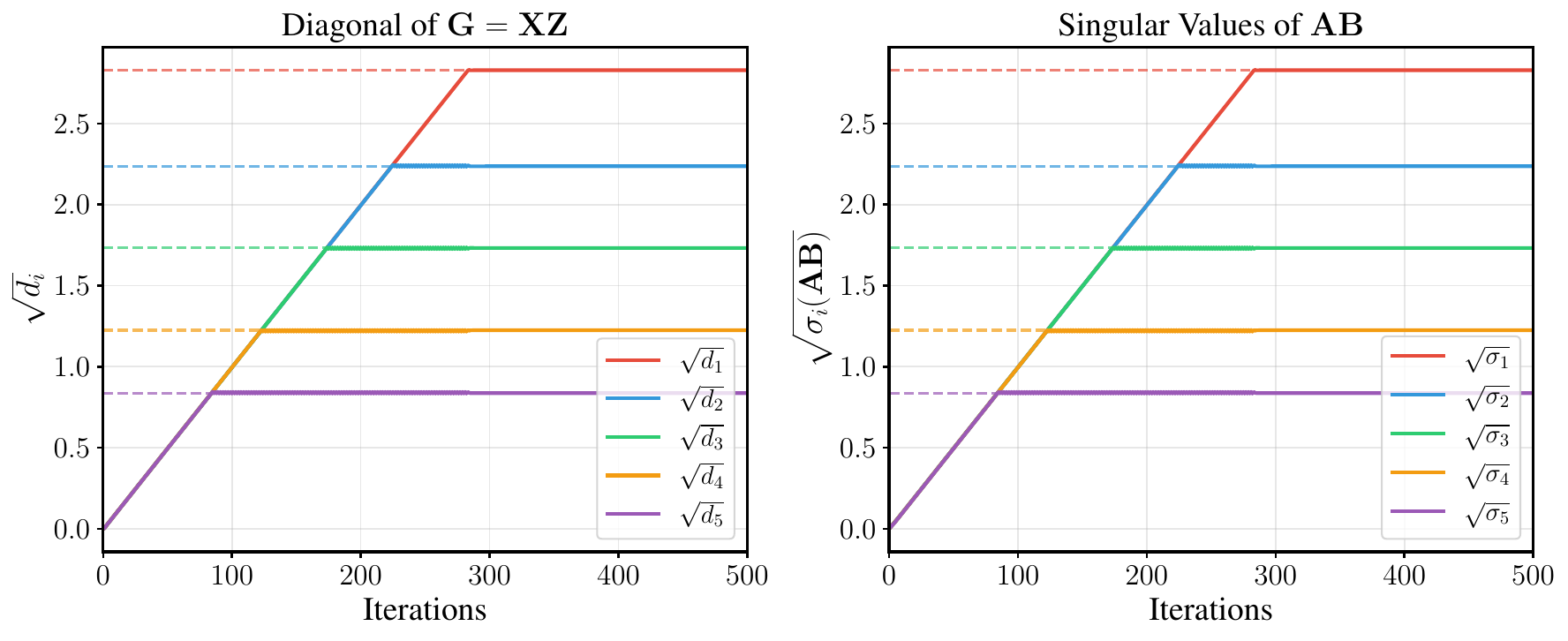}
        \caption{SpecGF ($\beta = 0$)}
    \end{subfigure}
    \caption{Comparison of SpecGF and vanilla GF on matrix factorization with a rank-$5$ target.}
	\label{fig:sval}
\end{figure}
\section{Conclusion}

We analyzed spectral gradient flow (SpecGF) for LoRA-style matrix factorization and proved that singular values of the product $\bA\bB$ exhibit uniform growth, leading to a ``equal-rate'' learning that contrasts with the largest-first behavior of standard gradient flow. We also established that SpecGF almost surely converges to global minima; with $\ell_2$ regularization, global convergence is also guaranteed. We validated our theory via experiments on matrix factorization and LLM fine-tuning. A key limitation is that global convergence without regularization remains open due to the absence of known invariants for SpecGF. Extending our continuous-time analysis to discrete-time SpecGD and more general network architectures is an important future work.


\section*{Impact Statement}

This paper presents work whose goal is to advance the field of machine learning. There are many potential societal consequences of our work, none of which we feel must be specifically highlighted here.


\bibliography{reference}
\bibliographystyle{icml2026}

\newpage
\appendix
\onecolumn

\section*{Supplementary Materials}
\section{Speed of singular values of \texorpdfstring{$\bA$}{A} and \texorpdfstring{$\bB$}{B}}

In this section, we show that the singular values of \emph{$\bA$ and $\bB$} increase at the same speed if the whole matrix evolution can be well-decoupled into scalar evolutions.

\subsection{Rank-1 Case}\label{sec:sng_evo_rk1}

Next, we consider the case with random initialization but simplify the problem to rank-1 $\bY$ and $r=1$ case:
\begin{equation*}
    \bY = \sigma \bu\bv^\top, \quad \bA \in \sR^{m\times 1}, \quad \bB \in \sR^{1\times n}.
\end{equation*}
There, $\bu$ and $\bv$ are unit vectors in $\sR^m$ and $\sR^n$, respectively.

\subsubsection{SpecGF with $\cT_\beta$}
We note that the dynamics exist uniquely (from analyticity of $\cT_\beta$) and globally (from boundedness of velocity by $\cT_\beta$). In addition, for this rank-1 case, SpecGF with $\cT_\beta$ is:
\begin{align*}
    \cT_\beta(\nabla_\bA \cL) = \frac{\nabla_\bA \cL}{\sqrt{\normF{\nabla_\bA \cL}^2 + \beta}}, \quad \cT_\beta(\nabla_\bB \cL) = \frac{\nabla_\bB \cL}{\sqrt{\normF{\nabla_\bB \cL}^2 + \beta}}.
\end{align*}
We sample a random unit vector $\bw$ in $\sR^n$, then $\alpha \triangleq \angleb{\bv, \bw} \ne 0$ with probability 1.
\begin{assumption}
    $\alpha \ne 0$.
\end{assumption}

A typical LoRA training would start at
\begin{equation*}
    \bA(0) = \bm 0, \quad \bB(0) = \gamma \bw^\top,
\end{equation*}
for sufficiently small $\gamma > 0$ with gradients
\begin{align*}
    \nabla_\bA \cL &= (\bA\bB - \bY)\bB^\top,\\
    \nabla_\bB \cL &= \bA^\top(\bA\bB - \bY).
\end{align*}

We first show that $\bA(t) \in \mathrm{span}\{\bu\}$ and $\bB(t)^\top \in \mathrm{span}\{\bv, \bw\}$ at all $t \geq 0$ for SpecGF with $\cT$.
\begin{lemma}
    For all $t \geq 0$, $\bA(t) \in \mathrm{span}\{\bu\}$ and $\bB(t)^\top \in \mathrm{span}\{\bv, \bw\}$ hold.
\end{lemma}
\begin{proof}
    Assume that $\bA(t) = a(t)\bu$ and $\bB(t) = b_0(t)\bv^\top + c_0(t)\bw^\top$ for some $a(t), b_0(t), c_0(t) \in \sR$. We directly compute the gradients:
    \begin{align*}
        \nabla_\bA \cL(t) &= (\bA(t)\bB(t) - \bY)\bB(t)^\top = \bu\roundb{(a(t)b_0(t)-\sigma)(b_0(t)+c_0(t)\alpha) + a(t)c_0(t)(b_0(t)\alpha+c_0(t))},\\
        \nabla_\bB \cL(t) &= \bA(t)^\top(\bA(t)\bB(t) - \bY) = a(t)(a(t)b_0(t)-\sigma)\bv^\top + a(t)^2c_0(t)\bw^\top.
    \end{align*}
    Normalization only modifies the coefficients, so the updates $\dot\bA(t) \in \mathrm{span}\{\bu\}$ and $\dot\bB(t)^\top \in \mathrm{span}\{\bv, \bw\}$ hold. The proof is concluded as $\bA(0) \in \mathrm{span}\{\bu\}$ and $\bB(0)^\top \in \mathrm{span}\{\bv, \bw\}$.
\end{proof}

Without loss of generality, assume $\alpha > 0$. If $\alpha < 0$, one can simply flip the sign of $\bu$ and $\bv$ and the argument analogously holds. If $\alpha = 1$, then $\bw = \bv$ and it is straightforward to see that
\begin{align*}
    \nabla_\bA \cL(t) &= (\bA(t)\bB(t) - \bY)\bB(t)^\top = \bu(a(t)b_0(t)-\sigma)b_0(t),\\
    \nabla_\bB \cL(t) &= \bA(t)^\top(\bA(t)\bB(t) - \bY) = a(t)(a(t)b_0(t)-\sigma)\bv^\top.
\end{align*}
Hence, if $\gamma^2 < \sigma$, then the dynamics of $a(t)$ and $b_0(t)$ are exactly equal to those in \Cref{prop:specgf_rk1_lora}. Henceforth, consider the case $\alpha \in (0, 1)$.

For the convenience of later analysis, define
\begin{equation}
    \bz \triangleq \frac{1}{\sqrt{1-\alpha^2}}(\bw - \alpha\bv),
\end{equation}
which gives $\bw = \alpha \bv + \sqrt{1-\alpha^2}\bz$ and $\angleb{\bv, \bz} = 0$. Thus, $\bB(t) = (b_0(t) + \alpha c_0(t))\bv^\top + \sqrt{1-\alpha^2}c_0(t)\bz^\top$. For brevity, we write the parameters as
\begin{equation}
    \bA(t) = a(t)\bu, \quad \bB(t) \triangleq b(t)\bv^\top + c(t)\bz^\top.
\end{equation}
The goal is to show that 1) $a(t)b_0(t) \to \sigma$ \textit{with} neither $a(t) \to \infty$ nor $b_0(t) \to \infty$ and 2) $c_0(t) \to 0$; this is equivalent to showing that 1) $a(t)b(t) \to \sigma$ \textit{with} neither $a(t) \to \infty$ nor $b(t) \to \infty$ and 2) $c(t) \to 0$. The remaining part is to show that $a(t) \approx b(t)$; due to randomness in the initialization, $\dot a(t)$ cannot be exactly equal to $\dot { b}(t)$. Instead, we show that $|a(t) - b(t)|$ and $|\dot a(t) - \dot { b}(t)|$ are both uniformly small over $t$, and the gap vanishes to 0 as $\gamma \to 0$.

We derive the time derivative of the coefficients. From
\begin{align*}
    \cL(t) &= \frac{1}{2}\normF{\bA(t)\bB(t) - \bY}^2 = \frac{1}{2}\normF{\bu\roundb{(a(t)b(t) - \sigma)\bv^\top + a(t)c(t)\bz^\top}}^2\\
    &= \frac{1}{2}\norm{(a(t)b(t) - \sigma)\bv^\top + a(t)c(t)\bz^\top}_2^2 \triangleq \frac{1}{2}\norm{\bd(t)}_2^2,
\end{align*}
with $\bd(t) = a(t)\bB(t) - \sigma\bv^\top$. Whenever $\norm{\bd(t)}_2 = 0$, $a(t)b(t) = \sigma$ and $c(t) = 0$ and the dynamics terminate. From now on, the interval of time of our interest is $[0, T_\mathrm{max})$ where $T_\mathrm{max} = \inf \{t:\norm{\bd(t)}_2 = 0\}$; $T_\mathrm{max}$ could be possibly infinite.

We derive the loss gradient, with $m(t) \triangleq (a(t)b(t)-\sigma)b(t) + a(t)c(t)^2$,
\begin{align*}
    \nabla_\bA \cL(t) &= \bu\roundb{m(t)},\\
    \nabla_\bB \cL(t) &= a(t)\bd(t),
\end{align*}
and we obtain
\begin{align*}
    \dot a(t) &= -\frac{m(t)}{\sqrt{m(t)^2 + \beta}}\\
    \dot { b}(t) &= -\frac{a(t)(a(t)b(t) - \sigma)}{\sqrt{a(t)^2\norm{\bd(t)}_2^2 + \beta}}\\
    \dot { c}(t) &= -\frac{a(t)^2c(t)}{\sqrt{a(t)^2\norm{\bd(t)}_2^2 + \beta}}.
\end{align*}

As $ c(0) = \sqrt{1-\alpha^2}\gamma > 0$, it directly follows that $c(t) \geq 0$ for all $t \geq 0$ and $c(t)$ is non-increasing. In addition, we have $a(0) = 0$ with $\dot a(0) > 0$ and $ b(0) = \alpha\gamma > 0$. The simultaneous non-negativity of $a$ and $ b$ at $t = 0$ makes them keep non-negative for all $t \geq 0$ under continuity. If $ b$ hits 0 while $a(t) \geq 0$, it gains non-negative speed; same for $a$. The only way for them to be negative is that both $a$ and $ b$ go across 0 at the same time, but then the dynamics become stationary. Therefore, we can safely argue that $a(t), b(t), c(t) \geq 0$.

We return our attention to the loss function. We know that $\cL(t)$ is non-increasing, hence $\norm{\bd(t)}_2$ is also non-increasing. Precisely,
\begin{align*}
    \ddt \frac{1}{2}\norm{\bd(t)}_2^2 &= \angleb{\bd(t), \dot \bd(t)}\\
    &= \angleb{\bd(t), \dot a(t)\bB(t) + a(t)\dot\bB(t)}\\
    &= \dot a(t) \angleb{\bd(t), \bB(t)} + a(t) \angleb{\bd(t), \dot\bB(t)}\\
    &= -\frac{m(t)}{\sqrt{m(t)^2 + \beta}}\cdot m(t) - a(t) \frac{a(t)\norm{\bd(t)}_2^2}{\sqrt{a(t)^2\norm{\bd(t)}_2^2 + \beta}}\\
    &= -\frac{m(t)^2}{\sqrt{m(t)^2 + \beta}} - \frac{a(t)^2\norm{\bd(t)}_2^2}{\sqrt{a(t)^2\norm{\bd(t)}_2^2 + \beta}} \leq 0.
\end{align*}

Notice that $m(0) = -\sigma b(0) < 0$, thus $\ddt \norm{\bd(0)}_2^2 < 0$. Hence, for $t > 0$, $\norm{\bd(t)}_2^2 < \norm{\bd(0)}_2^2$. Then,
$\norm{\bd(t)}_2 < \sigma$ for $t > 0$ also implies that $a(t), b(t) > 0$; if any of two is 0 at $t$, then $\norm{\bd(t)}_2 = \norm{(a(t)b(t) - \sigma)\bv^\top + a(t)c(t)\bz^\top}_2 \geq \sigma$.
\begin{lemma}\label{lem:m_neg_Tb}
    $m(t) < 0$.
\end{lemma}
\begin{proof}
    Recall that $m(t) = (a(t)b(t)-\sigma)b(t) + a(t)c(t)^2$ and $m(0) = -\sigma b(0) < 0$.
    
    If there exists some time $\tau \triangleq \min\curlyb{t: m(t) \geq 0}$, then $m(\tau) = 0$ and $\dot a(\tau) = 0$. In addition, from
    \begin{align*}
        m(t) &= \angleb{\bd(t), \bB(t)}\\
        &= a(t)\normF{\bB(t)}^2 - \sigma b(t).\\
        \dot m(t) &= \dot a(t)\normF{\bB(t)}^2 + a(t) \ddt\normF{\bB(t)}^2 - \sigma \dot { b}(t)\\
        &= \dot a(t)\normF{\bB(t)}^2 - 2a(t)^2\frac{m(t)}{\sqrt{a(t)^2\norm{\bd(t)}_2^2 + \beta}} - \sigma \dot { b}(t).
    \end{align*}
    Plugging $\tau$ vanishes the first two terms on the right-hand side, giving $\dot m(\tau) = -\sigma \dot { b}(\tau) < 0$, a contradiction: If $\dot { b}(\tau) = 0$, then we have $a(\tau) b(\tau) = \sigma$ and $ c(\tau) = 0$.
\end{proof}
\Cref{lem:m_neg_Tb} directly implies that $a(t)b(t) < \sigma$, thus $a(t)$ and $b(t)$ are strictly increasing and positive. Therefore, all the scalars $a(t), b(t)$, and $c(t)$ are bounded, thus SpecGF with $\cT_\beta$ converges. From \citet{bah2022learning}, the critical point is either the origin or $\bY$, the point where $a(t)b(t) = \sigma$ and $c(t) = 0$. For $t > 0$, $\norm{\bd(t)}_2^2 < \norm{\bd(0)}_2^2$ and thus the origin cannot be the converging point. Therefore, SpecGF with $\cT_\beta$ converges to the global minima, and we see that $\norm{\bd(t)}_2 \to 0$, thus $a(t)b(t) \to \sigma$ and $c(t) \to 0$. This recovers the limiting behavior of $b_0(t) = b(t) - \alpha c_0(t)$ and $c_0(t) = \frac{1}{\sqrt{1-\alpha^2}}c(t)$.

We now bound $\Delta(t) \triangleq a(t) - b(t)$. Recall that

\begin{equation*}
    \norm{\bd(t)}_2 = \sqrt{\roundb{a(t)b(t)-\sigma}^2 + \roundb{a(t)c(t)}^2}.
\end{equation*}

\begin{enumerate}
    \item $a(t) \geq b(t)$: Notice that, from $-m(t) = -(a(t)b(t)-\sigma)b(t) - a(t)c(t)^2 \leq -(a(t)b(t)-\sigma)b(t) \leq -(a(t)b(t)-\sigma)a(t)$,
    \begin{equation*}
        \dot a(t) = -\frac{m(t)}{\sqrt{m(t)^2 + \beta}} = \frac{-m(t)}{\sqrt{(-m(t))^2 + \beta}} \leq \frac{-(a(t)b(t)-\sigma)a(t)}{\sqrt{(a(t)b(t)-\sigma)^2a(t)^2 + \beta}}
    \end{equation*}
    Then
    \begin{align*}
        \dot \Delta(t) &= \dot a(t) - \dot { b}(t)\\
        &\leq \frac{-(a(t)b(t)-\sigma)a(t)}{\sqrt{(a(t)b(t)-\sigma)^2a(t)^2 + \beta}} + \frac{a(t)(a(t)b(t) - \sigma)}{\sqrt{a(t)^2\norm{\bd(t)}_2^2 + \beta}}\\
        &= \frac{-(a(t)b(t)-\sigma)a(t)\roundb{a(t)^2\norm{\bd(t)}_2^2 - a(t)^2(a(t)b(t)-\sigma)^2}}{\sqrt{a(t)^2\norm{\bd(t)}_2^2 + \beta}\sqrt{(a(t)b(t)-\sigma)^2a(t)^2 + \beta}\roundb{\sqrt{a(t)^2\norm{\bd(t)}_2^2 + \beta} + \sqrt{(a(t)b(t)-\sigma)^2a(t)^2 + \beta}}}\\
        &= \frac{-(a(t)b(t)-\sigma)a(t)^3\roundb{a(t)c(t)}^2}{\sqrt{a(t)^2\norm{\bd(t)}_2^2 + \beta}\sqrt{(a(t)b(t)-\sigma)^2a(t)^2 + \beta}\roundb{\sqrt{a(t)^2\norm{\bd(t)}_2^2 + \beta} + \sqrt{(a(t)b(t)-\sigma)^2a(t)^2 + \beta}}}\\
        &= -\dot { c}(t)\frac{-(a(t)b(t)-\sigma)a(t)^3c(t)}{\sqrt{(a(t)b(t)-\sigma)^2a(t)^2 + \beta}\roundb{\sqrt{a(t)^2\norm{\bd(t)}_2^2 + \beta} + \sqrt{(a(t)b(t)-\sigma)^2a(t)^2 + \beta}}}\\
        &\leq -\dot { c}(t)\frac{-(a(t)b(t)-\sigma)a(t)^3c(t)}{-(a(t)b(t)-\sigma)a(t)\roundb{a(t)\norm{\bd(t)}_2}} \leq -\dot { c}(t).
    \end{align*}
    \item $a(t) < b(t)$: We require additional assumptions on $\gamma$.
    \begin{assumption}\label{asm:alpha_small_Tb_2}
        Take $T \triangleq \frac{\sqrt{1-\alpha^2}}{\alpha}\roundb{1 + \frac{4\sqrt{\beta}}{\sigma}} > 0$; notice that $T$ is independent of $\gamma$. $\gamma > 0$ satisfies
        \begin{equation*}
            \gamma^2 < \frac{\alpha\sigma}{4(1-\alpha^2)^{3/2}} \quad \text{and} \quad \gamma < \min\curlyb{\frac{\sigma/2}{\sqrt{1-\alpha^2}(\alpha + T)}, 1}.
        \end{equation*}
    \end{assumption}
    \Cref{asm:alpha_small_Tb_2} translates to
    \begin{equation*}
         c(0)^3 \leq \frac{\sigma}{4} b(0), \quad  c(0) \roundb{ b(0) + T} \leq \frac{\sigma}{2}.
    \end{equation*}
    The first one is straightforward; the second one holds as
    \begin{align*}
         c(0) \roundb{ b(0) + T} = \gamma\sqrt{1-\alpha^2}\roundb{\gamma\alpha + T} \leq \gamma\sqrt{1-\alpha^2}\roundb{\alpha + T} < \frac{\sigma}{2}.
    \end{align*}
    
    We see that
    \begin{align*}
        -\dot \Delta(t) &= \dot { b}(t) - \dot a(t)\\
        &\leq -\frac{a(t)(a(t)b(t) - \sigma)}{\sqrt{(a(t)b(t)-\sigma)^2a(t)^2 + \beta}} + \frac{m(t)}{\sqrt{(a(t)b(t)-\sigma)^2b(t)^2 + \beta}}\\
        &= -\frac{a(t)(a(t)b(t) - \sigma)}{\sqrt{(a(t)b(t)-\sigma)^2a(t)^2 + \beta}} + \frac{b(t)(a(t)b(t) - \sigma)}{\sqrt{(a(t)b(t)-\sigma)^2b(t)^2 + \beta}} + \frac{a(t)c(t)^2}{\sqrt{(a(t)b(t)-\sigma)^2b(t)^2 + \beta}}\\
        &\leq \frac{a(t)c(t)^2}{\sqrt{(a(t)b(t)-\sigma)^2b(t)^2 + \beta}} \leq \frac{a(t)c(t)^2}{\sqrt{\beta}}.
    \end{align*}
    We aim to bound the integral of $a(t)c(t)^2$. To this end, define $t_c$ when $a$ gets larger than $ c$:
    \begin{equation*}
        t_c \triangleq \inf\curlyb{t:a(t) \geq c(t)} < T_\mathrm{max}.
    \end{equation*}
    This time uniquely exists as $a(t)$ is strictly increasing while $c(t)$ (strictly) decreases to $0$ for $t > 0$.
    \begin{itemize}
        \item On $t \geq t_c$,
        \begin{equation*}
            a(t)c(t)^2 = -\dot { c}(t)\frac{c(t) \sqrt{a(t)^2 \norm{\bd(t)}_2^2 + \beta}}{a(t)} \leq -\dot { c}(t) \frac{c(t)}{a(t)}\roundb{a(t)\norm{\bd(t)}_2 + \sqrt{\beta}} \leq -\dot { c}(t) \roundb{c(t)\sigma + \sqrt{\beta}}.
        \end{equation*}
        This implies
        \begin{align*}
            \int_{t_c}^{T_\mathrm{max}} a(t)c(t)^2 \mathrm{d}t \leq \sigma \int_{t_c}^{T_\mathrm{max}} -\dot { c}(t)c(t) \mathrm{d}t + \sqrt{\beta} \int_{t_c}^{T_\mathrm{max}} -\dot { c}(t) \mathrm{d}t
            \leq \frac{\sigma}{2} c(0)^2 + \sqrt{\beta} c(0).
        \end{align*}

        \item For $t \leq t_c$, we show that $T \geq t_c$ from \Cref{asm:alpha_small_Tb_2}. Then, the integral of $a(t)c(t)^2$ on $[0, t_c]$ is bounded by that on $[0, T]$. Suppose by contradiction that $t_c > T$. On $[0, T] \subset [0, t_c]$,
        \begin{equation*}
            a(t) < c(t) \leq  c(0), \quad b(t) \leq  b(0) + T.
        \end{equation*}
        The latter holds since $\dot { b}(t)$ is bounded by 1. Then, on $[0, T]$,
        \begin{equation*}
            \sigma - a(t)b(t) \geq \sigma -  c(0)( b(0) + T) \geq \frac\sigma 2,
        \end{equation*}
        and this further implies
        \begin{equation*}
            -m(t) = (\sigma - a(t)b(t))b(t) - a(t)c(t)^2 \geq  b(0)\frac\sigma 2 -  c(0)^3 \geq  b(0)\frac\sigma 4.
        \end{equation*}
        Hence,
        \begin{equation*}
            \dot a(t) = -\frac{m(t)}{\sqrt{m(t)^2 + \beta}} \geq \frac{-m(t)}{-m(t) + \sqrt\beta} \geq \frac{ b(0)\sigma}{ b(0)\sigma + 4\sqrt\beta}.
        \end{equation*}
        However, this implies
        \begin{equation*}
            a(T) \geq T\frac{ b(0)\sigma}{ b(0)\sigma + 4\sqrt\beta} = \frac{\sqrt{1-\alpha^2}}{\alpha}\roundb{1 + \frac{4\sqrt{\beta}}{\sigma}}\frac{ b(0)\sigma}{ b(0)\sigma + 4\sqrt\beta} =  c(0),
        \end{equation*}
        a contradiction.
    \end{itemize}
    
    We see that
    \begin{equation*}
        \int_0^{T_\mathrm{max}} a(t)c(t)^2 \mathrm{d}t = \int_0^{t_c} a(t)c(t)^2 \mathrm{d}t + \int_{t_c}^{T_\mathrm{max}} a(t)c(t)^2 \mathrm{d}t \leq T c(0)^3 + \frac{\sigma}{2} c(0)^2 + \sqrt{\beta} c(0).
    \end{equation*}
\end{enumerate}

We summarize what we have until now.
\begin{align*}
\begin{dcases}
    \text{On } I_+ \triangleq \{t: a(t) \geq b(t)\}, & \dot \Delta(t) \leq -c(t).\\
    \text{On } I_- \triangleq \{t: a(t) < b(t)\}, & -\dot \Delta(t) \leq \frac{a(t)c(t)^2}{\sqrt{\beta}}.
\end{dcases}
\end{align*}
Notice that $\Delta(t)$ is (real-)analytic, hence its zero set is discrete (even finite on each $[0, t]$). Let $\Delta_+(t) \triangleq \max\{\Delta(t), 0\}$, then $\Delta_+$ is everywhere continuous and a.e. differentiable (precisely, differentiable except on discrete set).
\begin{enumerate}
    \item Define $p(t) \triangleq \Delta_+(t) + c(t)$; it is everywhere continuous and a.e. differentiable. On $I_+$, $\dot p(t) = \dot a(t) - \dot { b}(t) + \dot { c}(t) \leq 0$. On $I_-$, $\dot p(t) = \dot { c}(t) \leq 0$. Hence, $p(t)$ is non-increasing and
    \begin{equation*}
        \Delta_+(t) + c(t) = p(t) \leq p(0) = \Delta_+(0) +  c(0).
    \end{equation*}

    \item Define $q(t) \triangleq (-\Delta)_+(t) - \frac{1}{\sqrt{\beta}}\int_0^t a(s) c(s)^2 \mathrm{d}s$; it is everywhere continuous and a.e. differentiable. On $I_-$, $\dot q(t) = \dot { b}(t) - \dot a(t) - \frac{a(t)c(t)^2}{\sqrt{\beta}} \leq 0$. On $I_+$, $\dot q(t) = - \frac{a(t)c(t)^2}{\sqrt{\beta}} \leq 0$. Hence, $q(t)$ is non-increasing and
    \begin{equation*}
        (-\Delta)_+(t) - \frac{1}{\sqrt{\beta}}\int_0^t a(s) c(s)^2 \mathrm{d}s = q(t) \leq q(0) = (-\Delta)_+(0).
    \end{equation*}
\end{enumerate}
Therefore,
\begin{align*}
    |\Delta(t)| &= \Delta_+(t) + (-\Delta)_+(t) \leq \Delta_+(0) +  c(0) - c(t) + (-\Delta)_+(0) + \frac{1}{\sqrt{\beta}}\int_0^t a(s) c(s)^2 \mathrm{d}s\\
    &= |\Delta(0)| +  c(0) + \frac{1}{\sqrt{\beta}}\int_0^t a(s) c(s)^2 \mathrm{d}s\\
    &\leq  b(0) +  c(0) + \frac{1}{\sqrt{\beta}}\squarb{T c(0)^3 + \frac{\sigma}{2} c(0)^2 + \sqrt{\beta} c(0)}\\
    &= \gamma\alpha + 2\gamma\sqrt{1-\alpha^2} + \frac{\sigma(1-\alpha^2)}{2\sqrt\beta}\gamma^2 + \frac{T}{\sqrt\beta}(1-\alpha^2)^{\frac32}\gamma^3 \triangleq C(\gamma).
\end{align*}
Observe that the difference between the ``singular values'' is uniformly small, and especially the difference vanishes to 0 as $\gamma \to 0$.

\paragraph{Synchronized speed of singular values.} We know that $a(t)$ is bounded; denote the upper bound of $a(t)$ as $A > 0$. $A$ is computed as follows:
\begin{equation*}
    a(t)^2 - \Delta(t)a(t) = a(t)b(t) \leq \sigma,
\end{equation*}
and thus 
\begin{equation*}
    a(t)^2 \leq \sigma + |\Delta(t)|a(t) \leq \sigma + C(\gamma)a(t),
\end{equation*}
giving
\begin{equation*}
    A = \frac{C(\gamma) + \sqrt{C(\gamma)^2 + 4\sigma}}{2}.
\end{equation*}
Notice that $A \to \sqrt\sigma$ as $\gamma \to 0$.

\begin{proposition}
    \begin{equation*}
        \abs*{\dot a(t) - \dot { b}(t)} \leq \frac{1}{\sqrt\beta} \roundb{C(\gamma)\sigma + A(1-\alpha^2)\gamma^2} + \frac{A^4(1-\alpha^2)\gamma^2}{2\beta}.
    \end{equation*}
    The right-hand side goes to 0 as $\gamma \to 0$.
\end{proposition}
\begin{proof}
    Recall
    \begin{align*}
        \dot a(t) &= -\frac{m(t)}{\sqrt{m(t)^2 + \beta}}, \quad m(t) = (a(t)b(t)-\sigma)b(t) + a(t)c(t)^2,\\
        \dot { b}(t) &= -\frac{a(t)(a(t)b(t) - \sigma)}{\sqrt{a(t)^2\norm{\bd(t)}_2^2 + \beta}}.
    \end{align*}
    Then,
    \begin{align*}
        \abs*{\dot a(t) - \dot { b}(t)} &\leq \abs*{-\frac{m(t)}{\sqrt{m(t)^2 + \beta}} - \frac{a(t)(\sigma - a(t)b(t))}{\sqrt{a(t)^2(\sigma - a(t)b(t))^2 + \beta}}}\\ &\phantom{=}+ \abs*{\frac{a(t)(\sigma - a(t)b(t))}{\sqrt{a(t)^2(\sigma - a(t)b(t))^2 + \beta}} - \frac{a(t)(\sigma - a(t)b(t))}{\sqrt{a(t)^2\norm{\bd(t)}_2^2 + \beta}}}\\
        &\leq \frac{1}{\sqrt\beta} \abs*{-m(t) - a(t)(\sigma - a(t)b(t))}\\
        &\phantom{=}+ a(t)(\sigma - a(t)b(t))\roundb{\frac{1}{\sqrt{a(t)^2(\sigma - a(t)b(t))^2 + \beta}} - \frac{1}{\sqrt{a(t)^2\roundb{(\sigma - a(t)b(t))^2 + (a(t)c(t))^2} + \beta}}}\\
        &\leq \frac{1}{\sqrt\beta} \abs*{(b(t) - a(t))(\sigma - a(t)b(t)) - a(t)c(t)^2} + a(t)(\sigma - a(t)b(t))\frac{a(t)^4c(t)^2}{2\roundb{a(t)^2(\sigma - a(t)b(t))^2 + \beta}^{\frac32}}\\
        &\leq \frac{1}{\sqrt\beta} \roundb{|\Delta(t)|(\sigma - a(t)b(t)) + a(t)c(t)^2} + \frac{a(t)^4c(t)^2}{2\roundb{a(t)^2(\sigma - a(t)b(t))^2 + \beta}}\\
        &\leq \frac{1}{\sqrt\beta} \roundb{C(\gamma)\sigma + A c(0)^2} + \frac{A^4 c(0)^2}{2\beta} = \frac{1}{\sqrt\beta} \roundb{C(\gamma)\sigma + A(1-\alpha^2)\gamma^2} + \frac{A^4(1-\alpha^2)\gamma^2}{2\beta}.
    \end{align*}
    The second inequality follows from the Lipschtiz continuity of $x \mapsto \frac{x}{\sqrt{x^2+\beta}}$ with constant $\frac{1}{\sqrt\beta}$. The third inequality follows from the mean value theorem applied to $x\mapsto\frac{1}{\sqrt x}$.
\end{proof}

\subsubsection{SpecGF with $\cT$}
We note that the dynamics exist globally (from boundedness of velocity by $\cT$). In addition, for this rank-1 case, SpecGF with $\cT$ is identical to the normalized gradient descent:
\begin{align*}
    \cT(\nabla_\bA \cL) = \frac{\nabla_\bA \cL}{\normF{\nabla_\bA \cL}}, \quad \cT(\nabla_\bB \cL) = \frac{\nabla_\bB \cL}{\normF{\nabla_\bB \cL}}.
\end{align*}
Just as SpecGF with $\cT_\beta$, we sample a random unit vector $\bw$ in $\sR^n$, then $\alpha \triangleq \angleb{\bv, \bw} \ne 0$ with probability 1.

\begin{assumption}
    $\alpha \ne 0$.
\end{assumption}

For now, assume $\alpha > 0$; we will justify it shortly after. A typical LoRA training would start at
\begin{equation*}
    \bA_0 = \bm 0, \quad \bB_0 = \gamma \bw^\top,
\end{equation*}
for sufficiently small $\gamma > 0$ with gradients
\begin{align*}
    \nabla_\bA \cL &= (\bA\bB - \bY)\bB^\top,\\
    \nabla_\bB \cL &= \bA^\top(\bA\bB - \bY).
\end{align*}
However, $\bA$ being zero induces a zero gradient for $\bB$, an ill-defined situation for SpecGF. This motivates us to initialize the SpecGF at the next iteration of LoRA training with step size $\eta$:
\begin{equation*}
    \bA_1 = \eta \bu, \quad \bB_1 = \gamma \bw^\top.
\end{equation*}
Thus the trajectory is intact but with nonzero gradients. For simplicity, we take $\eta = \gamma$. SpecGF with $\cT$ starts at
\begin{equation}
    \bA(0) = \gamma \bu, \quad \bB(0) = \gamma \bw^\top.
\end{equation}

We first show that $\bA_t \in \mathrm{span}\{\bu\}$ and $\bB_t^\top \in \mathrm{span}\{\bv, \bw\}$ at all $t \geq 0$ for \textit{SpecGD} with $\cT$.
\begin{lemma}
    For all $t \geq 0$, $\bA_t \in \mathrm{span}\{\bu\}$ and $\bB_t^\top \in \mathrm{span}\{\bv, \bw\}$ hold.
\end{lemma}
\begin{proof}
    Assume that $\bA_t = a\bu$ and $\bB_t = b\bv^\top + c\bw^\top$ for some $a, b, c \in \sR$. We directly compute the gradients:
    \begin{align*}
        \nabla_\bA \cL(t) &= (\bA_t\bB_t - \bY)\bB_t^\top = \bu\roundb{(ab-\sigma)(b+c\alpha) + ac(b\alpha+c)},\\
        \nabla_\bB \cL(t) &= \bA_t^\top(\bA_t\bB_t - \bY) = a(ab-\sigma)\bv^\top + a^2c\bw^\top.
    \end{align*}
    Normalization only modifies the coefficients, so the updates $\Delta\bA_t \in \mathrm{span}\{\bu\}$ and $\Delta\bB_t^\top \in \mathrm{span}\{\bv, \bw\}$ hold. The proof is concluded as $\bA_0 \in \mathrm{span}\{\bu\}$ and $\bB_0^\top \in \mathrm{span}\{\bv, \bw\}$.
\end{proof}

Hence, for SpecGF with $\cT$, we assume that there exists a solution that is the continuous version of the solution for SpecGD with $\cT$; we study only this solution. Indeed, such $\bA(t)$ and $\bB(t)$ are the solutions for the rank-1 case with suitable coefficient functions. 
\begin{assumption}
    There exist scalars $a(t), b_0(t)$, and $c_0(t)$ in $\sR$ that are absolutely continuous and  $\bA(t) = a(t)\bu$ and $\bB(t) = b_0(t)\bv^\top + c_0(t)\bw^\top$.
\end{assumption}
We require absolute continuity to define the derivative at almost every $t$. In addition, we assume that the initialization scale $\gamma$ is sufficiently small.
\begin{assumption}\label{asm:alpha_small}
    $\gamma^2 < \alpha\sigma$.
\end{assumption}

Without loss of generality, assume $\alpha > 0$. If $\alpha < 0$, one can simply flip the sign of $\bu$ and $\bv$ and the argument analogously holds. If $\alpha = 1$, then $\bw = \bv$ and it is straightforward to see that
\begin{align*}
    \nabla_\bA \cL(t) &= (\bA(t)\bB(t) - \bY)\bB(t)^\top = \bu(a(t)b_0(t)-\sigma)b_0(t),\\
    \nabla_\bB \cL(t) &= \bA(t)^\top(\bA(t)\bB(t) - \bY) = a(t)(a(t)b_0(t)-\sigma)\bv^\top.
\end{align*}
Hence, if $\gamma^2 < \sigma$, then $a(t)$ and $b_0(t)$ get a unit update, which yields
\begin{equation}
    a(t) = \gamma + \max\{T_\gamma, t\}, \quad b_0(t) = \gamma + \max\{T_\gamma, t\}
\end{equation}
where $T_\gamma > 0$ satisfies $(T_\gamma + \gamma)^2 = \sigma$. Moreover, in this case, $a(t) = b_0(t)$ for all $t\geq 0$. Henceforth, consider the case $\alpha \in (0, 1)$.

For the convenience of later analysis, define
\begin{equation}
    \bz \triangleq \frac{1}{\sqrt{1-\alpha^2}}(\bw - \alpha\bv),
\end{equation}
which gives $\bw = \alpha \bv + \sqrt{1-\alpha^2}\bz$ and $\angleb{\bv, \bz} = 0$. Thus, $\bB(t) = (b_0(t) + \alpha c_0(t))\bv^\top + \sqrt{1-\alpha^2}c_0(t)\bz^\top$. For brevity, we write the parameters as
\begin{equation}
    \bA(t) = a(t)\bu, \quad \bB(t) \triangleq b(t)\bv^\top + c(t)\bz^\top.
\end{equation}
The goal is to show that 1) $a(t)b_0(t) \to \sigma$ \textit{with} neither $a(t) \to \infty$ nor $b_0(t) \to \infty$ and 2) $c_0(t) \to 0$; this is equivalent to showing that 1) $a(t)b(t) \to \sigma$ \textit{with} neither $a(t) \to \infty$ nor $b(t) \to \infty$ and 2) $c(t) \to 0$. The remaining part is to show that $a(t) \approx b(t)$; due to randomness in the initialization, $\dot a(t)$ cannot be exactly equal to $\dot { b}(t)$ even if we employ $\cT$. Instead, we show that $|a(t) - b(t)|$ and $|\dot a(t) - \dot{ b}(t)|$ are both uniformly small over $t$, and the gap vanishes to 0 as $\gamma \to 0$.

We derive the time derivative of the coefficients. From
\begin{align*}
    \cL(t) &= \frac{1}{2}\normF{\bA(t)\bB(t) - \bY}^2 = \frac{1}{2}\normF{\bu\roundb{(a(t)b(t) - \sigma)\bv^\top + a(t)c(t)\bz^\top}}^2\\
    &= \frac{1}{2}\norm{(a(t)b(t) - \sigma)\bv^\top + a(t)c(t)\bz^\top}_2^2 \triangleq \frac{1}{2}\norm{\bd(t)}_2^2,
\end{align*}
with $\bd(t) = a(t)\bB(t) - \sigma\bv^\top$. Whenever $\norm{\bd(t)}_2 = 0$, $a(t)b(t) = \sigma$ and $c(t) = 0$ and the dynamics terminate. From now on, the interval of time of our interest is $[0, T_\mathrm{max})$ where $T_\mathrm{max} = \inf \{t:\norm{\bd(t)}_2 = 0\}$; $T_\mathrm{max}$ could be possibly infinite.

We derive the loss gradient, with $m(t) \triangleq (a(t)b(t)-\sigma)b(t) + a(t)c(t)^2$,
\begin{align*}
    \nabla_\bA \cL(t) &= \bu\roundb{m(t)},\\
    \nabla_\bB \cL(t) &= a(t)\bd(t),
\end{align*}
and we obtain
\begin{align*}
    \dot a(t) &= -\mathrm{sgn}\roundb{m(t)}\\
    \dot { b}(t) &= -\mathrm{sgn}(a(t))\frac{a(t)b(t) - \sigma}{\norm{\bd(t)}_2}\\
    \dot { c}(t) &= -\mathrm{sgn}(a(t))\frac{a(t)c(t)}{\norm{\bd(t)}_2}.
\end{align*}
Since all the time derivatives are bounded by 1, due to $\cT$, we see that $a,  b,  c$ are all (Lipschitz) continuous in $t$. As $ c(0) = \sqrt{1-\alpha^2}\gamma > 0$, it directly follows that $c(t) \geq 0$ for all $t \geq 0$ and $c(t)$ is non-increasing. In addition, we have $a(0) = \gamma$ with $\dot a(0) = 1$ (from \Cref{asm:alpha_small}) and $ b(0) = \alpha\gamma > 0$. The simultaneous non-negativity of $a$ and $ b$ at $t = 0$ makes them keep non-negative for all $t \geq 0$ under continuity. If $ b$ hits 0 while $a(t) \geq 0$, it gains non-negative speed; same for $a$. The only way for them to be negative is that both $a$ and $ b$ go across 0 at the same time, but then the dynamics become stationary. Therefore, we can safely argue that $a(t), b(t), c(t) \geq 0$.

We return our attention to the loss function. We know that $\cL(t)$ is non-increasing, hence $\norm{\bd(t)}_2$ is also non-increasing. Precisely,
\begin{align*}
    \ddt \frac{1}{2}\norm{\bd(t)}_2^2 &= \angleb{\bd(t), \dot \bd(t)}\\
    &= \angleb{\bd(t), \dot a(t)\bB(t) + a(t)\dot\bB(t)}\\
    &= \dot a(t) \angleb{\bd(t), \bB(t)} + a(t) \angleb{\bd(t), \dot\bB(t)}\\
    &= -\mathrm{sgn}\roundb{m(t)}m(t) - a(t) \mathrm{sgn}(a(t))\frac{\norm{\bd(t)}_2^2}{\norm{\bd(t)}_2}\\
    &= -\abs{m(t)} - |a(t)|\norm{\bd(t)}_2 \leq 0.
\end{align*}

Notice that $\ddt \norm{\bd(0)}_2^2 < 0$. Hence, for $t > 0$, $\norm{\bd(t)}_2^2 < \norm{\bd(0)}_2^2$. Then, $\norm{\bd(t)}_2 < \sigma$ for $t > 0$ also implies that $a(t), b(t) > 0$; if any of two is 0 at $t$, $\norm{\bd(t)}_2 = \norm{(a(t)b(t) - \sigma)\bv^\top + a(t)c(t)\bz^\top}_2 \geq \sigma$.
Therefore, we may discard the sign of $a(t)$ in the time derivatives for $t > 0$:
\begin{align*}
    \dot a(t) &= -\mathrm{sgn}\roundb{m(t)}\\
    \dot { b}(t) &= -\frac{a(t)b(t) - \sigma}{\norm{\bd(t)}_2}\\
    \dot { c}(t) &= -\frac{a(t)c(t)}{\norm{\bd(t)}_2}.
\end{align*}

\begin{lemma}\label{lem:m_neg_T}
    $m(t) < 0$.
\end{lemma}
\begin{proof}
    Recall that $m(t) = (a(t)b(t)-\sigma)b(t) + a(t)c(t)^2$ and $m(0) = -\sigma b(0) < 0$; moreover, $m(t)$ is (bounded) Lipschitz continuous. If there exists some time $\tau \triangleq \min\curlyb{t: m(t) \geq 0}$, then $m(\tau) = 0$ and $\dot a(\tau) = 0$. In addition, from
    \begin{align*}
        m(t) &= \angleb{\bd(t), \bB(t)}\\
        &= a(t)\normF{\bB(t)}^2 - \sigma b(t).\\
        \dot m(t) &= \dot a(t)\normF{\bB(t)}^2 + a(t) \ddt\normF{\bB(t)}^2 - \sigma \dot { b}(t)\\
        &= \dot a(t)\normF{\bB(t)}^2 - 2a(t)\frac{m(t)}{\norm{\bd(t)}_2} - \sigma \dot { b}(t).
    \end{align*}
    Plugging $\tau$ vanishes the first two terms on the right-hand side, giving $\dot m(\tau) = -\sigma \dot { b}(\tau) < 0$, a contradiction: If $\dot { b}(\tau) = 0$, then we have $a(\tau) b(\tau) = \sigma$ and $ c(\tau) = 0$.
\end{proof}
\Cref{lem:m_neg_T} directly implies that $a(t)b_0(t) < \sigma$, thus $a(t)$ and $b_0(t)$ are strictly increasing and positive. Therefore, all the scalars $a(t), b_0(t)$, and $c_0(t)$ are bounded.

Let $\Delta(t) \triangleq a(t) - b(t)$. We will bound $|\Delta(t) - \Delta(0)|$ for $t > 0$. Recall that
\begin{equation*}
    \norm{\bd(t)}_2 = \sqrt{\roundb{a(t)b(t)-\sigma}^2 + \roundb{a(t)c(t)}^2}.
\end{equation*}
Then,
\begin{equation*}
    \dot \Delta(t) = \dot a(t) - \dot { b}(t) = 1 + \frac{a(t)b(t) - \sigma}{\norm{\bd(t)}_2} \geq 0.
\end{equation*}
Furthermore, $\norm{\bd(t)}_2 + a(t)b(t) - \sigma \leq a(t)c(t)$ yields
\begin{equation*}
    \dot \Delta(t) = 1 + \frac{a(t)b(t) - \sigma}{\norm{\bd(t)}_2} \leq -\dot { c}(t).
\end{equation*}
We see that $\abs*{\dot \Delta(t)} \leq -\dot { c}(t)$, hence the total variation of $\Delta$ is at most $ c(0)$, that is,
\begin{equation*}
    |\Delta(t) - \Delta(0)| \leq  c(0) = \gamma\sqrt{1-\alpha^2}.
\end{equation*}
Plugging $\Delta(0) = -\gamma\alpha$, we see that
\begin{equation}
    |\Delta(t)| \leq \gamma\alpha + \gamma\sqrt{1-\alpha^2} \triangleq C(\gamma).
\end{equation}

Observe that the difference between the ``singular values'' are uniformly small, and especially the difference vanishes to 0 as $\gamma \to 0$.
It remains to show that SpecGF with $\cT$ converges to the global minima. There exists $d_\infty \in [0, \sigma)$ such that $\norm{\bd(t)}_2 \to d_\infty$ as $t \to \infty$. If $d_\infty \ne 0$, we show that this leads to a contradiction combined with $\limsup \ddt \frac{1}{2}\norm{\bd(t)}_2^2 = 0$.

\begin{proposition}
    $d_\infty = 0$.
\end{proposition}
\begin{proof}
    Suppose not, that is, $d_\infty > 0$. Since $a,  b,  c$ are bounded Lipschitz continuous, so is $\norm{\bd(t)}_2$. Thus, $\norm{\bd(t)}_2$ is differentiable a.e., and there exists a sequence $\{t_n\}$ where
    \begin{equation*}
        \ddt \norm{\bd(t_n)}_2 \to 0 \quad \Rightarrow \quad \ddt \norm{\bd(t_n)}_2^2 \to 0.
    \end{equation*}
    From
    \begin{align*}
        \ddt \frac{1}{2}\norm{\bd(t_n)}_2^2 = -\abs{(a(t_n) b(t_n)-\sigma) b(t_n) + a(t_n) c(t_n)^2} - |a(t_n)|\norm{\bd(t_n)}_2 \leq 0,
    \end{align*}
    we obtain
    \begin{equation*}
        a(t_n),  b(t_n) \to 0.
    \end{equation*}
    However, this implies that $\norm{\bd(t_n)}_2 = \sqrt{\roundb{a(t_n) b(t_n)-\sigma}^2 + \roundb{a(t_n) c(t_n)}^2} \to \sigma$.
\end{proof}
Therefore, we see that $\norm{\bd(t)}_2 \to 0$, thus $a(t)b(t) \to \sigma$ and $c(t) \to 0$. This recovers the limiting behavior of $b_0(t) = b(t) - \alpha c_0(t)$ and $c_0(t) = \frac{1}{\sqrt{1-\alpha^2}}c(t)$:
\begin{equation}
    a(t)b_0(t) \to \sigma, \quad c_0(t) \to 0.
\end{equation}

\paragraph{Synchronized speed of singular values.} We know that $a(t)$ is bounded; denote the upper bound of $a(t)$ as $A > 0$. $A$ is computed as follows:
\begin{equation*}
    a(t)^2 - \Delta(t)a(t) = a(t)b(t) \leq \sigma,
\end{equation*}
and thus 
\begin{equation*}
    a(t)^2 \leq \sigma + |\Delta(t)|a(t) \leq \sigma + C(\gamma)a(t),
\end{equation*}
giving
\begin{equation*}
    A = \frac{C(\gamma) + \sqrt{C(\gamma)^2 + 4\sigma}}{2}.
\end{equation*}
Notice that $A \to \sqrt\sigma$ as $\gamma \to 0$.

\begin{proposition}
    \begin{equation*}
        \abs*{\dot a(t) - \dot { b}(t)} \leq \frac{A\gamma \sqrt{1-\alpha^2}}{2(\sigma - a(t)b(t))}.
    \end{equation*}
    Specifically, whenever
    \begin{equation*}
        \sigma - a(t)b(t) \geq K\sigma,
    \end{equation*}
    for some $0 < K < 1$, we have
    \begin{equation}
        \abs*{\dot a(t) - \dot { b}(t)} \leq \frac{A\sqrt{1-\alpha^2}}{2K\sigma}\gamma = \Theta(\gamma).
    \end{equation}
\end{proposition}
\begin{proof}
    Recall
    \begin{align*}
        \dot a(t) = 1, \quad \dot {b}(t) = -\frac{a(t)b(t) - \sigma}{\norm{\bd(t)}_2}.
    \end{align*}
    Then,
    \begin{align*}
        \abs*{\dot a(t) - \dot { b}(t)} &= 1 - \frac{\sigma - a(t)b(t)}{\norm{\bd(t)}_2}\\
        &= 1 - \frac{\sigma - a(t)b(t)}{\sqrt{\roundb{a(t)b(t)-\sigma}^2 + \roundb{a(t)c(t)}^2}}\\
        &= \frac{\roundb{a(t)c(t)}^2}{\sqrt{\roundb{a(t)b(t)-\sigma}^2 + \roundb{a(t)c(t)}^2}\roundb{\sqrt{\roundb{a(t)b(t)-\sigma}^2 + \roundb{a(t)c(t)}^2} + \sigma - a(t)b(t)}}\\
        &\leq \frac{\roundb{a(t)c(t)}^2}{a(t)c(t) 2\roundb{\sigma - a(t)b(t)}}\\
        &\leq \frac{A\gamma\sqrt{1-\alpha^2}}{2\roundb{\sigma - a(t)b(t)}}.
    \end{align*}
\end{proof}

\paragraph{Fast rotation of $\bB$ towards $\bv$.} We independently state that $\bB(t)$ quickly aligns with $\bv^\top$ at the early time. Let $\theta(t) \in (0, \frac\pi2)$ be the angle between $\bB(t)$ and $\bv$. Then $b(t) = \normF{\bB(t)}\cos\theta(t)$ and $c(t) = \normF{\bB(t)}\sin\theta(t)$, hence
\begin{equation*}
    \dot\theta(t) = \frac{b(t)\dot { c}(t) - \dot { b}(t)c(t)}{\normF{\bB(t)}^2} = \frac{-\sigma c(t)}{\norm{\bd(t)}_2\normF{\bB(t)}^2} = -\frac{\sigma \sin\theta(t)}{\norm{\bd(t)}_2\normF{\bB(t)}}.
\end{equation*}
From the fact that $\frac{\mathrm{d}}{\mathrm{d}\theta} \log \tan\!\roundb{\frac{\theta}{2}} = \csc \theta$ for $\theta \in (0, \pi)$, we see that
\begin{equation*}
    \ddt \log \tan\roundb{\frac{\theta(t)}{2}} = -\frac{\sigma}{\norm{\bd(t)}_2\normF{\bB(t)}} \leq -\frac{1}{\normF{\bB(t)}}.
\end{equation*}
Note that, at small $t > 0$, the fraction is extremely small, making $\theta$ quickly decrease towards 0.

\subsection{Spectral initialization Case}\label{sec:spec_init}

We first show a closed formula for singular values of $\bA(t)$ and $\bB(t)$ if the singular matrices of $\bY$ are directly applied at the initialization. Recall that $\bY = \bU_{r^\star}\bS\bV_{r^\star}^\top$ is the compact SVD of $\bY$. Let $\bU_r$ and $\bV_r$ consist of the first $r$ columns of $\bU_{r^\star}$ and $\bV_{r^\star}$, respectively. The spectral initialization~\citep{zhang2025lora} sets
\begin{equation*}
    \bA(0) = \bU_r\bm\Sigma_{\bA, 0}\bQ^\top, \quad \bB(0) = \bQ\bm\Sigma_{\bB, 0}\bV_r^\top,
\end{equation*}
where $\bm\Sigma_{\bA, 0}$ and $\bm\Sigma_{\bB, 0}$ are non-negative diagonal matrices. \citet{zhang2025lora} set the diagonal matrices to be the square root of the singular value matrix of the one-step gradient of full fine-tuning; any non-negative diagonal matrices suffice for our case.

\begin{proposition}
    Let
    \begin{equation*}
        \bA(0) = \bU_r\bm\Sigma_{\bA, 0}\bQ^\top, \quad \bB(0) = \bQ\bm\Sigma_{\bB, 0}\bV_r^\top,
    \end{equation*}
    where $\bm\Sigma_{\bA, 0}$ and $\bm\Sigma_{\bB, 0}$ are non-negative diagonal matrices with $\norm{\bm\Sigma_{\bA, 0}}_2, \norm{\bm\Sigma_{\bB, 0}}_2 < \sqrt{\sigma_\mathrm{min}(\bY)}$ and $\bQ\in\sR^{r\times r}$ is any orthogonal matrix. Then, for all $t \geq 0$, the SpecGF iteration with $\cT$ can be written as
    \begin{equation*}
        \bA(t) = \bU_r\bm\Sigma_A(t)\bQ^\top, \quad \bB(t) = \bQ\bm\Sigma_B(t)\bV_r^\top
    \end{equation*}
    with $\bm\Sigma_A(0) = \bm\Sigma_{\bA, 0}$ and $\bm\Sigma_B(0) = \bm\Sigma_{\bB, 0}$.
    Moreover, $\bm\Sigma_A(t)\bm\Sigma_B(t) \to \bSigma_r$ and $\dot\bSigma_A(t) = \dot\bSigma_B(t) = \bI_r$, where $\bSigma_r \in \sR^{r\times r}$ is the leading principle submatrix of $\bSigma$, unless $\bA(0)$ and $\bB(0)$ do not have zero diagonal values at the same coordinate.
\end{proposition}
\begin{proof}
    Let $\sigma_{\bA,i}(t)$ and $\sigma_{\bB,i}(t)$ be the $i$-th singular value of $\bm\Sigma_\bA(t)$ and $\bm\Sigma_\bB(t)$, respectively. Similarly, denote the $i$-th singular value of $\bS$ as $\sigma_i$. If
    \begin{equation*}
        \bA(t) = \bU_r\bm\Sigma_A(t)\bQ^\top, \quad \bB(t) = \bQ\bm\Sigma_B(t)\bV_r^\top,
    \end{equation*}
    then
    \begin{align*}
        \nabla_\bA \cL(t) &= (\bA(t)\bB(t) - \bY)\bB(t)^\top = \bU_r(\bm\Sigma_A(t)\bm\Sigma_B(t) - \bSigma_r)\bm\Sigma_B(t)^\top\bQ^\top,\\
        \nabla_\bB \cL(t) &= \bA(t)^\top(\bA(t)\bB(t) - \bY) = \bQ\bm\Sigma_A(t)^\top(\bm\Sigma_A(t)\bm\Sigma_B(t) - \bSigma_r)\bV_r^\top.
    \end{align*}
    Thus,
    \begin{align*}
        \cT(\nabla_\bA \cL(t)) &= \bU_r\bm\Gamma_A(t)\bQ^\top = \bU_r\Diag\{\mathrm{sgn}((\sigma_{\bA,i}(t)\sigma_{\bB,i}(t) - \sigma_i)\sigma_{\bB,i}(t))\}\bQ^\top,\\
        \cT(\nabla_\bB \cL(t)) &= \bQ\bm\Gamma_B(t)\bV_r^\top = \bQ\Diag\{\mathrm{sgn}((\sigma_{\bA,i}(t)\sigma_{\bB,i}(t) - \sigma_i)\sigma_{\bA,i}(t))\}\bV_r^\top,
    \end{align*}
    for $i\in[r]$, where $\bm\Gamma_A(t)$ and $\bm\Gamma_B(t)$ are the rectangular diagonal matrices. Hence, only the singular values of $\bA(t)$ and $\bB(t)$ are updated, and we can write
    \begin{equation*}
        \bm\Sigma_A(0) = \bm\Sigma_{\bA, 0}, \quad \bm\Sigma_B(0) = \bm\Sigma_{\bB, 0}.
    \end{equation*}
    
    Since the initialization is small, we see that all the signs in $\bm\Gamma_A(t)$ and $\bm\Gamma_B(t)$ are negative before saturation, or $\sigma_{\bA,i}(t)\sigma_{\bB,i}(t) = \sigma_i$. Therefore, all the singular values have a speed of 1 until the saturation, and the solution is
    \begin{align*}
        \bm\Sigma_\bA(t) &= \Diag\curlyb{\sigma_{\bA,i}(0)+\max\{t, T_i\}}_i\\
        \bm\Sigma_\bB(t) &= \Diag\curlyb{\sigma_{\bB,i}(0)+\max\{t, T_i\}}_i,
    \end{align*}
    where for each $i$, $T_i > 0$ satisfies
    \begin{equation*}
        \roundb{\sigma_{\bA,i}(0)+T_i}\roundb{\sigma_{\bB,i}(0)+T_i} = \sigma_i. \qedhere
    \end{equation*}
\end{proof}
This implies that if $\bSigma$ is already in non-increasing order, then SpecGF with $\cT$ converges to the best rank-$r$ approximation of $\bY$. We have a similar result for SpecGF with $\cT_\beta$.

\begin{proposition}\label{prop:specgf_rk1_lora}
    Let
    \begin{equation*}
        \bA(0) = \bU_r\bm\Sigma_{\bA, 0}\bQ^\top, \quad \bB(0) = \bQ\bm\Sigma_{\bB, 0}\bV_r^\top,
    \end{equation*}
    where $\bm\Sigma_{\bA, 0}$ and $\bm\Sigma_{\bB, 0}$ are positive invertible diagonal matrices with $\norm{\bm\Sigma_{\bA, 0}}_2, \norm{\bm\Sigma_{\bB, 0}}_2 < \gamma < \sqrt{\sigma_\mathrm{min}(\bY)}$ for $\gamma > 0$ and $\bQ\in\sR^{r\times r}$ is any orthogonal matrix. Let $\sigma_{\bA,i}$ and $\sigma_{\bB,i}$ be the $i$-th singular value of $\bm\Sigma_{\bA, 0}$ and $\bm\Sigma_{\bB, 0}$, respectively. Assume that for all $i\in[r]$,
    \begin{equation*}
        \quad \abs*{\log \frac{\sigma_{\bA,i}}{\sigma_{\bB,i}}} < \gamma.
    \end{equation*}
    Then, for all $t \geq 0$, the SpecGF iteration with $\cT_\beta$ can be written as
    \begin{equation*}
        \bA(t) = \bU_r\bm\Sigma_A(t)\bQ^\top, \quad \bB(t) = \bQ\bm\Sigma_B(t)\bV_r^\top
    \end{equation*}
    with $\bm\Sigma_A(0) = \bm\Sigma_{\bA, 0}$ and $\bm\Sigma_B(0) = \bm\Sigma_{\bB, 0}$.
    Moreover, $\bm\Sigma_A(t)\bm\Sigma_B(t) \to \bSigma_r$, where $\bSigma_r \in \sR^{r\times r}$ is the leading principle submatrix of $\bSigma$, unless $\bA(0)$ and $\bB(0)$ do not have zero diagonal values at the same coordinate. In addition, both $\abs*{\bSigma_A(t) - \bSigma_B(t)}$ and $\abs*{\dot \bSigma_A(t) - \dot \bSigma_B(t)}$ go to 0 as $\gamma \to 0$.
\end{proposition}
\begin{proof}
    We note that the dynamics exist uniquely (from analyticity of $\cT_\beta$) and globally (from boundedness of velocity by $\cT_\beta$). Let $\sigma_{\bA,i}(t)$ and $\sigma_{\bB,i}(t)$ be the $i$-th singular value of $\bm\Sigma_{\bA}(t)$ and $\bm\Sigma_{\bB,}(t)$, respectively. Denote the $i$-th singular value of $\bS$ as $\sigma_i$. If
    \begin{equation*}
        \bA(t) = \bU_r\bm\Sigma_A(t)\bQ^\top, \quad \bB(t) = \bQ\bm\Sigma_B(t)\bV_r^\top,
    \end{equation*}
    then
    \begin{align*}
        \nabla_\bA \cL(t) &= (\bA(t)\bB(t) - \bY)\bB(t)^\top = \bU_r(\bm\Sigma_A(t)\bm\Sigma_B(t) - \bSigma_r)\bm\Sigma_B(t)^\top\bQ^\top,\\
        \nabla_\bB \cL(t) &= \bA(t)^\top(\bA(t)\bB(t) - \bY) = \bQ\bm\Sigma_A(t)^\top(\bm\Sigma_A(t)\bm\Sigma_B(t) - \bSigma_r)\bV_r^\top.
    \end{align*}
    Thus,
    \begin{align*}
        \cT_\beta(\nabla_\bA \cL(t)) &= \bU_r\bm\Gamma_A(t)\bQ^\top = \bU_r\Diag\{f_\beta((\sigma_{\bA,i}(t)\sigma_{\bB,i}(t) - \sigma_i)\sigma_{\bB,i}(t))\}\bQ^\top,\\
        \cT_\beta(\nabla_\bB \cL(t)) &= \bQ\bm\Gamma_B(t)\bV_r^\top = \bQ\Diag\{f_\beta((\sigma_{\bA,i}(t)\sigma_{\bB,i}(t) - \sigma_i)\sigma_{\bA,i}(t))\}\bV_r^\top,
    \end{align*}
    for $i\in[r]$, where $\bm\Gamma_A(t)$ and $\bm\Gamma_B(t)$ are the rectangular diagonal matrices and
    \begin{equation*}
        f_\beta(x) \triangleq \frac{x}{\sqrt{x^2+\beta}}.
    \end{equation*}
    Hence, only the singular values of $\bA(t)$ and $\bB(t)$ are updated, and we can write
    \begin{equation*}
        \bm\Sigma_A(0) = \bm\Sigma_{\bA, 0}, \quad \bm\Sigma_B(0) = \bm\Sigma_{\bB, 0}.
    \end{equation*}
    Notice that the overall dynamics are decoupled into $r$ different scalar dynamics:
    \begin{align*}
        \dot \sigma_{\bA,i}(t) &= -f_\beta((\sigma_{\bA,i}(t)\sigma_{\bB,i}(t) - \sigma_i)\sigma_{\bB,i}(t)), \\
        \dot \sigma_{\bB,i}(t) &= -f_\beta((\sigma_{\bA,i}(t)\sigma_{\bB,i}(t) - \sigma_i)\sigma_{\bA,i}(t)).
    \end{align*}
    So it suffices to demonstrate about one $i \in [r]$; for simplicity, we fix $i$ and simplify the notation:
    \begin{equation*}
        \sigma_{\bA,i}(t) \triangleq a(t), \quad \sigma_{\bB,i}(t) \triangleq b(t), \quad \sigma_i \triangleq \sigma
    \end{equation*}
    and rewrite the dynamics:
    \begin{align*}
        \dot a(t) &= -f_\beta((a(t)b(t) - \sigma)b(t)), \\
        \dot b(t) &= -f_\beta((a(t)b(t) - \sigma)a(t)).
    \end{align*}

    We first note an immediate consequence from the dynamics: $a(t) \geq 0, b(t) \geq 0$ always hold. If it were violated, then $a$ and $b$ must be both zero simultaneously, but it implies that the dynamics terminates.

    Another consequence is that $a(t)b(t) \leq \sigma$ and non-decreasing from the update of $a(t)b(t)$ itself:
    \begin{equation*}
        \ddt a(t)b(t) = -(a(t)b(t) - \sigma)\roundb{\frac{a(t)^2}{\sqrt{(a(t)b(t) - \sigma)^2a(t)^2+\beta}} + \frac{b(t)^2}{\sqrt{(a(t)b(t) - \sigma)^2b(t)^2+\beta}}}.
    \end{equation*}
    It is straightforward that $a(t)b(t)$ is non-decreasing as long as $a(t)b(t) \leq \sigma$. Thus, there exists $\sigma_\infty \in [0, \sigma]$ such that $a(t)b(t) \to \sigma_\infty$. Moreover, $a(t)$ and $b(t)$ are strictly increasing on $t>0$ as long as $a(t)b(t) < \sigma$; the problematic case is $t= 0$, but, without loss of generality, if $a(0) = 0$ then $b(0) > 0$ thus $\dot a(0) > 0$. Hence, $a(t), b(t) > 0$ for all $t > 0$, which yields strictly monotonic evolution. Therefore, both $a(t)$ and $b(t)$ are bounded and attain their respective forward limit, namely $a_\infty > 0$ and $b_\infty > 0$.
    
    Then $a_\infty b_\infty = \sigma_\infty > 0$ with $\lim_{t\to\infty}\dot a(t) = -f_\beta((\sigma_\infty - \sigma)b_\infty) = 0$ since $f_\beta''(x)$ is bounded, which implies $\sigma_\infty = \sigma$.

    Albeit the global convergence is proved, a closed form for the solution is not generally obtained. As a roundabout, we show that if $a(0), b(0) \leq \gamma$ then $a(t) \approx b(t)$ and $\dot a(t) \approx \dot b(t)$ as $\gamma \to 0$.

    Let $d(t) \triangleq a(t) - b(t)$. Then
    \begin{align*}
        \dot d(t) &= f_\beta((a(t)b(t) - \sigma)a(t)) -f_\beta((a(t)b(t) - \sigma)b(t))\\
        &\leq \frac{1}{\sqrt{\beta}}(a(t)b(t) - \sigma)(a(t) - b(t))\\
        &= \frac{1}{\sqrt{\beta}}(a(t)b(t) - \sigma)d(t).
    \end{align*}
    The inequality holds as $f_\beta$ is $\frac{1}{\sqrt{\beta}}$-Lipschitz. Since $a(t)b(t) \leq \sigma$ for all $t$, $|d(t)| \leq |d(0)| \leq \gamma$.

    We now bound $v(t) \triangleq \dot a(t) - \dot b(t)$. From above,
    \begin{equation*}
        |\dot v(t)| \leq \frac{1}{\sqrt{\beta}}|(a(t)b(t) - \sigma)d(t)| \leq \frac{\sigma }{\sqrt{\beta}}\gamma. \qedhere
    \end{equation*}
\end{proof}
\newpage
\section{Uniform Growth of Singular Values}\label{sec:uniform_growth}

This section proves the uniform growth results stated in \Cref{thm:uniform_growth_main}.

\subsection{Setup and Notation}

\paragraph{$\beta$-regularized orthogonalization.}
Fix $\beta > 0$.
For $\bG \in \sR^{m \times r}$ define
\begin{align*}
\cT_{\mathrm{col},\beta}(\bG) \coloneqq \bG(\bG^\top \bG + \beta \bI)^{-1/2},
\end{align*}
and for $\bG \in \sR^{r \times n}$ define
\begin{align*}
\cT_{\mathrm{row},\beta}(\bG) \coloneqq (\bG\bG^\top + \beta \bI)^{-1/2}\bG.
\end{align*}
The case $\beta = 0$ (unregularized SpecGF) can be handled by a limiting argument; we focus on $\beta > 0$ throughout.

\begin{lemma}[Operator norm bound]\label{lem:T_norm_bound}
For all $\beta > 0$ and all $\bG$,
\begin{align*}
\norm{\cT_{\mathrm{col},\beta}(\bG)}_2 \le 1,
\qquad
\norm{\cT_{\mathrm{row},\beta}(\bG)}_2 \le 1.
\end{align*}
\end{lemma}
\begin{proof}
Let $s$ be any singular value of $\bG$. Then $s$ becomes $s/\sqrt{s^2+\beta} \le 1$ under either normalization, hence the operator norm is $\le 1$.
\end{proof}

\begin{lemma}[Invariance]\label{lem:invariance}
    If $\bA(0) = \mathbf{0}$, then $\bA(t) \in \mathrm{col}(\bU_r)$ for all $t \geq 0$. Consequently, $\bA(t) = \bU_r\bX(t)$ for some $\bX(t) \in \sR^{r \times r}$.
\end{lemma}
\begin{proof}
    The gradient with respect to $\bA$ is $\nabla_\bA \cL = (\bA\bB - \bY)\bB^\top$. At $t=0$, we have $\nabla_\bA \cL(0) = -\bY\bB(0)^\top = -\bU_r\bSigma\bV_r^\top\bB(0)^\top$, which lies in $\mathrm{col}(\bU_r)$. Since $\cT_\beta$ (and $\cT$) preserves the column space, $\dot\bA(0) \in \mathrm{col}(\bU_r)$.

    By induction, suppose $\bA(t) \in \mathrm{col}(\bU_r)$, i.e., $\bA(t) = \bU_r\bX(t)$ for some $\bX(t) \in \sR^{r \times r}$. Then
    \begin{align*}
        \nabla_\bA \cL(t) &= (\bU_r\bX(t)\bB(t) - \bU_r\bSigma\bV_r^\top)\bB(t)^\top = \bU_r(\bX(t)\bB(t)\bB(t)^\top - \bSigma\bV_r^\top\bB(t)^\top).
    \end{align*}
    This again lies in $\mathrm{col}(\bU_r)$, so $\dot\bA(t) = -\cT_\beta(\nabla_\bA \cL(t)) \in \mathrm{col}(\bU_r)$. By continuity of the flow, $\bA(t) \in \mathrm{col}(\bU_r)$ for all $t \geq 0$.
\end{proof}

\paragraph{Core variables.}
Let $\bY = \bU_r \bSigma \bV_r^\top$ be the compact SVD with $\bSigma = \Diag(\sigma_1, \ldots, \sigma_r)$ and $\sigma_1 > \cdots > \sigma_r > 0$, where $\bU_r \in \sR^{m \times r}$ and $\bV_r \in \sR^{n \times r}$. Let $\bV_\perp \in \sR^{n \times (n-r)}$ denote the orthogonal complement of $\bV_r$.
We define the core variables:
\begin{align*}
\bX(t) \coloneqq \bU_r^\top \bA(t) \in \sR^{r \times r}, \quad
\bZ(t) \coloneqq \bB(t)\bV_r \in \sR^{r \times r}, \quad
\bZ_\perp(t) \coloneqq \bB(t)\bV_\perp \in \sR^{r \times (n-r)}.
\end{align*}
The core product is $\bG(t) \coloneqq \bX(t)\bZ(t)$. For a square matrix $\bM$, we write $\Off(\bM) \coloneqq \bM - \Diag(\diag(\bM))$ for its off-diagonal part.

We introduce two tolerances:
\begin{itemize}
    \item \textbf{Alignment tolerance $\delta > 0$:} Time $t$ is $\delta$-aligned if $\normF{\Off(\bG(t))} \le \delta$ and $\normF{\bX(t)\bZ_\perp(t)} \le \delta$.
    \item \textbf{Target tolerance $\varepsilon > 0$:} We denote $d_i(t) \coloneqq [\bG(t)]_{ii}$, $e_i(t) \coloneqq d_i(t) - \sigma_i$, and the active set $\cI_\varepsilon(t) \coloneqq \{i \in [r] : |e_i(t)| > \varepsilon\}$.
\end{itemize}
A mode $i$ is \emph{active} if $i \in \cI_\varepsilon(t)$. Let $T \coloneqq \inf\{t \ge \tau : \cI_\varepsilon(t) = \emptyset\}$ be the termination time.

\subsection{Alignment from Small Initialization}

\begin{proposition}[High-probability alignment]\label{prop:alignment_O_gamma_app}
Let $T$ be a termination time above. 
Under Gaussian initialization $\bA(0) = \mathbf{0}$, $\bB(0) = \gamma \bN$ with $N_{ij} \stackrel{\mathrm{i.i.d.}}{\sim} \mathcal{N}(0,1)$,
there exist constants $C_1, C_2 < \infty$ (depending on $\bY, r, n, T$) such that,
with probability at least $1 - \exp(-c_1 r^2) - \exp(-c_1 rn)$ for some constant $c_1 > 0$,
\begin{align*}
    \sup_{t \le T} \normF{\Off(\bG(t))} \le C_1 \gamma,
    \qquad
    \sup_{t \le T} \normF{\bX(t)\bZ_\perp(t)} \le C_2 \gamma.
\end{align*}
Consequently, the trajectory is $O(\gamma)$-aligned for all $t \le T$.
\end{proposition}

\begin{proof}
Both bounds follow from Lipschitz stability near the diagonal manifold $\cM \coloneqq \{(\bX, \bZ, \bZ_\perp) : \bX, \bZ \text{ diagonal}, \bZ_\perp = 0\}$. Let $\bS(t) \coloneqq (\bX(t), \bZ(t), \bZ_\perp(t))$ and consider the orthogonal projection onto $\cM$:
\begin{align*}
    \Pi(\bX, \bZ, \bZ_\perp) \coloneqq \big(\Diag(\diag(\bX)), \Diag(\diag(\bZ)), \mathbf{0}\big).
\end{align*}
Let $\bS^\star(t) \coloneqq (\bX^\star(t), \bZ^\star(t), \bZ^\star_\perp(t))$ be the trajectory of SpecGF with $\cT_\beta$ trajectory, started at $\bS^\star(0) \coloneqq \Pi(\bS(0))$.
By diagonal manifold invariance, $\bS^\star(t) \in \cM$ for all $t$, hence $\bZ^\star_\perp(t) \equiv \mathbf{0}$ and, writing $\bG^\star(t) \coloneqq \bX^\star(t)\bZ^\star(t)$, $\Off(\bG^\star(t)) \equiv \mathbf{0}$.

By speed bounds and Gaussian concentration, both trajectories remain in a compact set $\cK$ for $t \in [0,T]$.
Since SpecGF with $\cT_\beta$ defines a $\cC^1$ vector field, finite-horizon Lipschitz stability gives
\begin{align*}
    \normF{\bS(t) - \bS^\star(t)} \le e^{LT} \normF{\bS(0) - \bS^\star(0)}
\end{align*}
for some finite $L = L(\cK)$.

Note that $\normF{\bS(0) - \bS^\star(0)} \le \normF{\Off(\bZ(0))} + \normF{\bZ_\perp(0)}$ because $\bX(0) = \mathbf{0}$.
Since $\bZ(0) = \gamma \bN\bV_r$ and $\bZ_\perp(0) = \gamma \bN\bV_\perp$ where $\bN$ has i.i.d.\ $\cN(0,1)$ entries, and $\bV_r$, $\bV_\perp$ have orthonormal columns, the entries of $\bN\bV_r$ and $\bN\bV_\perp$ are also i.i.d.\ $\cN(0,1)$.
Thus $\normF{\Off(\bN\bV_r)}^2 \sim \chi^2_{r(r-1)}$ and $\normF{\bN\bV_\perp}^2 \sim \chi^2_{r(n-r)}$.
By chi-squared concentration,
\begin{align*}
\mathbb{P}\big(\normF{\Off(\bZ(0))} + \normF{\bZ_\perp(0)} > C\gamma(r + \sqrt{rn})\big) \le \exp(-c_1 r^2) + \exp(-c_1 rn)
\end{align*}
for some absolute constants $C, c_1 > 0$.
Consequently, $\sup_{t \le T} \normF{\bS(t) - \bS^\star(t)} \le \kappa \gamma$ for some finite $\kappa > 0$ with probability at least $1 - \exp(-c_1 r^2) - \exp(-c_1 rn)$.

Since $\bZ^\star_\perp(t) \equiv \mathbf{0}$, it follows that 
\begin{align*}
    \normF{\bZ_\perp(t)} = \normF{\bZ_\perp(t) - \bZ^\star_\perp(t)} \le \normF{\bS(t) - \bS^\star(t)} \le \kappa \gamma.
\end{align*}
By speed bounds $\norm{\bX(t)}_2 \le t \le T$, hence $\sup_{t \le T} \normF{\bX(t)\bZ_\perp(t)} \le T \kappa \gamma \coloneqq C_2 \gamma$.

For the off-diagonal bound, using $\bG = \bX\bZ$ and $\bG^\star = \bX^\star\bZ^\star$, we have
\begin{align*}
\bG - \bG^\star = \bX\bZ - \bX^\star\bZ^\star = (\bX - \bX^\star)\bZ + \bX^\star(\bZ - \bZ^\star).
\end{align*}
Since both trajectories remain in the compact set $\cK$ for $t \in [0,T]$, there exists $M = M(\cK) < \infty$ such that $\|\bZ(t)\|_2, \|\bX^\star(t)\|_2 \le M$.
Therefore, we have 
\begin{align*}
    \normF{\bG(t) - \bG^\star(t)} &\le \normF{\bX(t) - \bX^\star(t)} \norm{\bZ(t)}_2 + \norm{\bX^\star(t)}_2 \normF{\bZ(t) - \bZ^\star(t)} \\
    &\le M \big( \normF{\bX(t) - \bX^\star(t)} + \normF{\bZ(t) - \bZ^\star(t)} \big) \\
    &\le M\sqrt{2} \normF{\bS(t) - \bS^\star(t)}.
\end{align*}
Since $\Off(\bG^\star(t)) \equiv \mathbf{0}$, we obtain
\begin{align*}
    \normF{\Off(\bG(t))} = \normF{\Off(\bG(t) - \bG^\star(t))} \le \normF{\bG(t) - \bG^\star(t)} \le M\sqrt{2} \normF{\bS(t) - \bS^\star(t)}.
\end{align*}
Combining these bounds yields $\sup_{t \le T} \normF{\Off(\bG(t))} \le C_1 \gamma$.
\end{proof}

\begin{lemma}[Diagonal approximation]\label{lem:diagonal_approx_app}
Suppose $\normF{\Off(\bG(t))} \le \delta$ and $\normF{\bX(t)\bZ_\perp(t)} \le \delta$. Then there exists a permutation $\pi(t)$ such that
\begin{align*}
	|\sigma_{\pi_i(t)}(\bA(t)\bB(t)) - d_i(t)| = O(\delta)
\end{align*}
for all $i \in [r]$, where $\pi_i(t)$ denotes the index satisfying $d_i(t) \approx \sigma_{\pi_i(t)}(\bA(t)\bB(t))$.
\end{lemma}

\begin{proof}
We have
\begin{align*}
\bA\bB = \bU_r\bX\bZ\bV_r^\top + \bU_r\bX\bZ_\perp\bV_\perp^\top.
\end{align*}
Since $\bU_r, \bV_r, \bV_\perp$ have orthonormal columns, the singular values of $\bA\bB$ equal those of $[\bG \mid \bX\bZ_\perp]$.
We write $\bD = \Diag(d_1, \ldots, d_r)$. The singular values of $[\bD \mid \mathbf{0}]$ are exactly $\{|d_1|, \ldots, |d_r|\}$, and
\begin{align*}
\|[\bG \mid \bX\bZ_\perp] - [\bD \mid \mathbf{0}]\|_F \le \|\Off(\bG)\|_F + \|\bX\bZ_\perp\|_F \le 2\delta.
\end{align*}
By Weyl's inequality (or Mirsky), there exists a permutation $\pi$ of $[r]$ such that
\begin{align*}
|\sigma_{\pi(i)}([\bG \mid \bX\bZ_\perp]) - |d_i|| \le 2\delta \quad \text{for all } i \in [r].
\end{align*}
Since $d_i > 0$ on our trajectory, we conclude $|\sigma_{\pi(i)}(\bA\bB) - d_i| = O(\delta)$.
\end{proof}

\subsection{Initial Growth and Non-degeneracy}

\begin{lemma}[Initial growth]\label{lem:initial_growth_app}
Let $\bZ(0) = \gamma \bM$ where $\bM \coloneqq \bN\bV_r \in \sR^{r \times r}$, with $\bN$ having i.i.d.\ $\mathcal{N}(0,1)$ entries and $\beta = O(\gamma^2)$.
Write $\bM = [\bc_1 \mid \cdots \mid \bc_r]$ where $\bc_i \in \sR^r$ denotes the $i$-th column.
Define the high-probability event
\begin{align*}
	\cE \coloneqq \left\{ \|\bM\|_2 \le C_3\sqrt{r}, \quad \min_{i \in [r]} \|\bc_i\|^2 \ge c_1 r \right\}.
\end{align*}
Then $\mathbb{P}(\cE) \ge 1 - e^{-c_1 r}$ for some absolute constants $C_3, c_1 > 0$. Let $\tau \asymp \frac{\sigma_r}{\sigma_1\sqrt{r}}$. On the event $\cE$,
\begin{align*}
	d_{\min}(\tau) \ge c_2 \tau \gamma
\end{align*}
for some constant $c_2 > 0$.
Moreover, all modes are active at $\tau$: $\cI_\varepsilon(\tau) = [r]$ and $d_i(\tau) < \sigma_i$ for all $i$.
\end{lemma}

\begin{proof}
For an $r \times r$ Gaussian matrix $\bM$, standard concentration gives $\|\bM\|_2 \le 3\sqrt{r}$ with probability at least $1 - 2e^{-r/2}$.
Each column $\bc_i \sim \mathcal{N}(0, \bI_r)$ satisfies $\|\bc_i\|^2 \sim \chi^2_r$, so $\mathbb{P}(\|\bc_i\|^2 < \frac r2) \le e^{-c_1 r}$ for some absolute constant $c_1 > 0$.
By a union bound, $\mathbb{P}(\cE) \ge 1 - e^{-c_1 r}$.

We now bound $[\dot{\bX}(0)\bZ(0)]_{ii}$ on $\cE$; note that $\dot \bZ(0) = \bm 0$.
From the dynamics of SpecGF with $\cT_\beta$ at $t = 0$ with $\bA(0) = \mathbf{0}$ and $\bZ(0) = \gamma \bM$:
\begin{align*}
\dot{\bX}(0) = \cT_{\mathrm{col},\beta}(\gamma \bSigma \bM^\top) = \gamma \bSigma \bM^\top (\gamma^2 \bM \bSigma^2 \bM^\top + \beta \bI)^{-1/2}.
\end{align*}
Define the positive definite matrix $\bQ \coloneqq \gamma^2 \bM \bSigma^2 \bM^\top + \beta \bI$. Then
\begin{align*}
\dot{\bX}(0) \bZ(0) = \gamma^2 \bSigma \bM^\top \bQ^{-1/2} \bM,
\end{align*}
whose $i$-th diagonal entry equals $\gamma^2 \sigma_i \cdot \bc_i^\top \bQ^{-1/2} \bc_i$.
Since $\bQ^{-1/2}$ is positive definite with minimum eigenvalue $\lambda_{\max}(\bQ)^{-1/2}$,
\begin{align*}
\bc_i^\top \bQ^{-1/2} \bc_i \ge \frac{\|\bc_i\|^2}{\sqrt{\lambda_{\max}(\bQ)}}.
\end{align*}
On $\cE$, we have $\lambda_{\max}(\bQ) \le \gamma^2 \sigma_1^2 \|\bM\|_2^2 + \beta \le C_3^2 \gamma^2 \sigma_1^2 r + \beta$.
Assuming $\beta \le C\gamma^2 \sigma_1^2 r$ for some absolute constant $C > 0$, this gives $\lambda_{\max}(\bQ) \le C'\gamma^2 \sigma_1^2 r$, and hence
\begin{align*}
\sqrt{\lambda_{\max}(\bQ)} \le \sqrt{C'} \cdot \gamma \sigma_1 \sqrt{r}.
\end{align*}
Combined with $\|\bc_i\|^2 \ge c_1 r$, we obtain
\begin{align*}
[\dot{\bX}(0) \bZ(0)]_{ii} \ge \gamma^2 \sigma_r \cdot \frac{c_1 r}{\sqrt{C'} \cdot \gamma \sigma_1 \sqrt{r}} = \frac{c_1 \gamma \sigma_r \sqrt{r}}{\sqrt{C'} \sigma_1}.
\end{align*}

Finally, we apply Taylor expansion.
For $\beta > 0$, the vector field is $C^1$, so for sufficiently small $\tau$:
\begin{align*}
\bX(\tau) = \tau \dot{\bX}(0) + O(\tau^2), \qquad
\bZ(\tau) = \bZ(0) + O(\tau^2),
\end{align*}
where the second equality uses $\dot{\bZ}(0) = \mathbf{0}$ (because $\bX(0) = \mathbf{0}$).
Expanding $\bG(\tau) = \bX(\tau)\bZ(\tau)$,
\begin{align*}
d_{\min}(\tau) \ge \tau \cdot \min_i [\dot{\bX}(0)\bZ(0)]_{ii} - O(\tau^2 \gamma) \asymp \tau \gamma.
\end{align*}
For the second claim, note that $d_i(\tau) \asymp \tau\gamma \ll \sigma_r - \varepsilon$ for sufficiently small $\gamma$. Hence, it holds that $e_i(\tau) = d_i(\tau) - \sigma_i < -\varepsilon$ for all $i \in [r]$, so all modes are active (i.e., early phase of training).
\end{proof}

\subsection{Invertibility and Near-Diagonal Factors}

\begin{lemma}[Invertibility under alignment and non-degeneracy]\label{lem:invertibility_app}
Suppose $\normF{\Off(\bG(t))} \le \delta$ and $d_{\min}(t) > 0$.
If $\delta < d_{\min}(t)$, then both $\bX(t)$ and $\bZ(t)$ are invertible.
\end{lemma}

\begin{proof}
Let $\bD(t) \coloneqq \Diag(d_1(t), \ldots, d_r(t))$.
Since $d_{\min} > 0$, the diagonal matrix $\bD$ is invertible. Write $\bG = \bD(\bI + \bD^{-1}\Off(\bG))$.
We have
\begin{align*}
\|\bD^{-1}\Off(\bG)\|_2 \le \|\bD^{-1}\|_2 \|\Off(\bG)\|_2 \le \frac{1}{d_{\min}} \cdot \delta = \frac{\delta}{d_{\min}} < 1.
\end{align*}
By the Neumann series, $\bI + \bD^{-1}\Off(\bG)$ is invertible, hence $\bG = \bD(\bI + \bD^{-1}\Off(\bG))$ is invertible.
Thus, $\det(\bG) = \det(\bX)\det(\bZ) \ne 0$ implies both factors are invertible.
\end{proof}

\begin{lemma}[Near-diagonal factors]\label{lem:off_diag_XZ}
Under the hypotheses of \Cref{prop:alignment_O_gamma_app}, it follows that
\begin{align*}
	\sup_{t \le T} \normF{\Off(\bX(t))} \le C_4 \gamma,
	\qquad
	\sup_{t \le T} \normF{\Off(\bZ(t))} \le C_5 \gamma.
\end{align*}
Moreover, on any interval where $d_{\min}(t) \ge d_0 > 0$, the diagonal entries satisfy
\begin{align*}
	|[\bX(t)]_{ii} - \widetilde{s}_i(t)| \le C_6 \Big(\gamma + \frac{\gamma}{\sqrt{d_0}}\Big), \qquad |[\bZ(t)]_{ii} - \widetilde{s}_i(t)| \le C_6 \Big(\gamma + \frac{\gamma}{\sqrt{d_0}}\Big),
\end{align*}
where $\widetilde{s}_i(t) \coloneqq \mathrm{sgn}(\xi_i) \sqrt{d_i(t)}$ is the signed square root with $\xi_i \sim \mathcal{N}(0,1)$, $\kappa$ is the Lipschitz constant from \Cref{prop:alignment_O_gamma_app}, and $C_6$ depends only on $\kappa$.
\end{lemma}

\begin{proof}
By Lipschitz stability (\Cref{prop:alignment_O_gamma_app}), $\|\bX(t) - \bX^\star(t)\|_F \le \kappa\gamma$ where $\bX^\star(t) \in \cM$ is diagonal. Since $\Off(\bX^\star) = \mathbf{0}$,
\begin{align*}
	\|\Off(\bX)\|_F = \|\Off(\bX - \bX^\star)\|_F \le \|\bX - \bX^\star\|_F \le \kappa\gamma.
\end{align*}
The same bound holds for $\bZ$ similarly. For the diagonal entries, we let
\begin{align*}
	x_i \coloneqq [\bX]_{ii}, \quad z_i \coloneqq [\bZ]_{ii}, \quad x_i^\star \coloneqq [\bX^\star]_{ii}, \quad z_i^\star \coloneqq [\bZ^\star]_{ii}.
\end{align*}
The reference trajectory starts from $x_i^\star(0) = 0$ and $z_i^\star(0) = \gamma \xi_i$. Since $\xi_i \sim \mathcal{N}(0,1)$, Gaussian concentration gives $|\xi_i| = O(1)$ with high probability, so it holds that $|x_i^\star(0) - z_i^\star(0)| = O(\gamma)$.

On diagonal manifold $\cM$, the SpecGF dynamics for mode $i$ are
\begin{align*}
\dot{x}_i^\star = -\frac{e_i^\star z_i^\star}{\sqrt{(e_i^\star z_i^\star)^2 + \beta}}, \qquad \dot{z}_i^\star = -\frac{e_i^\star x_i^\star}{\sqrt{(e_i^\star x_i^\star)^2 + \beta}},
\end{align*}
where $e_i^\star(t) = x_i^\star(t) z_i^\star(t) - \sigma_i$. Define the strictly increasing function $g(t) \coloneqq \frac{t}{\sqrt{t^2 + \beta}}$. The imbalance $\delta_i^\star \coloneqq x_i^\star - z_i^\star$ satisfies
\begin{align*}
\dot{\delta}_i^\star = g(e_i^\star x_i^\star) - g(e_i^\star z_i^\star).
\end{align*}
Since $e_i^\star x_i^\star - e_i^\star z_i^\star = e_i^\star \delta_i^\star$, for $e_i^\star < 0$, the sign of $(e_i^\star x_i^\star) - (e_i^\star z_i^\star)$ is opposite to that of $\delta_i^\star$. By the monotonicity of $g$, this implies $\delta_i^\star \cdot \dot{\delta}_i^\star \le 0$, hence $|\delta_i^\star(t)|$ is non-increasing.

We now show that $e_i^\star(t) \le 0$ for all $t \ge 0$. At $t = 0$, we have $e_i^\star(0) = x_i^\star(0)z_i^\star(0) - \sigma_i = -\sigma_i < 0$. If $e_i^\star(t_0) = 0$ for some $t_0 > 0$, then $\dot{x}_i^\star(t_0) = \dot{z}_i^\star(t_0) = 0$. By ODE uniqueness (due to locally Lipschitz), $(x_i^\star, z_i^\star)$ remains constant for $t \ge t_0$, so $e_i^\star(t) = 0$ for all $t \ge t_0$. Thus, $e_i^\star(t)$ cannot become positive, and $e_i^\star(t) \le 0$ for all $t$.

Combining the above, $\delta_i^\star \cdot \dot{\delta}_i^\star \le 0$ holds for all $t$, and thus
\begin{align*}
	|x_i^\star(t) - z_i^\star(t)| \le |x_i^\star(0) - z_i^\star(0)| = O(\gamma) \quad \text{for all } t.
\end{align*}

Now, we define
\begin{align*}
\widetilde{s}_i^\star \coloneqq \mathrm{sgn}(\xi_i) \sqrt{d_i^\star}, \qquad \widetilde{s}_i \coloneqq \mathrm{sgn}(\xi_i) \sqrt{d_i}.
\end{align*}
We claim that $x_i^\star, z_i^\star, \widetilde{s}_i^\star$ all share the same sign as $\xi_i$. To see this, note that $z_i^\star(0) = \gamma\xi_i$ has the sign of $\xi_i$, while $x_i^\star(0) = 0$. For $e_i^\star \le 0$, the dynamics give
\begin{align*}
\dot{x}_i^\star = -\frac{e_i^\star z_i^\star}{\sqrt{(e_i^\star z_i^\star)^2 + \beta}},
\end{align*}
which has the same sign as $z_i^\star$. Thus, $x_i^\star$ immediately enters the same sign as $z_i^\star$ and cannot cross zero thereafter (since $\dot{x}_i^\star$ would push it back). Similarly, $\dot{z}_i^\star$ has the same sign as $x_i^\star$, so $z_i^\star$ also maintains its sign.

\medskip
Since $\|\bG - \bG^\star\|_F \le \kappa\gamma$, we have $|d_i - d_i^\star| \le \kappa\gamma$. On intervals where $d_{\min}(t) \ge d_0$, 
\begin{align*}
	|\widetilde{s}_i - \widetilde{s}_i^\star| = |\sqrt{d_i} - \sqrt{d_i^\star}| = \frac{|d_i - d_i^\star|}{\sqrt{d_i} + \sqrt{d_i^\star}} \le \frac{\kappa\gamma}{\sqrt{d_i}} \le \frac{\kappa\gamma}{\sqrt{d_0}}.
\end{align*}
Since $x_i^\star$ and $\widetilde{s}_i^\star$ have the same sign, we obtain 
\begin{align*}
	|x_i^\star - \widetilde{s}_i^\star|
	= \Big||x_i^\star| - \sqrt{|x_i^\star||z_i^\star|}\Big|	\le |x_i^\star - z_i^\star| = O(\gamma).
\end{align*}
Combining with $|x_i - x_i^\star| \le \|\bX - \bX^\star\|_F \le \kappa\gamma$, the triangle inequality gives
\begin{align*}
|x_i - \widetilde{s}_i| \le |x_i - x_i^\star| + |x_i^\star - \widetilde{s}_i^\star| + |\widetilde{s}_i^\star - \widetilde{s}_i| \le C_6 \Big(\gamma + \frac{\gamma}{\sqrt{d_0}}\Big),
\end{align*}
where $C_6$ depends only on $\kappa$. The same argument gives $|z_i - \widetilde{s}_i| \le C_6(\gamma + \gamma/\sqrt{d_0})$.
\end{proof}

\subsection{Square-Root Dynamics}

\begin{lemma}[Square-root dynamics]\label{lem:sqrt_dynamics_app}
Suppose $\delta$-alignment holds with $\delta = O(\gamma)$, $\gamma = \Theta(\varepsilon^{1/2})$, and $\beta = O(\varepsilon^3)$.
Let $[\tau',T']$ be any subinterval of $[\tau,T]$
on which
\begin{align*}
    d_{\min}(t) \coloneqq \min_{i\in[r]} d_i(t) \ge d_0 = \Theta(\gamma).
\end{align*}
Then for each active mode $i$ (i.e., $e_i(t)\le -\varepsilon$), the square-root coordinate
$s_i(t)\coloneqq \sqrt{d_i(t)}$ satisfies, for a.e.\ $t\in[\tau',T']$,
\begin{equation}\label{eq:sqrt_dyn_main}
\dot s_i(t)
= -\mathrm{sgn}(e_i(t))\,\alpha_i(t) + R_i(t),
\qquad
\alpha_i(t) \coloneqq \frac{|e_i(t)|\,s_i(t)}{\sqrt{e_i(t)^2 s_i(t)^2+\beta}} \in (0,1].
\end{equation}
Moreover, the remainder obeys
\begin{equation}\label{eq:sqrt_dyn_rem}
|R_i(t)| \le C_7 \varepsilon^{1/4},
\end{equation}
where $C_7=C_7(\bY,r)$ depends only on constants from \Cref{lem:off_diag_XZ} (and earlier constants).

In particular, for every active mode $i$ (i.e., $e_i(t)\le -\varepsilon$),
\begin{align*}
\dot s_i(t) = 1 + O(\varepsilon^{1/4}).
\end{align*}
\end{lemma}

\begin{proof}
Since $d_i(t)=[\bX(t)\bZ(t)]_{ii}$ and $s_i(t)=\sqrt{d_i(t)}$, the chain rule gives, for a.e.\ $t$ with $d_i(t)>0$,
\begin{align}\label{eq:d_dot_chain}
\dot d_i(t) = [\dot \bX(t)\bZ(t)]_{ii} + [\bX(t)\dot \bZ(t)]_{ii},
\qquad
\dot s_i(t) = \frac{\dot d_i(t)}{2s_i(t)}.
\end{align}
Let $\bE(t)\coloneqq \bG(t)-\bSigma$ and $\bG_\bX \coloneqq \bE \bZ^\top + \bX\bZ_\perp \bZ_\perp^\top = \bU_r^\top \nabla_\bA \cL$.
By \Cref{lem:off_diag_XZ}, $\bX$ and $\bZ$ are nearly diagonal:
\begin{align*}
\bX=\Diag(x_i)+\bm{\Delta}_X,
\qquad
\bZ=\Diag(z_i)+\bm{\Delta}_Z,
\end{align*}
where $x_i,z_i=\mathrm{sgn}(\xi_i)s_i+O(\gamma + \gamma/\sqrt{d_0})$ and $\|\bm{\Delta}_X\|_F,\|\bm{\Delta}_Z\|_F=O(\gamma)$. The SpecGF dynamics give
\begin{align*}
\dot \bX = -\cT_{\mathrm{col},\beta}(\bG_\bX),
\qquad
\dot \bZ = -\cT_{\mathrm{row},\beta}(\nabla_\bB \cL)\,\bV_r.
\end{align*}
Since $\|\bX\bZ_\perp\|_F=O(\delta)$ (by $\delta$-alignment), we have $\bG_\bX = \bE\bZ^\top + O(\delta)$.
Combined with $\bE=\Diag(e_i)+O(\delta)$ and $z_i=\mathrm{sgn}(\xi_i)s_i+O(\gamma)$, this yields
\begin{align*}
\bG_\bX = \Diag(e_i z_i) + O(\delta+\gamma),
\qquad
\bG_\bX^\top \bG_\bX = \Diag(e_i^2 s_i^2)+O(\delta+\gamma).
\end{align*}
Since $\cT_{\mathrm{col},\beta}(\bG)=\bG\,(\bG^\top \bG+\beta \bI)^{-1/2}$, we obtain
\begin{align}\label{eq:Xdot_diag}
[\dot \bX]_{ii}
= -\frac{e_i z_i}{\sqrt{e_i^2 s_i^2+\beta}} + O(\delta+\gamma)
= -\frac{e_i\,\mathrm{sgn}(\xi_i)s_i}{\sqrt{e_i^2 s_i^2+\beta}} + O(\delta+\gamma).
\end{align}

For $\dot \bZ$, write $\nabla_\bB \cL = \bX^\top \bE \bV_r^\top + \bX^\top \bX \bZ_\perp \bV_\perp^\top$.
Using $\bV_r^\top \bV_r=\bI$ and $\bV_\perp^\top \bV_r=\mathbf{0}$, we have $\nabla_\bB \cL\,\bV_r = \bX^\top \bE$, so
\begin{align*}
\dot \bZ
= \dot \bB \bV_r
= -\cT_{\mathrm{row},\beta}(\nabla_\bB \cL)\,\bV_r
= -\bigl((\bX^\top \bE)(\bX^\top \bE)^\top+\beta \bI\bigr)^{-1/2}\,\bX^\top \bE.
\end{align*}
Moreover, $(\bX^\top \bE)(\bX^\top \bE)^\top = \bX^\top \bE \bE^\top \bX = \Diag(s_i^2 e_i^2)+O(\delta+\gamma)$, so
\begin{align}\label{eq:Zdot_diag}
[\dot \bZ]_{ii}
= -\frac{x_i e_i}{\sqrt{x_i^2 e_i^2+\beta}} + O(\delta+\gamma)
= -\frac{\mathrm{sgn}(\xi_i)s_i\,e_i}{\sqrt{e_i^2 s_i^2+\beta}} + O(\delta+\gamma).
\end{align}

Using \Cref{lem:off_diag_XZ} and \eqref{eq:Xdot_diag}--\eqref{eq:Zdot_diag}, we compute the diagonal entries of $\dot \bX \bZ$ and $\bX \dot \bZ$:
\begin{align}
[\dot \bX \bZ]_{ii}
&= [\dot \bX]_{ii}\,[\bZ]_{ii} + O(\|\dot \bX\|_F\|\Off(\bZ)\|_F)
= -\frac{e_i s_i^2}{\sqrt{e_i^2 s_i^2+\beta}} + O\!\Bigl(\delta+\gamma+\frac{\gamma}{\sqrt{d_0}}\Bigr),
\label{eq:XdZ_term}
\\[3pt]
[\bX\dot \bZ]_{ii}
&= [\bX]_{ii}\,[\dot \bZ]_{ii} + O(\|\bX\|_F\|\Off(\dot \bZ)\|_F)
= -\frac{e_i s_i^2}{\sqrt{e_i^2 s_i^2+\beta}} + O\!\Bigl(\delta+\gamma+\frac{\gamma}{\sqrt{d_0}}\Bigr).
\label{eq:XZdot_term}
\end{align}
Summing \eqref{eq:XdZ_term}--\eqref{eq:XZdot_term} and using \eqref{eq:d_dot_chain},
\begin{align*}
\dot d_i
= -\frac{2e_i s_i^2}{\sqrt{e_i^2 s_i^2+\beta}} + O\!\Bigl(\delta+\gamma+\frac{\gamma}{\sqrt{d_0}}\Bigr),
\qquad
\dot s_i
= -\frac{e_i s_i}{\sqrt{e_i^2 s_i^2+\beta}} + O\!\Bigl(\delta+\frac{\gamma}{\sqrt{d_0}}\Bigr),
\end{align*}
where we used $s_i\ge \sqrt{d_0}$ on $[\tau',T']$ to absorb the $O(\gamma)$ contribution into $O(\gamma/\sqrt{d_0})$ after dividing by $2s_i$.
Under the scaling $\gamma = \Theta(\varepsilon^{1/2})$, $d_0 = \Theta(\gamma)$, and $\delta = O(\gamma)$, the dominant error term is $\gamma/\sqrt{d_0} = O(\varepsilon^{1/4})$.

Defining $\alpha_i(t)\coloneqq\frac{|e_i(t)|\,s_i(t)}{\sqrt{e_i(t)^2 s_i(t)^2+\beta}} \in(0,1]$ and $R_i(t)\coloneqq\dot s_i(t) + \mathrm{sgn}(e_i(t))\,\alpha_i(t)$, we see that \eqref{eq:sqrt_dyn_main} holds and \eqref{eq:sqrt_dyn_rem} follows.

Finally, for active modes $e_i(t)\le -\varepsilon$, we have $\mathrm{sgn}(e_i)=-1$ and
\begin{align*}
\alpha_i
= \frac{|e_i|\,s_i}{\sqrt{e_i^2 s_i^2+\beta}}
= \Bigl(1+\frac{\beta}{e_i^2 s_i^2}\Bigr)^{-1/2}
\ge 1-\frac{\beta}{2\varepsilon^2 d_0}
= 1 - O(\varepsilon^{1/2}),
\end{align*}
using $\beta = O(\varepsilon^3)$ and $d_0 = \Theta(\varepsilon^{1/2})$. Since $\alpha_i \le 1$ trivially and $|R_i| = O(\varepsilon^{1/4})$, we conclude $\dot s_i = 1 + O(\varepsilon^{1/4})$ for active modes.\qedhere
\end{proof}

\subsection{Non-degeneracy Persistence}

\begin{lemma}[Non-degeneracy persistence]\label{lem:bootstrap_persistence_app}
Suppose $d_{\min}(\tau) \ge d_0 = \Theta(\gamma)$ with $\gamma = \Theta(\varepsilon^{1/2})$ and $\beta = O(\varepsilon^3)$.
Let $\underline{d} \coloneqq \min\{d_0, \sigma_r - \varepsilon\}$.
Then for all $t \in [\tau, T]$,
\begin{align*}
d_{\min}(t) \ge \underline{d} > 0.
\end{align*}
\end{lemma}

\begin{proof}
We define the exit time $T_{\mathrm{exit}} \coloneqq \inf\{t > \tau : d_{\min}(t) < \underline{d}\}$.
We show that $T_{\mathrm{exit}} > T$, which implies $d_{\min}(t) \ge \underline{d}$ for all $t \in [\tau, T]$. On $[\tau, T_{\mathrm{exit}})$, we have $d_{\min}(t) \ge \underline{d} > 0$ by definition, so \Cref{lem:sqrt_dynamics_app} applies.
For any active mode $i$ (i.e., $e_i(t) \le -\varepsilon$), we have $\dot{s}_i(t) = 1 + O(\varepsilon^{1/4}) > 0$ for small $\varepsilon$.
Hence, $s_i(t)$ is strictly increasing, and so is $d_i(t) = s_i(t)^2$.

For contradiction, suppose that some mode $i$ reaches $d_i(t) = \underline{d}$ from above. Then $e_i(t) = d_i(t) - \sigma_i \le \underline{d} - \sigma_r < -\varepsilon$ (since $\underline{d} \le \sigma_r - \varepsilon$), so $i$ is active and $\dot{d}_i(t) > 0$. This contradicts the decrease after reaching $\underline{d}$. By continuity, $d_{\min}(t)$ cannot reach $\underline{d}$ on $[\tau, T]$. It follows that $T_{\mathrm{exit}} > T$, and hence $d_{\min}(t) \ge \underline{d}$ for all $t \in [\tau, T]$.
\end{proof}

\subsection{Main Theorem: Uniform Growth}

\begin{theorem}[Uniform growth]\label{thm:uniform_growth_app}
Fix a small target tolerance $\varepsilon > 0$. Consider SpecGF with $\cT_\beta$ initialized at $\bA(0) = \mathbf{0}$, $\bB(0) = \gamma \bN$.
Let $\tau \asymp \frac{\sigma_r}{\sigma_1 \sqrt{r}}$ and let $T$ be the termination time.
For $\gamma = \Theta(\varepsilon^{1/2})$ and $\beta = O(\varepsilon^3)$, the following holds with high probability for all $t \in [\tau, T]$ and all active modes $i \in \cI_\varepsilon(t)$:
\begin{enumerate}[label=\textnormal{(\roman*)}]
    \item \textbf{Approximate unit speed:} $\displaystyle \frac{\rd}{\rd t} \sqrt{d_i(t)} = 1 + O(\varepsilon^{1/4})$.
    \item \textbf{Uniform growth:} For any two active modes $i, j \in \cI_\varepsilon(t)$,
    \begin{align*}
    \abs*{\frac{\rd}{\rd t}\sqrt{d_i(t)} - \frac{\rd}{\rd t}\sqrt{d_j(t)}} = O(\varepsilon^{1/4}).
    \end{align*}
\end{enumerate}
\end{theorem}

\begin{proof}
By \Cref{lem:initial_growth_app}, $d_{\min}(\tau) \ge d_0$ for $d_0 \asymp \tau\gamma = \Theta(\gamma) = \Theta(\varepsilon^{1/2})$ with high probability.
By \Cref{prop:alignment_O_gamma_app}, $\delta = O(\gamma) = O(\varepsilon^{1/2})$.
Applying \Cref{lem:sqrt_dynamics_app,lem:bootstrap_persistence_app}, every active mode satisfies $\dot{s}_i(t) = 1 + O(\varepsilon^{1/4})$.
This proves (i). For (ii), for any active $i, j$,
\begin{align*}
|\dot{s}_i(t) - \dot{s}_j(t)| \le |\dot{s}_i(t) - 1| + |\dot{s}_j(t) - 1| = O(\varepsilon^{1/4}),
\end{align*}
which concludes the proof.
\end{proof}

\section{Extension to Underparameterized Case ($r < r^\star$)}\label{sec:underparameterized}

We extend the uniform growth analysis to the underparameterized regime where the LoRA rank $r$ is smaller than the target rank $r^\star \coloneqq \mathrm{rank}(\bY)$. The key insight is that the column space invariance and diagonal manifold structure still hold within the relevant subspace.

\subsection{Setup and Extended Core Variables}

Let $\bY = \bU \bSigma \bV^\top$ be the compact SVD with $\bU \in \sR^{m \times r^\star}$, $\bV \in \sR^{n \times r^\star}$, and $\bSigma = \Diag(\sigma_1, \ldots, \sigma_{r^\star})$ with $\sigma_1 > \cdots > \sigma_{r^\star} > 0$. The LoRA factorization has rank $r \le r^\star$: $\bA \in \sR^{m \times r}$ and $\bB \in \sR^{r \times n}$.

We partition the target singular vectors as $\bU = [\bU_r \mid \bU_{>r}]$ and $\bV = [\bV_r \mid \bV_{>r}]$, where $\bU_r, \bV_r \in \sR^{\cdot \times r}$ correspond to the top $r$ singular values and $\bU_{>r}, \bV_{>r}$ correspond to the remaining $r^\star - r$ singular values. Let $\bV_\perp \in \sR^{n \times (n - r^\star)}$ denote the orthogonal complement of $\bV$ in $\sR^n$.

The extended core variables are:
\begin{align*}
    \bX_r & \coloneqq \bU_r^\top \bA \in \sR^{r \times r}, & \bX_{>r} &\coloneqq \bU_{>r}^\top \bA \in \sR^{(r^\star-r) \times r}, \\
    \bZ_r & \coloneqq \bB \bV_r \in \sR^{r \times r}, & \bZ_{>r} &\coloneqq \bB \bV_{>r} \in \sR^{r \times (r^\star-r)}, \\
    & & \bZ_\perp &\coloneqq \bB \bV_\perp \in \sR^{r \times (n-r^\star)}.
\end{align*}
In this setup, the core product is defined by $\bG_r \coloneqq \bX_r \bZ_r \in \sR^{r \times r}$. 

\medskip
Similar to $\delta$-alignment, we define \emph{extended $\delta$-alignment} by
\begin{align*}
    \normF{\Off(\bG_r(t))} \le \delta, \quad \normF{\bX_{>r}(t)} \le \delta, \quad \normF{\bZ_{>r}(t)} \le \delta, \quad \normF{\bX_r(t) \bZ_\perp(t)} \le \delta.
\end{align*}
This generalizes the original $\delta$-alignment condition to the underparameterized setting by additionally controlling the components outside the top-$r$ subspace. Note that, given a convergence to the best rank-$r$ approximation (cf. \Cref{prop:SSP_characterization,thm:SpecGF_avoids_SSP}), for any $\delta > 0$ we can always find some finite $t_0 > 0$ such that extended $\delta$-alignment holds for all $t \ge t_0$.

\subsection{Column Space Invariance}

\begin{lemma}[Column space invariance]\label{lem:col_space_invariance}
If $\bA(0) = \bm{0}$, then $\bA(t) \in \mathrm{col}(\bU)$ for all $t \ge 0$. Consequently, $\bA(t) = \bU_r \bX_r(t) + \bU_{>r} \bX_{>r}(t)$.
\end{lemma}

\begin{proof}
The gradient $\nabla_\bA \cL = (\bA\bB - \bY)\bB^\top$ lies in $\mathrm{col}(\bA) + \mathrm{col}(\bY) \subseteq \mathrm{col}(\bU)$ whenever $\bA \in \mathrm{col}(\bU)$. At $t = 0$, $\nabla_\bA \cL(0) = -\bY \bB(0)^\top \in \mathrm{col}(\bU)$. Since $\cT_\beta$ preserves column spaces, $\dot{\bA}(0) \in \mathrm{col}(\bU)$.
By induction and continuity of the flow, $\bA(t) \in \mathrm{col}(\bU)$ for all $t \ge 0$.
\end{proof}

\subsection{Invariant Manifold}

\begin{lemma}[Invariant manifold]\label{lem:invariant_manifold}
Under the initialization $\bA(0) = \bm{0}$ and $\bB(0) = \gamma \bN$ with $\bN$ having i.i.d.\ $\cN(0,1)$ entries, the extended manifold
\begin{align*}
\cM_{\mathrm{ext}} \coloneqq \big\{(\bX_r, \bX_{>r}, \bZ_r, \bZ_{>r}, \bZ_\perp) : \bX_{>r} = \bm{0}, \; \bZ_{>r} = \bm{0}, \; \bZ_\perp = \bm{0}\big\}
\end{align*}
is invariant under the SpecGF dynamics.
\end{lemma}

\begin{proof}
Let $\bSigma_r = \Diag(\sigma_1, \ldots, \sigma_r)$ and $\bSigma_{>r} = \Diag(\sigma_{r+1}, \ldots, \sigma_{r^\star})$. On $\cM_{\mathrm{ext}}$, we have $\bA = \bU_r \bX_r$ and $\bB = \bZ_r \bV_r^\top$, so
\begin{align*}
\bA\bB - \bY &= \bU_r \bX_r \bZ_r \bV_r^\top - \bU_r \bSigma_r \bV_r^\top - \bU_{>r} \bSigma_{>r} \bV_{>r}^\top \\
&= \bU_r (\bX_r \bZ_r - \bSigma_r) \bV_r^\top - \bU_{>r} \bSigma_{>r} \bV_{>r}^\top.
\end{align*}
The gradient $\nabla_\bA \cL = (\bA\bB - \bY)\bB^\top = \bU_r(\bX_r \bZ_r - \bSigma_r)\bZ_r^\top$ lies entirely in $\mathrm{col}(\bU_r)$, implying $\dot{\bX}_{>r} = \bU_{>r}^\top \dot{\bA} = \bm{0}$.

Similarly, $\nabla_\bB \cL = \bX_r^\top(\bX_r \bZ_r - \bSigma_r)\bV_r^\top - \bX_{>r}^\top \bSigma_{>r} \bV_{>r}^\top$. When $\bX_{>r} = \bm{0}$, this lies entirely in the row space of $\bV_r^\top$, so $\dot{\bZ}_{>r} = \dot{\bB} \bV_{>r} = \bm{0}$ and $\dot{\bZ}_\perp = \dot{\bB} \bV_\perp = \bm{0}$.
\end{proof}

\subsection{Lipschitz Stability near the Invariant Manifold}

\begin{proposition}[Alignment in the underparameterized case]\label{prop:alignment_underparameterized}
Under the initialization $\bA(0) = \bm{0}$ and $\bB(0) = \gamma \bN$, there exist constants $C_{>r}, C_\perp < \infty$ such that, with probability at least $1 - e^{-cr(r^\star-r)} - e^{-cr(n-r^\star)}$,
\begin{align*}
    \sup_{t \le T} \|\bX_{>r}(t)\|_F \le C_{>r} \gamma, \qquad
    \sup_{t \le T} \|\bZ_{>r}(t)\|_F \le C_{>r} \gamma, \qquad
    \sup_{t \le T} \|\bZ_\perp(t)\|_F \le C_\perp \gamma.
\end{align*}
\end{proposition}

\begin{proof}
Let $\bS(t) \coloneqq (\bX_r(t), \bX_{>r}(t), \bZ_r(t), \bZ_{>r}(t), \bZ_\perp(t))$ denote the true trajectory. At initialization with $\bA(0) = \bm{0}$ and $\bB(0) = \gamma \bN$, the core variables satisfy
\begin{align*}
\bX_r(0) = \bm{0}, \quad \bX_{>r}(0) = \bm{0}, \quad
\bZ_r(0) = \gamma \bN \bV_r, \quad \bZ_{>r}(0) = \gamma \bN \bV_{>r}, \quad \bZ_\perp(0) = \gamma \bN \bV_\perp.
\end{align*}
We define the projection onto $\cM_{\mathrm{ext}}$ by $\Pi(\bX_r, \bX_{>r}, \bZ_r, \bZ_{>r}, \bZ_\perp) \coloneqq (\bX_r, \bm{0}, \bZ_r, \bm{0}, \bm{0})$, and consider the reference trajectory $\bS^\star(t)$ initialized at $\bS^\star(0) = \Pi(\bS(0)) = (\bm{0}, \bm{0}, \gamma \bN \bV_r, \bm{0}, \bm{0})$. By \Cref{lem:invariant_manifold}, $\bS^\star(t) \in \cM_{\mathrm{ext}}$ for all $t \ge 0$.

The initial discrepancy is
\begin{align*}
\|\bS(0) - \bS^\star(0)\|_F = \|\bZ_{>r}(0)\|_F + \|\bZ_\perp(0)\|_F = O(\gamma \sqrt{r(n-r)})
\end{align*}
with high probability by Gaussian concentration.

The SpecGF vector field with $\cT_\beta$ is $\cC^1$, so finite-horizon Lipschitz stability gives
\begin{align*}
\|\bS(t) - \bS^\star(t)\|_F \le e^{LH} \|\bS(0) - \bS^\star(0)\|_F \le \kappa \gamma
\end{align*}
for some finite $\kappa = \kappa(\bY, r, T)$. Since $\bX^\star_{>r}(t) \equiv \bm{0}$, $\bZ^\star_{>r}(t) \equiv \bm{0}$, and $\bZ^\star_\perp(t) \equiv \bm{0}$, the bounds naturally hold.
\end{proof}

\subsection{Reduction to the Square Case}

On the invariant manifold $\cM_{\mathrm{ext}}$, the dynamics reduce to the $r \times r$ system:
\begin{align*}
\dot{\bX}_r &= -\cT_{\mathrm{col},\beta}\big((\bX_r \bZ_r - \bSigma_r)\bZ_r^\top\big), \\
\dot{\bZ}_r &= -\cT_{\mathrm{row},\beta}\big(\bX_r^\top(\bX_r \bZ_r - \bSigma_r)\big).
\end{align*}
This is equivalent to the original rank-$r^\star$ dynamics (in case of $r = r^\star$) with target $\bSigma_r = \Diag(\sigma_1, \ldots, \sigma_r)$.

By \Cref{prop:alignment_underparameterized}, perturbations from $\cM_{\mathrm{ext}}$ are controlled: $\|\bX_{>r}\|_F, \|\bZ_{>r}\|_F, \|\bZ_\perp\|_F = O(\gamma)$. These contribute only at most $O(\gamma)$ errors to the effective dynamics of $(\bX_r, \bZ_r)$.

\subsection{Extended Uniform Growth Theorem}

\begin{theorem}[Uniform growth for $r < r^\star$]\label{thm:uniform_growth_underparameterized}
Consider SpecGF with $\cT_\beta$ on the rank-$r$ factorization $\bA\bB$ approximating $\bY$ of rank $r^\star \ge r$. Let $\bG_r(t) \coloneqq \bX_r(t)\bZ_r(t)$ and $d_i(t) \coloneqq [\bG_r(t)]_{ii}$ for $i \in [r]$. Under the same scaling $\gamma = \Theta(\varepsilon^{1/2})$ and $\beta = O(\varepsilon^3)$, the following holds for all $t \in [\tau, T]$ and all active modes $i \in \cI_\varepsilon(t) \cap [r]$, with probability at least $1 - e^{-cr}$:
\begin{enumerate}[label=\textnormal{(\roman*)}]
    \item \textbf{Approximate unit speed:} $\displaystyle \frac{\rd}{\rd t}\sqrt{d_i(t)} = 1 + O(\varepsilon^{1/4})$.
    \item \textbf{Uniform growth:} For any two active modes $i, j \in \cI_\varepsilon(t) \cap [r]$,
    \begin{align*}
        \abs*{\frac{\rd}{\rd t}\sqrt{d_i(t)} - \frac{\rd}{\rd t}\sqrt{d_j(t)}} = O(\varepsilon^{1/4}).
    \end{align*}
\end{enumerate}
\end{theorem}

\begin{proof}
By \Cref{prop:alignment_underparameterized}, the trajectory remains $O(\gamma)$-close to $\cM_{\mathrm{ext}}$. On this manifold, the $(\bX_r, \bZ_r)$ dynamics are identical to the full-rank case with target $\bSigma_r$. Therefore, all lemmas from the full-rank analysis (\Cref{lem:initial_growth_app,lem:off_diag_XZ,lem:sqrt_dynamics_app,lem:bootstrap_persistence_app}) naturally hold with $r^\star$ replaced by $r$ and target singular values $(\sigma_1, \ldots, \sigma_r)$. The perturbations from $\bX_{>r}$, $\bZ_{>r}$, $\bZ_\perp$ contribute only at most $O(\gamma)$ errors, which are absorbed into the existing $O(\gamma)$ terms. The proof then follows \Cref{thm:uniform_growth_app}.
\end{proof}

\clearpage
\section{Convergence to Global Minima of SpecGF}

\subsection{Analyticity of SpecGF}\label{sec:Tb}

We first state that the inverse square root of symmetric positive definite (SPD) matrices is analytic. This directly follows from \citet[Lemma 4.4 (b)]{bauer2022smooth}.

\begin{lemma}\label{lem:ISRM_analytic}
    At any real SPD $\bS$, the function
    \begin{equation*}
        g(\bS) = \bS^{-\frac{1}{2}}
    \end{equation*}
    is analytic. That is, $g$ is analytic on the (open) set of the symmetric positive-definite matrices.
\end{lemma}
\begin{proof}
    Any SPD matrix $\bS$ is closed, invertible, (densely) defined linear operator on a Banach space (also a Hilbert space for our case), denoted as $X$, and has real eigenvalues. Throughout this proof, we abuse the notation $\sigma(\bS)$ to denote the spectrum of $\bS$.
    
    We first show $\bS$ is sectorial of any $\omega \in [0, \pi)$; fix any $\omega$. $\sigma(\bS)$ is exactly the set of eigenvalues of $\bS$, and therefore is contained in $\overline {S_\omega}$. Next, for each $\omega' \in (\omega, \pi)$, take any $\lambda \not\in \overline {S_{\omega'}}$. Then $\abs*{\mathrm{Arg}(\lambda)} > \omega'$ and thus, for any eigenvalue $s$ of $\bS$,
    \begin{equation*}
        |s - \lambda| \geq \mathrm{dist}(\lambda, (0, +\infty)) = \begin{dcases}
            |\lambda|\abs*{\sin(\mathrm{Arg}(\lambda))} \geq |\lambda|\sin\omega' & \text{if } \abs*{\mathrm{Arg}(\lambda)} \leq \frac{\pi}{2},\\
            |\lambda| \geq |\lambda|\sin\omega' & \text{if } \abs*{\mathrm{Arg}(\lambda)} > \frac{\pi}{2}.
        \end{dcases}
    \end{equation*}
    Hence,
    \begin{equation*}
        \norm{\lambda(\bS - \lambda)^{-1}} = \max_{s\in\sigma(\bS)} \frac{|\lambda|}{|s - \lambda|} \leq \frac{1}{\sin\omega'}
    \end{equation*}
    is uniformly bounded. On Hilbert space, R-boundedness is equivalent to boundedness \citep{bauer2022smooth}. Thus, $\bS$ is R-sectorial of any $\omega \in [0, \pi)$. We also see that $\bS$ admits a bounded $\cH^\infty(S_\omega)$ for any $\omega \in [0, \pi)$, because for $f \in \cH^\infty(S_\omega) \setminus \{0\}$,
    \begin{equation*}
        \norm{f(\bS)}_{L(X)} = \max_{s\in\sigma(\bS)} |f(s)| \leq \sup_f \; \norm{f}_{\cH^\infty(S_\omega)}
    \end{equation*}
    As $\bS$ is SPD, there exists a closed centered ball not in $\sigma(\bS)$. Now take $r = \frac{1}{2}$ in the statement of \citet[Lemma 4.4 (b)]{bauer2022smooth}. Then $g(\bS)$ is holomorphic, thus real-analytic restricted to the set of real matrices. As the dimension is finite, the fractional domain spaces $\dot X_p$ coincides with $X$ (viewed as sets) and thus $L(\dot X_p, \dot X_q) = L(X)$ for any real $p$ and $q$ (so are $L(\dot X_0, \dot X_{<r})$ and $L(\dot X_{>1-r}, \dot X_r)$).
\end{proof}

\begin{corollary}\label{cor:T_analytic}
Consider a matrix $\bX \in \sR^{m\times r}$ with $m \geq r$.
    \begin{itemize}
        \item[(a)] $\bX \mapsto (\bX^\top \bX)^{-\frac{1}{2}}$ is analytic on the (open) set of full-rank matrices.
        \item[(b)] For every $\beta > 0$,
        $\bX \mapsto (\bX^\top \bX + \beta \bI_r)^{-\frac{1}{2}}$ is analytic everywhere.
    \end{itemize}
    The above statements analogously hold for $\bX \in \sR^{r\times n}$ with $n \geq r$ by switching $\bX^\top \bX$ to $\bX\bX^\top$.
\end{corollary}
\begin{proof}
    This is apparent from \Cref{lem:ISRM_analytic}, the fact that the polynomials are analytic, and the composition of analytic functions is again analytic.
\end{proof}

Recall that, for a matrix $\bM \in \sR^{m\times n}$,
\begin{align*}
    \cT(\bM) &= \roundb{\roundb{\bM\bM^\top}^\dagger}^{1/2}\!\bM,\\
    \cT_\beta(\bM) &= \roundb{\bM\bM^\top + \beta \bI_m}^{-1/2}\!\bM,
\end{align*}

Let the compact SVD of $\bX\in \sR^{m\times n}$ be $\bU_r \bS \bV_r^\top$ with $\bS \triangleq \Diag\{\sigma_1, \ldots, \sigma_r\}$. $\sigma_i > 0$ are the positive singular values of $\bX$, $1\leq i \leq r \leq \min\{m, n\}$. Then,
\begin{align*}
    \cT(\bX) &= \bU_r\bV_r^\top,\\
    \cT_\beta(\bX) &= \bU_r \bS_\beta\bV_r^\top,
\end{align*}
where
\begin{equation*}
    \bS_\beta \triangleq \Diag\curlyb{\frac{\sigma_i}{\sqrt{\sigma_i^2+\beta}}}_{i\in[r]}.
\end{equation*}

\begin{proposition}[Analyticity of $\cT$ and $\cT_\beta$]\label{prop:specgf_analytic}
    SpecGF with $\cT_\beta$ is analytic everywhere. SpecGF with $\cT$ is analytic on the set of full-rank matrices. Moreover, the solution uniquely exists for all $t \geq 0$. 
\end{proposition}
\begin{proof}
    Analyticity is straightforward from \Cref{cor:T_analytic}, so we prove the existence and uniqueness part. For both $\cT$ and $\cT_\beta$, $\|\dot \bA\|_\mathrm{F}^2$ and $\|\dot \bB\|_\mathrm{F}^2$ are both bounded by $r$. Hence, the trajectory should satisfy
    \begin{equation*}
        \normF{(\bA(t), \bB(t))} \leq \normF{(\bA(0), \bB(0))} + \sqrt{2r}t.
    \end{equation*}
    If the maximal solution does not exist for all $t \geq 0$, then the iterate should leave every compact set of the domain. Therefore, the solution must exist for all $t \geq 0$. Uniqueness follows from the local Lipschitzness in $(\bA, \bG)$ of SpecGF (by Picard-Lindel\"of theorem).
\end{proof}

\subsection{Convergence Analysis of SpecGF}\label{sec:conv_specgf}

We restate the setup of SpecGF.
\begin{itemize}
    \item The target matrix $\bY \in \sR^{m\times n}$ yields a compact SVD $\bY = \bU_{r^\star}\bm\Sigma\bV_{r^\star}^\top$ where $\bm \Sigma = \Diag(\sigma_1, \sigma_2, \cdots, \sigma_{r^\star})$ with $\sigma_1 \geq \cdots \geq \sigma_{r^\star} > 0$.
    \item The loss function $\cL: \sR^{m\times r} \times \sR^{r\times n} \to \sR$ is defined as
    \begin{equation}
        \cL(\bA, \bB) \triangleq \frac{1}{2}\normF{\bA \bB - \bY}^2.
    \end{equation}
    We may abuse the notation and simply state $\cL(t)$.
    \item The time derivative of $\bA$ and $\bB$ is governed by either $\cT$ or $\cT_\beta$ applied to the gradient:
    \begin{equation*}
        \dot{\bA}(t) = -\cT\roundb{\nabla_\bA \cL(t)}, \quad \dot{\bB}(t) = -\cT\roundb{\nabla_\bB \cL(t)}.
    \end{equation*}
\end{itemize}

\medskip

\begin{lemma}\label{lem:SpecGF_T_dec}
    On SpecGF with $\cT$, the loss is monotonically non-increasing:
    \begin{align*}
        \frac{d \cL(t)}{\mathrm{d}t} \le 0.
    \end{align*}
    The equality holds if and only if $\nabla_\bA \cL(t) = \nabla_\bB \cL(t) = 0$ holds.
\end{lemma}

\begin{proof}
    We first show that for any matrix $\bX$, $\angleb{\cT(\bX), \bX} = \norm{\bX}_*$. Let the compact SVD of $\bX$ be $\bU'\bS'\bV'^\top$.
    \begin{align*}
        \angleb{\cT(\bX), \bX} = \mathrm{tr}\roundb{(\bU'\bV'^\top)^\top \bU'\bS'\bV'^\top} = \mathrm{tr}\roundb{\bV'\bS'\bV'^\top} = \mathrm{tr}\roundb{\bS'} = \norm{\bX}_*.
    \end{align*}
    
    We drop the argument $t$ for simplicity.
    \begin{align*}
        \frac{d \cL}{\mathrm{d}t} &= \angleb{\nabla_\bA \cL, \dot{\bA}} + \angleb{\nabla_\bB \cL, \dot\bB}\\
        &= \angleb{\nabla_\bA \cL, -\cT\roundb{\nabla_\bA \cL}} + \angleb{\nabla_\bB \cL, -\cT\roundb{\nabla_\bB \cL}}\\
        &= -\|\nabla_\bA \cL\|_* - \|\nabla_\bB \cL\|_* \le 0. \qedhere
    \end{align*}
\end{proof}

\begin{lemma}\label{lem:SpecGF_Tb_dec}
    On SpecGF with $\cT_\beta$, the loss is monotonically non-increasing:
    \begin{align*}
        \frac{d \cL(t)}{\mathrm{d}t} \le 0.
    \end{align*}
    The equality holds if and only if $\nabla_\bA \cL(t) = \nabla_\bB \cL(t) = 0$ holds.
\end{lemma}

\begin{proof}
    We first simplify $\angleb{\cT_\beta(\bX), \bX}$ for a matrix $\bX$. Let the SVD of $\bX$ be $\bU'\bS'\bV'^\top$ with singular values $\sigma_1, \ldots, \sigma_r \geq 0$.
    \begin{align*}
        \angleb{\cT_\beta(\bX), \bX} = \mathrm{tr}\roundb{(\bU'\bS'_\beta\bV'^\top)^\top \bU'\bS'\bV'^\top} = \mathrm{tr}\roundb{\bV'\bS'_\beta\bS'\bV'^\top} = \mathrm{tr}\roundb{\bS'_\beta\bS'} = \sum_{i\in[r]} \frac{\sigma_i^2}{\sqrt{\sigma_i^2+\beta}}.
    \end{align*}
    
    We drop the argument $t$ for simplicity. Let $\sigma_i(\nabla_\bA \cL), \sigma_i(\nabla_\bB \cL)$ be the singular values of $\nabla_\bA \cL$ and $\nabla_\bB \cL$, respectively. Then:
    \begin{align*}
        \frac{d \cL}{\mathrm{d}t} &= \angleb{\nabla_\bA \cL, \dot{\bA}} + \angleb{\nabla_\bB \cL, \dot\bB}\\
        &= \angleb{\nabla_\bA \cL, -\cT_\beta\roundb{\nabla_\bA \cL}} + \angleb{\nabla_\bB \cL, -\cT_\beta\roundb{\nabla_\bB \cL}}\\
        &= -\sum_{i\in[r]} \frac{\sigma_i(\nabla_\bA \cL)^2}{\sqrt{\sigma_i(\nabla_\bA \cL)^2+\beta}} - \sum_{i\in[r]} \frac{\sigma_i(\nabla_\bB \cL)^2}{\sqrt{\sigma_i(\nabla_\bB \cL)^2+\beta}} \le 0. \qedhere
    \end{align*}
\end{proof}

Thus, $\cL(t)$ is monotonically non-increasing and lower bounded, hence $\cL(\infty) \leq \cL(0)$ should exist. Moreover,
\begin{align}
    \normF{\bA(t)\bB(t)} \leq \normF{\bA(t)\bB(t) - \bY} + \normF{\bY} = 2\sqrt{\cL(t)} + \normF{\bY} \leq 2\sqrt{\cL(0)} + \normF{\bY}.
\end{align}

We now show that SpecGF iterate should either converge or diverge to infinity.
\begin{lemma}[\L{}ojasiewicz]\label{lem:Loja_inequ}
    Let $F:U \to \sR$ be a real analytic function on an open set $U\subset\sR^d$. Then, for every critical point $\bp\in U$ of $F$, there exists a (possibly small) neighborhood $W$ of $\bp$, constants $C > 0$ and $b \in [\frac12, 1)$ such that for all $\bx \in W$,
    \begin{equation}
        |F(\bx) - F(\bp)|^b \leq C\|\nabla F(\bx)\|.
    \end{equation}
\end{lemma}

\begin{proposition}\label{prop:infdiv_conv}
    Let $\bx(t)$ be a differentiable and absolutely continuous curve on $\sR^d$ and $F$ be a real analytic function on an open set of $\sR^d$. Assume that $F(\bx(t))$ is well-defined for all $t \geq 0$ and the condition
    \begin{equation}\label{cond:stationary}
        \frac{\mathrm{d}F(\bx(t))}{\mathrm{d}t} = 0 \quad \Rightarrow \quad \dot \bx(t) = 0.
    \end{equation}
    If there exists a constant $c > 0$ such that
    \begin{equation}\label{eq:F_dec}
        \frac{\mathrm{d}F(\bx(t))}{\mathrm{d}t} \leq -c\|\nabla F(\bx(t))\| \|\dot \bx(t)\|
    \end{equation}
    for all $t \geq 0$, then either $\bx(t)$ converges in $\sR^d$ or it diverges to infinity (in norm).
\end{proposition}
\begin{proof}
    Assume that $\|\bx(t)\| \not\to\infty$ as $t\to\infty$. Then $\bx(t)$ has an accumulation point in $\sR^n$; it suffices to prove that the accumulation point is the limit of $\bx(t)$. Denote the accumulation point as $\Bar \bx$.

    By continuity of $F$ and monotonicity of $F$ from \cref{eq:F_dec}, $F(\bx(t)) \downarrow F(\Bar{\bx})$. If there exists a finite time $t' \geq 0$ such that $F(\bx(t')) = F(\Bar{\bx})$, then $F(\bx(t)) = F(\Bar{\bx})$ for $t \geq t'$, hence $\bx(t) = \Bar{\bx}$ by \cref{cond:stationary}. From here we assume that $F(\bx(t)) > F(\Bar{\bx})$ for all $t$.

    From \cref{lem:Loja_inequ} and \cref{eq:F_dec}, we see that
    \begin{equation*}
        \frac{\mathrm{d}F(\bx(t))}{\mathrm{d}t} \leq -c\|\nabla F(\bx(t))\| \|\dot \bx(t)\| \leq -\frac{c}{C} |F(\bx(t)) - F(\Bar{\bx})|^b\|\dot \bx(t)\| = -\frac{c}{C} \roundb{F(\bx(t)) - F(\Bar{\bx})}^b\|\dot \bx(t)\|
    \end{equation*}
    on some neighborhood $W$ of $\Bar{\bx}$ with $C > 0$ and $b \in [\frac12, 1)$. Hence,
    \begin{align*}
        -\|\dot \bx(t)\| \geq \frac{C}{c}\frac{\mathrm{d}F(\bx(t))}{\mathrm{d}t}\roundb{F(\bx(t)) - F(\Bar{\bx})}^{-b} = \frac{C}{c(1-b)} \frac{\mathrm{d}}{\mathrm{d}t} \roundb{F(\bx(t)) - F(\Bar{\bx})}^{1-b}.
    \end{align*}

    For brevity, let $G(t) \triangleq F(\bx(t)) - F(\Bar{\bx}) > 0$. Then $G$ is absolutely continuous and $G(t) \downarrow 0$. We also see that
    \begin{equation}\label{eq:dotx_G}
        -\|\dot \bx(t)\| \geq \frac{C}{c(1-b)} \frac{\mathrm{d}}{\mathrm{d}t} \roundb{G(t)^{1-b}}.
    \end{equation}

    We now show that $\bx(t)$ eventually moves to the vicinity of $\Bar{\bx}$ and remains there forever. As $W$ is open, there exists $r > 0$ where
    \begin{equation*}
        B(r) \triangleq \{z\in\sR^n : \|z - \Bar{\bx}\| < r\} \subset W.
    \end{equation*}
    Fix any such $r > 0$. Because $\Bar{\bx}$ is the accumulation point of $\bx(t)$ and $G$ continuously decreases to zero, there exists a time $t_1 \geq 0$ such that
    \begin{equation*}
        \|x(t_1) - \Bar{\bx}\| < \frac{r}{2}, \quad G(t_1) < \roundb{\frac{rc(1-b)}{2C}}^{\frac{1}{1-b}}.
    \end{equation*}
    Notice the former implies $x(t_1) \in B(r)$ and the latter implies $\tfrac{C}{c(1-b)}G(t_1)^{1-b} < \tfrac{r}{2}$. Suppose the claim is false, so there exists $t_2 > t_1$ where $x(t_2) \not\in B(r)$; choose the smallest such $t_2$. Then $\bx(t) \in B(r)$ for $t \in [t_1, t_2)$, hence
    \begin{align*}
        \|x(t_2) - x(t_1)\| &\leq \int_{t_1}^{t_2} \|\dot \bx(t)\| \mathrm{d}t\\
        &\leq \int_{t_1}^{t_2} -\frac{C}{c(1-b)} \frac{\mathrm{d}}{\mathrm{d}t} \roundb{G(t)^{1-b}} \mathrm{d}t\\
        &= \frac{C}{c(1-b)} \squarb{G(t_1)^{1-b} - G(t_2)^{1-b}} < \frac{r}{2}.
    \end{align*}

    The first inequality and the equality follow from the absolute continuity of $\bx(t)$ and $G(t)$ on $[t_1, t_2]$, respectively. The second inequality stems from \cref{eq:dotx_G}. The above yields a contradiction:
    \begin{equation*}
        \|x(t_2) - \Bar{\bx}\| \leq \|x(t_2) - x(t_1)\| + \|x(t_1) - \Bar{\bx}\| < r \quad \Rightarrow \quad x(t_2) \in B(r).
    \end{equation*}
    Therefore, for all $t \geq t_1$, $\bx(t) \in B(r)$. As $r$ can be arbitrarily small, $\bx(t) \to \Bar{\bx}$ should hold. \qedhere
\end{proof}

It is evident that $\cL(\bA, \bB)$ is analytic everywhere. For $\cT_\beta$, $\bA(t)$ and $\bB(t)$ are continuously differentiable as their time derivative $-\cT_\beta(\nabla_\bA \cL(t))$ and $-\cT_\beta(\nabla_\bB \cL(t))$ are analytic. In addition, $\bM = 0 \Leftrightarrow \cT_\beta(\bM) = 0$, hence \cref{cond:stationary} is satisfied. Moreover, let the rank of $\nabla_\bA \cL(t)$ and $\nabla_\bB \cL(t)$ be $r_A$, $r_B$, respectively. Sort the singular values $\sigma(t)$ of the gradients with respect to $\bA$ and $\bB$ in non-increasing order. Then:
\begin{align*}
    \frac{d \cL(t)}{\mathrm{d}t} &= -\sum_{i\in[r_A + r_B]} \sigma_i(t)\frac{\sigma_i(t)}{\sqrt{\sigma_i(t)^2 + \beta}}\\
    &\leq -\frac{1}{r_A+r_B} \sum_{i\in[r_A + r_B]} \sigma_i(t) \sum_{i\in[r_A + r_B]}\frac{\sigma_i(t)}{\sqrt{\sigma_i(t)^2 + \beta}}\\
    &\leq -\frac{1}{r_A+r_B}\normF{\nabla \cL(t)}\normF{(\dot\bA, \dot\bB)}\\
    &\leq -\frac{1}{2r}\normF{\nabla \cL(t)}\normF{(\dot\bA, \dot\bB)}.
\end{align*}
The first inequality follows from Chebyshev's sum inequality. Hence, \cref{eq:F_dec} is also met. Therefore, the SpecGF with $\cT_\beta$ can only either converge to a critical point or diverge to infinity.

For $\cT$, assuming that iteration is always full rank (for the analyticity),
\begin{align*}
    \frac{d \cL(t)}{\mathrm{d}t} &= -\|\nabla_\bA \cL(t)\|_* - \|\nabla_\bB \cL(t)\|_*\\
    &\leq -\normF{\nabla_\bA \cL(t)} - \normF{\nabla_\bB \cL(t)} \leq -\normF{\nabla \cL(t)}\\
    &\leq -\frac{1}{\sqrt{2r}}\normF{\nabla \cL(t)}\normF{(\dot\bA, \dot\bB)}.
\end{align*}
The last inequality follows from the fact that $\normF{\cT(\bM)}^2 \leq \mathrm{rank}(\bM)$ for any matrix $\bM$.

\paragraph{Invariance.} We conjecture that SpecGF may not have a nontrivial invariant term with respect to $\bA(t)$ and $\bB(t)$, at least for polynomials, while plain GF has one. Here, the terminology ``invariant'' is a function $I: \sR^d \to \sR$ such that for all solutions $\bx(t)$ of the differential equation $\dot \bx(t) = f(\bx)$, $I(\bx(t))$ is constant (possibly depending on $\bx(0)$). Then, $\frac{\mathrm{d}}{\mathrm{d}t} I(\bx(t)) \equiv 0$, or 
\begin{equation*}
    \angleb{\nabla_\bx I, f(\bx)} \equiv 0
\end{equation*}
globally holds on the set of initial points where the solution is well-defined. If there exists an invariant polynomial in $\bA$ and $\bB$, it should work for the $m=n=1$ case. However, it can be shown that any polynomial should be identically constant on whole $\sR\times \sR$. Let $m=n=1=r$. Then $\bY = \sigma \ne 0$, $a(t), b(t) \in \sR$, and thus the loss becomes
\begin{equation*}
    \cL(a, b) = \frac{1}{2}(ab-\sigma)^2.
\end{equation*}
We state the partial derivatives for convenience:
\begin{equation*}
    \partial_a \cL = b(ab-\sigma), \quad \partial_b \cL = a(ab-\sigma).
\end{equation*}

\begin{proposition}[Absence of polynomial invariant for $\cT_\beta$]\label{prop:no_poly_invar_Tb}
     For SpecGF with $\cT_\beta$, there does not exist a non-constant invariant polynomial with real coefficients in $a$ and $b$ on all of $\:\sR \times \sR$.
\end{proposition}
\begin{proof}
    Note that, as $\sR$ is a field, $\sR[a][b] \cong \sR[a, b] \cong \sR[b][a]$. In addition, $\sR[x]$ is a unique factorization domain (UFD) for an indeterminate $x$.
    
    Let $P \in \sR[a, b]$ be a polynomial (of finite degree) with real coefficients. Then
    \begin{equation*}
        \partial_a P(a(t), b(t))\dot a(t) + \partial_b P(a(t), b(t))\dot b(t) = 0.
    \end{equation*}
    Substitute $\dot a(t) = -\cT_\beta(\partial_a \cL)$ and $\dot b(t) = -\cT_\beta(\partial_b \cL)$. Rearranging the equations gives that, for any $a$ and $b$,
    \begin{equation*}
        \partial_a P(a, b) \times b(ab-\sigma)\sqrt{a^2(ab-\sigma)^2 + \beta} + \partial_b P(a, b) \times a(ab-\sigma)\sqrt{b^2(ab-\sigma)^2 + \beta} = 0.
    \end{equation*}
    For simplicity, denote $a^2(ab-\sigma)^2 + \beta$ and $b^2(ab-\sigma)^2 + \beta$ as $Q_{(a)}$ and $Q_{(b)}$ in $\sR[a, b]$, respectively. Then
    \begin{equation*}
        \roundb{\partial_a P(a, b)}^2 \times b^2Q_{(a)} = \roundb{\partial_b P(a, b)}^2 \times a^2Q_{(b)}.
    \end{equation*}
    Plugging $a = 0$ gives $\roundb{\partial_a P(0, b)}^2 \times b^2\beta = 0$, which implies that $a\:|\: \partial_a P(a, b)$. Plugging $b = 0$ implies $b\:|\: \partial_b P(a, b)$ by analogous reasoning. We see that there exist some polynomials $P_{(a)}'$ and $P_{(b)}'$ that
    \begin{equation*}
        \partial_a P(a, b) = a P_{(a)}', \quad \partial_b P(a, b) = b P_{(b)}'.
    \end{equation*}
    Thus, for any $a$ and $b$,
    \begin{equation*}
        P_{(a)}'^2 Q_{(a)} = P_{(b)}'^2 Q_{(b)}.
    \end{equation*}
    As $\beta > 0$, the discriminant of $Q_{(a)}$, viewed as a quadratic in $b$, is
    \begin{equation*}
        (-2a^3\sigma)^2 - 4(a^4)(a^2\sigma^2 + \beta) = -4a^4\beta,
    \end{equation*}
    which is not a perfect square in $\sR[a]$. For a quadratic with a non-perfect square discriminant, it has no roots in $\sR[a]$. $\sR[a]$ being UFD allows us to use Gauss' lemma, hence $Q_{(a)}$ is (irreducible, thus) a prime in $\sR[a,b]$. Similarly, $Q_{(b)}$ is also a prime in $\sR[a,b]$. We now show that $Q_{(a)}$ and $Q_{(b)}$ are coprime. Indeed, for every common factor $D$ of $Q_{(a)}$ and $Q_{(b)}$, $D$ should also divide
    \begin{equation*}
        b^2Q_{(a)} - a^2Q_{(b)} = \beta(b^2 - a^2) = \beta(b-a)(b+a).
    \end{equation*}
    Since $b-a$ and $b+a$ are coprime, either $D\:|\: (b-a)$ or $D\:|\: (b+a)$. However, $Q_{(a)}(a, a) = a^2(a^2-\sigma)^2 + \beta \not\equiv 0$ and $Q_{(a)}(a, -a) = a^2(-a^2-\sigma)^2 + \beta \not\equiv 0$; so $D = 1$.

    Therefore, we have
    \begin{align*}
        Q_{(a)} \:|\: P_{(b)}'^2 \quad &\Rightarrow \quad Q_{(a)} \:|\: P_{(b)}',\\
        Q_{(b)} \:|\: P_{(a)}'^2 \quad &\Rightarrow \quad Q_{(b)} \:|\: P_{(a)}'.
    \end{align*}
    Then, there exist some polynomials $P_{(a)}''$ and $P_{(b)}''$ that
    \begin{equation*}
        P_{(a)}' = Q_{(b)} P_{(a)}'', \quad P_{(b)}' = Q_{(a)} P_{(b)}''.
    \end{equation*}
    Thus,
    \begin{equation*}
        P_{(a)}''^2 Q_{(b)} = P_{(b)}''^2 Q_{(a)}.
    \end{equation*}
    Then, with similar argument, there exist some polynomials $P_{(a)}'''$ and $P_{(b)}'''$ that
    \begin{equation*}
        P_{(a)}'' = Q_{(a)} P_{(a)}''', \quad P_{(b)}'' = Q_{(b)} P_{(b)}'''.
    \end{equation*}
    Thus,
    \begin{equation*}
        P_{(a)}'''^2 Q_{(a)} = P_{(b)}'''^2 Q_{(b)}.
    \end{equation*}
    The above reasoning implies that $P_{(a)}'$ is divisible by $Q_{(a)}^k$ for all positive integers $k \geq 1$, which can only happen when $P_{(a)}' \equiv 0$; we also have $P_{(b)}' \equiv 0$. Then $\partial_a P(a, b) \equiv 0$ and $\partial_b P(a, b) \equiv 0$. Therefore, such an invariant $P$ has to be constant.
\end{proof}

The similar result is obtained for $\cT$ on the set where SpecGF with $\cT$ is analytic.
\begin{proposition}[Absence of polynomial invariant for $\cT$]\label{prop:no_poly_invar_T}
    For SpecGF with $\cT$, there does not exist a non-constant invariant polynomial with real coefficients in $a$ and $b$ on
    \begin{equation*}
        O \triangleq \curlyb{(a, b): a\ne0, b\ne0, ab-\sigma \ne 0}.
    \end{equation*}
\end{proposition}
\begin{proof}
    Note that on $\sR$,
    \begin{equation*}
        \cT(x) = \mathrm{sgn}(x) \quad (x \ne 0),
    \end{equation*}
    where $\mathrm{sgn}$ is the sign function.
    
    Let $P$ be a polynomial (of finite degree) with real coefficients. Then
    \begin{equation*}
        \partial_a P(a(t), b(t))\dot a(t) + \partial_b P(a(t), b(t))\dot b(t) = 0.
    \end{equation*}
    Substitute $\dot a(t) = -\cT(\partial_a \cL)$ and $\dot b(t) = -\cT(\partial_b \cL)$. Rearranging the equations gives that, for any $a$ and $b$ in $O$,
    \begin{equation*}
        \partial_a P(a, b) \times \mathrm{sgn}\roundb{b(ab-\sigma)} + \partial_b P(a, b) \times \mathrm{sgn}\roundb{a(ab-\sigma)} = 0.
    \end{equation*}
    If $\mathrm{sgn}(a) = \mathrm{sgn}(b)$, we have
    \begin{equation*}
        \partial_a P(a, b) + \partial_b P(a, b) = 0.
    \end{equation*}
    As the equality holds on the open subset, we actually have
    \begin{equation*}
        \partial_a P(a, b) + \partial_b P(a, b) \equiv 0
    \end{equation*}
    on all of $\sR \times \sR$. Analogously, if $\mathrm{sgn}(a) = -\mathrm{sgn}(b)$, we have
    \begin{equation*}
        \partial_a P(a, b) - \partial_b P(a, b) = 0,
    \end{equation*}
    so in fact, on all of $\sR \times \sR$,
    \begin{equation*}
        \partial_a P(a, b) - \partial_b P(a, b) \equiv 0.
    \end{equation*}
    Therefore $\partial_a P(a, b)$ and $\partial_b P(a, b)$ is identically zero everywhere, implying that such an invariant $P$ has to be constant.
\end{proof}

In addition, we show that $\bA(t)^\top\bA(t) - \bB(t)\bB(t)^\top$ is not an invariant for SpecGF with both $\cT$ and $\cT_\beta$.

\begin{proposition}\label{prop:AAT-BTB_not_inv}
     For both SpecGF with $\cT_\beta$ and $\cT$, there exists $\bY$ such that $\bA(t)^\top\bA(t) - \bB(t)\bB(t)^\top$ cannot be an invariant.
\end{proposition}
\begin{proof}
    Fix $\sigma > 0$ and $a, b \ne 0$ with $ab \ne \sigma$, $|a|\ne |b|$. We show that with
    \begin{equation*}
        \bY = \begin{bmatrix}
            \sigma \bI_{r^\star} & \bm 0\\
            \bm 0 & \bm 0
        \end{bmatrix}
    \end{equation*}
    for any $r^\star \geq r$, SpecGF with $\cT_\beta$ cannot have $\bA(t)^\top\bA(t) - \bB(t)\bB(t)^\top$ as its invariant, due to its non-zero time derivative at
    \begin{equation*}
        (\bA, \bB) = \roundb{\begin{bmatrix}
            a \bI_{r} \\
            \bm 0
        \end{bmatrix}, \begin{bmatrix}
            b \bI_{r} & \bm 0\\
        \end{bmatrix}}
    \end{equation*}

    Compute the gradient at $(\bA, \bB)$:
    \begin{align*}
        \nabla_\bA\cL(\bA, \bB) &= \begin{bmatrix}
            b(ab-\sigma) \bI_{r}\\
            \bm 0
        \end{bmatrix} \triangleq \begin{bmatrix}
            a' \bI_{r}\\
            \bm 0
        \end{bmatrix}\\
        \nabla_\bB\cL(\bA, \bB) &= \begin{bmatrix}
            a(ab-\sigma) \bI_{r} & \bm 0
        \end{bmatrix} \triangleq \begin{bmatrix}
            b' \bI_{r} & \bm 0
        \end{bmatrix}.
    \end{align*}
    Then, for SpecGF with $\cT_\beta$
    \begin{align*}
        \dot \bA = -\cT_\beta\roundb{\nabla_\bA\cL(\bA, \bB)} &= \begin{bmatrix}
            -\frac{a'}{\sqrt{a'^2+\beta}} \bI_{r}\\
            \bm 0
        \end{bmatrix}\\
        \dot \bB = -\cT_\beta\roundb{\nabla_\bB\cL(\bA, \bB)} &= \begin{bmatrix}
            -\frac{b'}{\sqrt{b'^2+\beta}} \bI_{r} & \bm 0
        \end{bmatrix}.
    \end{align*}
    Hence, the time derivative of $\bA$ and $\bB$ translate to that of $a$ and $b$, respectively. We obtain
    \begin{align*}
        \frac{\mathrm{d}}{\mathrm{d}t} (\bA^\top\bA - \bB\bB^\top) &= \frac{\mathrm{d}}{\mathrm{d}t} (a^2 - b^2)\bI_r\\
        &= -2\roundb{a\frac{a'}{\sqrt{a'^2+\beta}} - b\frac{b'}{\sqrt{b'^2+\beta}}}\bI_r\\
        &= -2ab(ab-\sigma)\roundb{\frac{1}{\sqrt{b^2(ab-\sigma)^2+\beta}} - \frac{1}{\sqrt{a^2(ab-\sigma)^2+\beta}}}\bI_r \ne \bm 0.
    \end{align*}

    The same setting results in the same result for SpecGF with $\cT$. For SpecGF with $\cT$
    \begin{align*}
        \dot \bA = -\cT\roundb{\nabla_\bA\cL(\bA, \bB)} &= \begin{bmatrix}
            -\mathrm{sgn}(a') \bI_{r}\\
            \bm 0
        \end{bmatrix}\\
        \dot \bB = -\cT\roundb{\nabla_\bB\cL(\bA, \bB)} &= \begin{bmatrix}
            -\mathrm{sgn}(b') \bI_{r} & \bm 0
        \end{bmatrix}.
    \end{align*}
    Again, the time derivative of $\bA$ and $\bB$ translate to that of $a$ and $b$, respectively. We obtain
    \begin{align*}
        \frac{\mathrm{d}}{\mathrm{d}t} (\bA^\top\bA - \bB\bB^\top) &= \frac{\mathrm{d}}{\mathrm{d}t} (a^2 - b^2)\bI_r\\
        &= -2\roundb{a\cdot\mathrm{sgn}(a') - b\cdot\mathrm{sgn}(b')}\bI_r\\
        &= -2\mathrm{sgn}(ab-\sigma)\roundb{a\cdot\mathrm{sgn}(b) - b\cdot\mathrm{sgn}(a)}\bI_r \ne \bm 0. \qedhere
    \end{align*}
\end{proof}

\subsection{Saddle Point Avoidance}

\citet{bah2022learning} characterize the critical points of the loss
\begin{equation}\label{eq:N-layer_linear_network}
    L^N(\bW_1, \ldots, \bW_N) = \frac{1}{2}\normF{\bY - \bW_N\cdots \bW_1\bX}^2,
\end{equation}
where $\bW_j \in \sR^{d_j\times d_{j-1}}$ for $d_0=d_x$, $d_N=d_y$ and $\bX \in \sR^{d_x\times m}$ and $\bY \in \sR^{d_y\times m}$. They analyzed whether the critical points are the minima or strict saddle points through a closely related loss
\begin{equation*}
    L^1(\bW) = \frac{1}{2}\normF{\bY - \bW\bX}^2,
\end{equation*}
where $\bW \in \sR^{d_y\times d_x}$. We summarize their result, especially for $N=2$ case.

\begin{proposition}[\citep{bah2022learning}, Proposition 6.6, 6.9, 6.11] \label{prop:SSP_characterization}
    Let $r \triangleq \min\{d_0, d_1, d_2\}$. Assume that $\bX\bX^\top$ has full rank $d_0$, $\bY$ is of rank $q$, and $\bW_k$ is the global minimum of $L^1$ restricted to a set of matrices whose rank is at most $k$. Let $K$ be the rank of $\bW_2\bW_1$.
    \begin{enumerate}
        \item $\bW_k$ is the best rank-$k$ approximation of $\bY$.
        \item If $K = \min\{r, q\}$, every critical point $(\bW_1, \bW_2)$ that satisfies $\bW_2\bW_1 = \bW_K$ is the global minimum of $L^2$. If $\bW_2\bW_1 \ne \bW_K$, then $(\bW_1, \bW_2)$ is a strict saddle point.
        \item Any other critical point $(\bW_1, \bW_2)$, including the origin, is a strict saddle point.
    \end{enumerate}
\end{proposition}

We say that a critical point $x_c$ of a twice continuously differentiable function $f$ is a \textit{strict saddle point} if the Hessian of $f$ at $x_c$ has a negative eigenvalue. \Cref{prop:SSP_characterization} implies that every critical point of $L^2$ is exactly either a global minimum or a strict saddle point. Notice that if $\bX$ is the identity matrix and $N=2$, then the setting of \citet{bah2022learning} exactly matches with ours. They also proved that the set of initial points such that the corresponding plain gradient flow for \Cref{eq:N-layer_linear_network} converges to a strict saddle point has measure zero \citep{bah2022learning}[Theorem 6.3]. We also show that this holds for SpecGF with $\cT_\beta$. Then, if SpecGF converges, it almost surely converges to the global minima, or the best low-rank approximation of $\bY$.

To this end, we employ a proposition from \citet{cheridito2024gradient} that claims the similar result.

\begin{proposition}[\citep{cheridito2024gradient}, Proposition 2.5]\label{prop:fixed_point_avoid}
    Let $f:\sR^d \to \sR^d$. Suppose that there exist open sets $V \subseteq U \subseteq \sR^d$ such that the complement of $V$ have Lebesgue measure zero and $f$ is continuously differentiable on $U$ with a locally Lipschitz continuous Jacobian, which is non-degenerate on $V$. Let $\cS \subseteq \{x \in U: f(x) = x\}$ and assume for all $x \in \cS$ that $f'(x)$ is symmetric and has an eigenvalue of absolute value strictly greater than 1. Then, the set $\{x \in \sR^d: \lim_{k\to\infty} f^k(x)\in\cS\}$ has Lebesgue measure 0.
\end{proposition}
There, $f^k(x)$ for $k \in \mathbb{N}_0$ is inductively defined as $f^0 \equiv \mathrm{id}$ and $f^{k+1} = f \circ f^k$.

\begin{theorem}\label{thm:SpecGF_avoids_SSP}
    SpecGF with $\cT_\beta$ almost surely avoids strict saddle points. Specifically, if $r \leq \mathrm{rank}(\bY)$, then SpecGF with $\cT_\beta$ almost surely converges to the best rank-$r$ approximation of $\bY$ if convergence is guaranteed.
\end{theorem}
\begin{proof}
    We first verify that the conditions of \Cref{prop:fixed_point_avoid} are met for SpecGF with $\cT_\beta$ and then extend the statement from the discrete to the continuous version.

    For simplicity, let $\bx(t) \triangleq (\bA(t), \bB(t))$ and $f_\beta (\bx) \triangleq \roundb{-\cT_\beta(\nabla_\bA \cL(\bx)), -\cT_\beta(\nabla_\bB \cL(\bx))}$. Denote the solution of
    \begin{equation*}
        \dot \bx(t) = f_\beta(\bx), \quad \bx(0) \triangleq \bx_0.
    \end{equation*}
    as $\phi(t)$; when the initial point matters, we use
    \begin{equation*}
        \varphi_t(\bx_0) = \phi(t).
    \end{equation*}
    We note that $\varphi_t(\bx_0)$, and also $\phi(t)$, satisfy the semigroup property:
    \begin{equation*}
        \varphi_{t'+t}(\bx_0) = \varphi_{t'}(\varphi_t(\bx_0)).
    \end{equation*}

    The Jacobian matrix of $\varphi_t(\bx_0)$, $D\varphi_t(\bx_0)$, satisfies the following with Jacobian matrix of $f_\beta$, $Df_\beta$:
    \begin{equation*}
        \frac{\mathrm{d}}{\mathrm{d}t} D\varphi_t(\bx_0) = Df_\beta\roundb{\varphi_t(\bx_0)} D\varphi_t(\bx_0), \quad D\varphi_0(\bx_0) = \bI.
    \end{equation*}
    Liouville's theorem yields:
    \begin{equation*}
        \abs*{D\varphi_t(\bx_0)} = \abs*{D\varphi_0(\bx_0)}\exp\roundb{\int_0^t \mathrm{tr} Df_\beta\roundb{\varphi_\tau(\bx_0)} \mathrm{d}\tau}.
    \end{equation*}
    Analyticity of $f_\beta$ makes $\abs*{D\varphi_t(\bx_0)}$ nonzero for all $\bx_0$ and all $t \geq 0$, hence the conditions meet with $V = U = \sR^d$. Now take $\cS$ as the set of strict saddle points. It is straightforward to see that $\cT_\beta$ is (Fr\'echet) differentiable at 0 with
    \begin{equation}\label{eq:DTb0}
        D\cT_\beta(0)[\bm\Delta] = \beta^{-1/2}\bm\Delta,
    \end{equation}
    for all matrix $\bm\Delta$. It suffices to prove that
    \begin{equation*}
        \lim_{h\to0} \frac{\normF{\cT_\beta(h\bm\Delta) - \cT_\beta(0) - h\beta^{-\frac{1}{2}}\bm\Delta}}{|h|} = 0,
    \end{equation*}
    or
    \begin{equation*}
        \lim_{h\to0} \normF{\frac{\cT_\beta(h\bm\Delta)}{h} - \beta^{-\frac{1}{2}}\bm\Delta} = 0.
    \end{equation*}
    Notice that, for $s(x) = x^{-\frac{1}{2}}$ on $[\beta, \infty)$ and $c > 0$,
    \begin{equation}\label{eq:isq_mvt}
        |s(x+c) - s(x)| = |s'(x+rc)|c \leq \frac{c}{2\beta^\frac{3}{2}}
    \end{equation}
    for some $r \in (0, 1)$. Applying \Cref{eq:isq_mvt} to the eigenvalues of $\bm\Delta^\top\bm\Delta$ yields
    \begin{align*}
        \normF{\frac{\cT_\beta(h\bm\Delta)}{h} - \beta^{-\frac{1}{2}}\bm\Delta} &= \normF{\bm\Delta\roundb{h^2\bm\Delta^\top\bm\Delta + \beta\bI}^{-\frac{1}{2}} - \beta^{-\frac{1}{2}}\bm\Delta}\\
        &\leq \normF{\bm\Delta}\norm{\roundb{h^2\bm\Delta^\top\bm\Delta + \beta\bI}^{-\frac{1}{2}} - \beta^{-\frac{1}{2}}\bI}_2\\
        &\leq \normF{\bm\Delta}\frac{h^2\norm{\bm\Delta}_2^2}{2\beta^\frac{3}{2}}.
    \end{align*}
    \Cref{eq:DTb0} directly yields that
    \begin{equation}\label{eq:Df_HessL}
        Df_\beta(\bar \bx) = -\beta^{-1/2} \nabla^2\cL(\bar \bx)
    \end{equation}
    and it is symmetric for all $\bar \bx \in \cS$ (in fact, this holds at any critical point). As $\bar\bx$ is an equilibrium point,
    \begin{equation*}
        \frac{\mathrm{d}}{\mathrm{d}t} D\varphi_t(\bar\bx) = Df_\beta\roundb{\varphi_t(\bar\bx)} D\varphi_t(\bar\bx) = Df_\beta\roundb{\bar\bx} D\varphi_t(\bar\bx).
    \end{equation*}
    Thus
    \begin{equation*}
        D\varphi_t(\bar\bx) = e^{tDf_\beta\roundb{\bar\bx}}.
    \end{equation*}
    Since $\bar\bx$ is a strict saddle, the Hessian of $\cL$ at $\bar\bx$ has a negative eigenvalue. By \Cref{eq:Df_HessL}, $Df_\beta\roundb{\bar\bx}$ has a positive eigenvalue, and hence $D\varphi_t(\bar\bx)$ has an eigenvalue whose (absolute) value is strictly greater than 1 at any $t \geq 0$.

    The remaining part, the extension from discrete to continuous, is simple. Fix any time $T > 0$. By the semigroup property,
    \begin{align*}
        \curlyb{\bx_0: \lim_{t\to\infty} \varphi_t(\bx_0) \in \cS} &= \bigcup_{\bar\bx \in \cS} \curlyb{\bx_0: \lim_{t\to\infty} \varphi_t(\bx_0) = \bar\bx}\\
        &\subseteq \bigcup_{\bar\bx \in \cS} \curlyb{\bx_0: \lim_{k\to\infty} \varphi_{kT}(\bx_0) = \bar\bx}\\
        &= \curlyb{\bx_0: \lim_{k\to\infty} (\varphi_{T})^k(\bx_0) \in \cS}.
    \end{align*}
    The proof is concluded since the subset of a Lebesgue measure zero set is again of Lebesgue measure zero.
\end{proof}

\subsection{Stability Analysis of SpecGF}

Recall that, from \Cref{sec:Tb},
\begin{align*}
    \cT_\beta(\bX) &= \bU\bS_\beta\bV^\top, \quad \bS_\beta \triangleq \Diag\curlyb{\frac{\sigma_i} {\sqrt{\sigma_i^2+\beta}}}_{i\in[r]}.
\end{align*}
Thus,
\begin{equation}
    \angleb{\bX, \cT_\beta(\bX)} \geq \sqrt{\beta} \normF{\cT_\beta(\bX)}^2 \quad \Rightarrow \quad \normF{\cT_\beta(\bX)} \leq \beta^{-\frac{1}{2}}\normF{\bX}.
\end{equation}

\begin{proposition}[\citep{bhat2010arc}, Theorem 5.2]\label{prop:Lyap_stab} Consider the system of differential equations
\begin{equation}
    \dot \bx(t) = f(\bx(t))
\end{equation}
with $f: \sR^d \to \sR^d$ that yields a unique $\cC^1$ solution for every initial condition $\bx(0)$. Let $\bx \in \sR^d$, and suppose there exists a continuous function $V: \cV \to \sR$ defined on an open neighborhood $\cV$ of $\bx$, and $c > 0$ such that
\begin{equation}
    c\dot V(\bz) + \norm{f(\bz)} \leq 0
\end{equation}
is satisfied for every $\bz \in \cV$. If $\bx$ is a local minimizer of V, then $\bx$ is a Lyapunov stable equilibrium.
\end{proposition}

We now prove that every (global) minimum of the loss is Lyapunov stable.
\begin{theorem}[Lyapunov stability of minima]\label{thm:SpecGF_Tb_stable}
    Every global minimum of $\cL(\bA, \bB)$ is Lyapunov stable for SpecGF with $\cT_\beta$.
\end{theorem}
\begin{proof}
    Recall \Cref{lem:Loja_inequ}: at each global minima $\bx_*$ of $\cL$, there exist a neighborhood $W$ of $\bx_*$, constants $C > 0$ and $b \in [\frac12, 1)$ such that for all $\bx \in W$,
    \begin{equation*}
        |\cL(\bx) - \cL(\bx_*)|^b \leq C\|\nabla \cL(\bx)\|.
    \end{equation*}
    Note that $\cL(\bx_*) \triangleq \cL_*$ for any global minima $\bx_*$. Without loss of generalization, take $W$ as an open ball centered at $\bx_*$ and of (finite) positive radius. On $W$, define
    \begin{equation}
        V(\bx) \triangleq \roundb{\cL(\bx) - \cL_*}^{1-b}.
    \end{equation}
    Then $V$ is continuous, nonnegative, minimized at $\bx_*$ and
    \begin{equation*}
        \dot V(\bx) = (1-b)\roundb{\cL(\bx) - \cL_*}^{-b}\dot \cL(\bx).
    \end{equation*}
    On $W$, $\nabla_\bA \cL$ and $\nabla_\bB \cL$ are both bounded; take the bound as $C_W$ so that
    \begin{equation*}
        \normF{\nabla_\bA \cL} \leq C_W, \quad \normF{\nabla_\bB \cL} \leq C_W.
    \end{equation*}
    Observe
    \begin{align*}
        -\dot \cL &= \angleb{\nabla_\bA \cL, \cT_\beta(\nabla_\bA \cL)} + \angleb{\nabla_\bB \cL, \cT_\beta(\nabla_\bB \cL)}\\
        &= \sum_i \frac{\sigma_i^2(\nabla_\bA \cL)}{\sqrt{\sigma_i^2(\nabla_\bA \cL)+\beta}} + \sum_j \frac{\sigma_j^2(\nabla_\bB \cL)}{\sqrt{\sigma_j^2(\nabla_\bB \cL)+\beta}}\\
        &\geq \sum_i \frac{\sigma_i^2(\nabla_\bA \cL)}{\sqrt{C_W^2 + \beta}} + \sum_j \frac{\sigma_j^2(\nabla_\bB \cL)}{\sqrt{C_W^2 + \beta}}\\
        &= \frac{1}{\sqrt{C_W^2 + \beta}}\normF{\nabla \cL}^2.
    \end{align*}
    In addition, for the same $f_\beta$ in the proof of \Cref{thm:SpecGF_avoids_SSP},
    \begin{equation*}
        \normF{f_\beta(\bx)} = \sqrt{\normF{\cT_\beta(\nabla_\bA \cL(\bx))}^2 + \normF{\cT_\beta(\nabla_\bB \cL(\bx))}^2} \leq \sqrt{\beta^{-1} \roundb{\normF{\nabla_\bA \cL(\bx)}^2 + \normF{\nabla_\bB \cL(\bx)}^2}} = \beta^{-\frac{1}{2}}\normF{\nabla \cL(\bx)}.
    \end{equation*}
    Hence, on $W$,
    \begin{align*}
        -\dot V(\bx) &= -(1-b)\roundb{\cL(\bx) - \cL_*}^{-b}\dot \cL(\bx)\\
        &\geq (1-b)\roundb{\cL(\bx) - \cL_*}^{-b} \frac{1}{\sqrt{C_W^2 + \beta}}\normF{\nabla \cL(\bx)}^2\\
        &\geq \frac{1-b}{C\sqrt{C_W^2 + \beta}}\normF{\nabla \cL(\bx)}\\
        &\geq \frac{1-b}{C\sqrt{C_W^2 + \beta}}\beta^{\frac{1}{2}}\normF{f_\beta(\bx)}.
    \end{align*}
    $c = \frac{C\sqrt{C_W^2 + \beta}}{(1-b)\sqrt{\beta}}$ gives the desired inequality.
\end{proof}
Therefore, the a.s. local convergence of SpecGF with $\cT_\beta$ towards global minima is guaranteed.

\paragraph{Conjecture: The basin of attraction is large.}For now, let $\rho_W$ be the radius of $W$. Fix $\epsilon > 0$ and take $\rho = \min\{\rho_W, \epsilon\} > 0$. We will find $\delta > 0$ for the Lyapunov stability of each minimizer $\bx_* = (\bA_*, \bB_*)$.\newline Let $\bz = (\bA, \bB)$ and assume $\normF{\bz - \bx_*}\leq \delta$. We have
\begin{equation*}
    \normF{\Delta \bA} \triangleq \normF{\bA - \bA_*} \leq \delta, \quad \normF{\Delta \bB} \triangleq \normF{\bB - \bB_*} \leq \delta.
\end{equation*}
Then, for some constants $c' > 0$ (may depend on $\bx_*$)
\begin{align*}
    \normF{\bA\bB - \bY} &= \normF{\Delta \bA\bB_* + \bA_*\Delta\bB + \Delta \bA\Delta \bB}\\
    &\leq \normF{\Delta \bA\bB_*} + \normF{\bA_*\Delta\bB} + \normF{\Delta \bA\Delta \bB}\\
    &\leq \delta\roundb{\norm{\bA_*}_2 + \norm{\bB_*}_2} + \delta^2\\
    &\leq c' \max\{\delta, \delta^2\}.
\end{align*}
(e.g., take $c' = \norm{\bA_*}_2 + \norm{\bB_*}_2 + 1$.) Hence,
\begin{equation*}
    \cL(\bz) = \frac{1}{2}\normF{\bA\bB - \bY}^2 \leq \frac{1}{2} c'^2 \max\{\delta^2, \delta^4\} \quad \Rightarrow \quad V(\bz) \leq \roundb{\frac{1}{2} c'^2 \max\{\delta^2, \delta^4\} - \cL_*}^{1-b}.
\end{equation*}
\newline In the proof of \citet[Theorem 5.2]{bhat2010arc}, two conditions are required:
\begin{equation*}
    \normF{\bz - \bx_*}\leq \frac{\rho}{4}, \quad |V(\bz) - V(\bx_*)| \leq \frac{\rho}{4c},
\end{equation*}
where $c = \frac{C\sqrt{C_W^2 + \beta}}{(1-b)\sqrt{\beta}}$. It suffices that
\begin{equation*}
    \delta \leq \frac{\rho}{4}, \quad \roundb{\frac{1}{2} c'^2 \max\{\delta^2, \delta^4\} - \cL_*}^{1-b} \leq \frac{\rho}{4c},
\end{equation*}
which translates to
\begin{equation}
    \delta(\epsilon, \bx_*) = \min\curlyb{\frac{\rho}{4}, \sqrt{\frac{2}{c'^2}\roundb{\roundb{\frac{\rho}{4c}}^{\frac{1}{1-b}} + \cL_*}}, \sqrt[4]{\frac{2}{c'^2}\roundb{\roundb{\frac{\rho}{4c}}^{\frac{1}{1-b}} + \cL_*}}}, \quad \rho = \min\{\rho_W, \epsilon\}.
\end{equation}
We conjecture that $\rho_W$ is sufficiently large for each and every minimizer $\bx_*$. If then, $\delta$ increases as $\epsilon$ increases, and there is a possibility that initial points near (but not exactly) the origin fall in some basin of attraction. We empirically checked that the basin of attraction is sufficiently large.

\begin{figure}[H]
    \centering
    \includegraphics[width=0.48\linewidth]{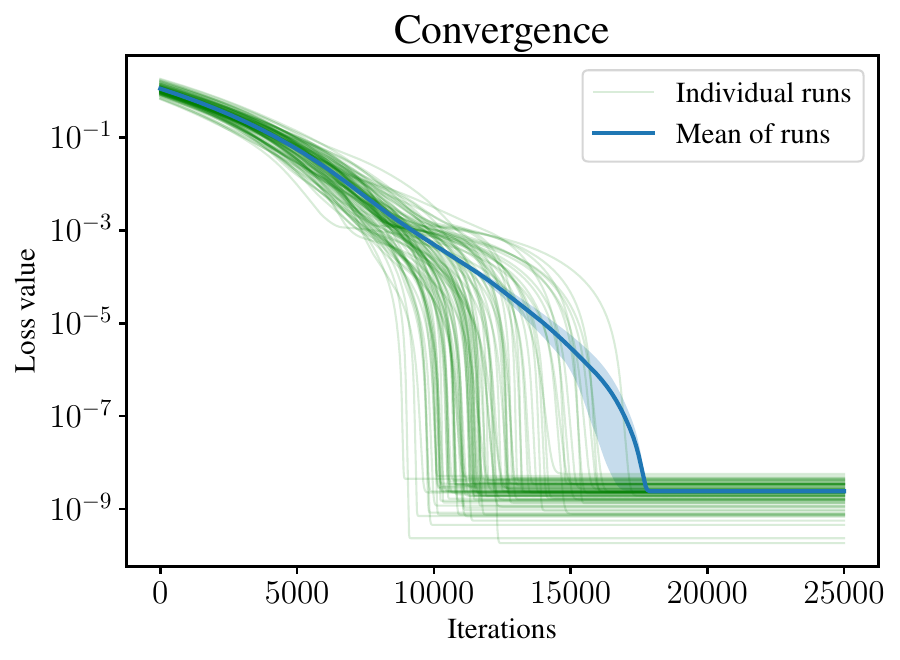}
    \includegraphics[width=0.48\linewidth]{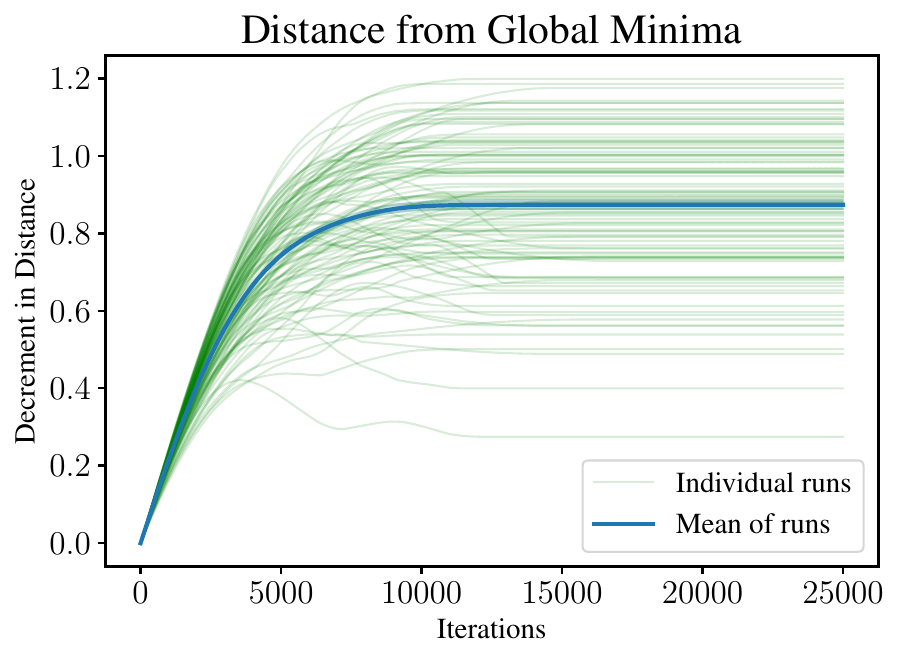}
    \caption{Lyapunov stability of global minima. \textbf{Left}: Loss trajectory. \textbf{Right}: Decrement of the distance from global minima. The attraction of the basin of some global minima is sufficiently large that, empirically, the origin lies in there.}
    \label{fig:conj_origin_attracted}
\end{figure}

We set $\bY \in \sR^{9\times 9}$ with $r = 4$ and $\bm\Sigma = \Diag(1, 0.5, 0.2, 0.05)$, $\bA \in \sR^{9\times 4}$, and $\bB \in \sR^{4\times 9}$. We first sample 20 random initial points $\bB_i(0)$ with i.i.d. Gaussian entries of scale $\gamma = 5\times10^{-4}$; $\bA_i(0) = \bm 0$. We use a fixed learning rate of $\eta = 10^{-4}$. We ran SpecGD with $\cT_\beta$ starting from $(\bA_i(0), \bB_i(0))$ until convergence. For all $i \in [20]$, SpecGD converged to global minima; see the left plot of \Cref{fig:conj_conv_i}. Name those global minima as $(\bA_i^*, \bB_i^*)$. For each $i$, we sample 50 random directions $(\Delta\bA_{ij}, \Delta\bB_{ij})$ from i.i.d. Gaussian entries, with magnitude equal to the Euclidean distance between $(\bA_i(0), \bB_i(0))$ and $(\bA_i^*, \bB_i^*)$. Denote the initial points as $(\bA_{ij}(0), \bB_{ij}(0)) = (\bA_i^*, \bB_i^*) + (\Delta\bA_{ij}, \Delta\bB_{ij})$. We ran $20\times50=1000$ SpecGD experiments and recorded (i) loss trajectory $\cL(t)$ and (ii) decrement of the distance from global minima, $\mathrm{dist}((\bA_{ij}(0), \bB_{ij}(0)), (\bA_i^*, \bB_i^*)) - \mathrm{dist}((\bA_{ij}(t), \bB_{ij}(t)), (\bA_i^*, \bB_i^*))$; $\mathrm{dist}(\bx, \by)$ stands for the Euclidean distance between $\bx$ and $\by$. The recorded values are plotted in \Cref{fig:conj_origin_attracted}.

\begin{figure}
    \centering
    \includegraphics[width=0.48\linewidth]{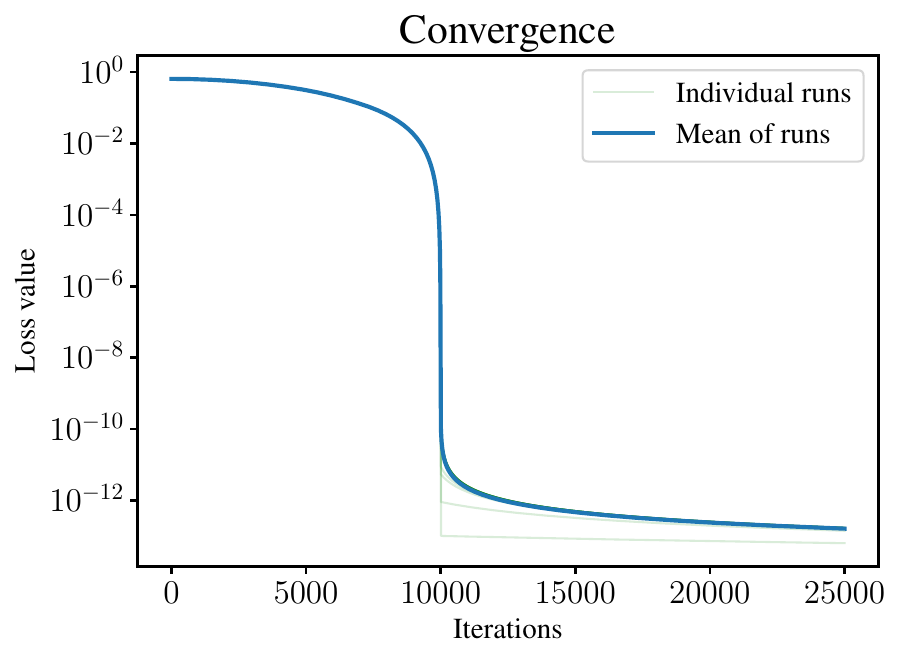}
    \caption{Loss trajectories started at $(\bA_i(0), \bB_i(0))$.}
    \label{fig:conj_conv_i}
\end{figure}

We see that 1) all SpecGF simulation initialized at $(\bA_{ij}(0), \bB_{ij}(0))$ converged to the global minima and 2) the decrement is positive, supporting the conjecture that the near-origin points lies in a basin of attraction of some global minima.

\subsection{Convergence Rate of SpecGF}\label{sec:conv_rate_specgf}
Here, we study the convergence rate of SpecGF with $\cT_\beta$, assuming the convergence itself. This analysis closely follows that of \Cref{prop:infdiv_conv}; SpecGF iterate will get close to some critical point of $\cL$ after a finite time, where the \L{}ojasiewicz gradient inequality holds. On the vicinity of that critical point, the rate of convergence is explicitly derived.

Let $(\bA(t), \bB(t)) \to (\bar\bA, \bar\bB)$. Then $\cT_\beta(\nabla_\bA\cL(t)), \cT_\beta(\nabla_\bA\cL(t)) \to 0$, thus $\nabla_\bA\cL(t), \nabla_\bB\cL(t) \to 0$. This implies that there exists some time $t' > 0$ such that for all $t \geq t'$,
\begin{equation*}
    \norm{\nabla_\bA \cL(t)}_2 \leq \sqrt{\beta}, \quad \norm{\nabla_\bB \cL(t)}_2 \leq \sqrt{\beta}.
\end{equation*}
In addition, let $W$ be the neighborhood of $(\bar\bA, \bar\bB)$ where the \L{}ojasiewicz gradient inequality hold: there also exist constants $C > 0$ and $b \in [\frac12, 1)$ such that, for all $(\bA, \bB) \in W$,
\begin{equation*}
    |\cL(\bA, \bB) - \cL(\bar\bA, \bar\bB)|^b \leq C\|\nabla \cL(\bA, \bB)\|.
\end{equation*}
We know that there exists some time $t'' > 0$ such that for all $t \geq t''$, $(\bA(t), \bB(t)) \in W$. Take $T \triangleq \max\{t', t''\} > 0$. We briefly state the helper lemma for deriving the convergence rate.
\begin{lemma}
    For all matrix $\bM$ such that $\norm{\bM}_2 \leq \sqrt{\beta}$,
    \begin{equation*}
        \angleb{\bM, \cT_\beta(\bM)} \geq \frac{1}{\sqrt{2\beta}}\normF{\bM}^2.
    \end{equation*}
\end{lemma}
\begin{proof}
    \begin{equation*}
        \angleb{\bM, \cT_\beta(\bM)} = \sum_i \frac{\sigma_i(\bM)^2}{\sqrt{\sigma_i(\bM)^2 + \beta}} \geq \sum_i \frac{\sigma_i(\bM)^2}{\sqrt{2\beta}} = \frac{1}{\sqrt{2\beta}}\normF{\bM}^2.
    \end{equation*}
\end{proof}

Observe that, for $t \geq T$,
\begin{align*}
    \frac{d}{\mathrm{d}t} \cL(t) &= \angleb{\nabla_\bA \cL(t), -\cT_\beta\roundb{\nabla_\bA \cL(t)}} + \angleb{\nabla_\bB \cL(t), -\cT_\beta\roundb{\nabla_\bB \cL(t)}}\\
    &\leq -\frac{1}{\sqrt{2\beta}}\roundb{\normF{\nabla_\bA \cL(t)}^2 + \normF{\nabla_\bB \cL(t)}^2} = -\frac{1}{\sqrt{2\beta}}\normF{\nabla \cL(t)}^2\\
    &\leq -\frac{1}{C^2\sqrt{2\beta}}(\cL(t) - \cL(\bar\bA, \bar\bB))^{2b}.
\end{align*}
Let $\cE(\bA, \bB) \triangleq \cL(\bA, \bB) - \cL(\bar\bA, \bar\bB) \geq 0$; we use simplifed notation $\cE(t) \triangleq \cL(t) - \cL(\bar\bA, \bar\bB) \geq 0$. Then for $t \geq T$, $\dot\cE(t) \leq -C(\beta) \cE(t)^{2b}$, where $C(\beta) = \frac{1}{C^2\sqrt{2\beta}}$. The convergence rate of $\cE$ differs by $b$. In our case, we established the almost sure convergence towards global minima (under assuming convergence), hence assume $(\bar\bA, \bar\bB)$ is the global minimum. Then, $\cE$ is the Morse-Bott function, and thus the (optimal) exponent $b$ is exactly $\frac{1}{2}$~\citep{feehan2020morse}.
Then,
\begin{equation*}
    \dot\cE(t) \leq -C(\beta) \cE(t),
\end{equation*}
or
\begin{equation*}
    \cE(t) \leq \cE(T)e^{-C(\beta)(t-T)}.
\end{equation*}
Therefore, $\cE(t)$ exponentially decreases to 0. It remains to show that $\cE(\bA, \bB)$ is the Morse-Bott function.
\begin{assumption}
    The nonzero singular values of $\bY$ are mutually distinct.
\end{assumption}

Recall that $\bar\bA\bar\bB$ is the best rank-$r$ approximation of $\bY$~\citep{bah2022learning}. We denote $\bar\bY_r \triangleq \bar\bA\bar\bB$ for simplicity. Furthermore, denote
\begin{equation*}
    \bar\bY_r = \bU_r\bm\Sigma_r\bV_r^\top, \quad \bY - \bar\bY_r \triangleq \bU_\perp\bm\Sigma_\perp\bV_\perp^\top.
\end{equation*}
$[\bU_r | \bU_\perp]$ forms a $m\times m$ orthogonal matrix; similarly, $[\bV_r | \bV_\perp]$ forms a $n\times n$ orthogonal matrix. Note that zero singular values may appear in $\bm\Sigma_\perp$, but never appear in $\bm\Sigma_r$.

The solution set
\begin{equation*}
    S \triangleq \curlyb{\bM\in\sR^{m\times r}, \bN\in\sR^{r\times n}: \bM\bN = \bar\bY_r}
\end{equation*}
is easily identified with $(\bar\bA, \bar\bB)$:
\begin{equation*}
    S = \curlyb{(\bar\bA\bQ, \bQ^{-1}\bar\bB): \bQ \in \mathsf{GL}(r)}.
\end{equation*}
Indeed, for $(\bM, \bN) \in S$, if the left inverse of $\bar\bA$ is $\bL_{\bar\bA}$, then $\bQ = \bL_{\bar\bA}\bM$ gives $\bar\bA\bQ = \bar\bA\bL_{\bar\bA}\bM = \bM$ and $\bQ\bN = \bL_{\bar\bA}\bar\bY_r = \bar\bB$; definitely $\mathrm{rank}(\bQ) = r$ from $\bar\bA\bQ = \bM$.

The mapping
\begin{equation*}
    \Phi:\mathsf{GL}(r) \to \sR^{m\times r}\times \sR^{r\times n}, \quad  \Phi(\bQ) = (\bar\bA\bQ, \bQ^{-1}\bar\bB)
\end{equation*}
is real-analytic and satisfies
\begin{equation*}
    \Phi^{-1}(\bM, \bN) = \bL_{\bar\bA}\bM
\end{equation*}
for a left inverse $\bL_{\bar\bA}$ of $\bar\bA$. Thus $\Phi$ is a real-analytic diffeomorphism from an open subset (hence an analytic submanifold) $\mathsf{GL}(r)$ to $\Phi(\mathsf{GL}(r)) = S$, and $S$ is an analytic submanifold. As $\mathsf{GL}(r)$ has two connected components---matrices with positive and negative determinants---so does $S$. This issue is resolved by shrinking the domain of $S$ to an open neighborhood of $(\bar\bA, \bar\bB)$; then $S$ is connected. It remains to show that $T_{(\bar\bA, \bar\bB)} S = \mathrm{ker}\nabla^2\cE(\bar\bA, \bar\bB)$~\citep[Definition 1.5]{feehan2020morse}; Other conditions for being a Morse-Bott function are automatically satisfied~\citep[Remark 1.6]{feehan2020morse}; $\nabla^2 \cE: (\sR^{m\times r}\times \sR^{r\times n}) \to (\sR^{m\times r}\times \sR^{r\times n})$ is a linear map, thus a Fredholm operator index 0, hence the range condition is automatically satisfied.

Let $\bR(\bA, \bB) \triangleq \bA\bB - \bar\bY_r$ be the residual matrix. The differential of $\bR$ at $(\bar\bA, \bar\bB)$ is
\begin{equation*}
    D\bR(\bar\bA, \bar\bB)[\Delta\bA, \Delta\bB] = \bar\bA\Delta\bB + \Delta\bA\bar\bB.
\end{equation*}
It is straightforward to see that
\begin{equation*}
    \mathrm{ker} D\bR(\bar\bA, \bar\bB) = \curlyb{(\bar\bA \bK, -\bK \bar\bB): \bK \in \sR^{r\times r}}.
\end{equation*}
Indeed, if $\bar\bA\Delta\bB + \Delta\bA\bar\bB = \bm 0$, the right inverse of $\bar\bB$, $\bR_{\bar\bB}$, gives $\bar\bA\Delta\bB\bR_{\bar\bB} + \Delta\bA = \bm 0$, which implies $\Delta\bA = -\bar\bA\Delta\bB\bR_{\bar\bB}$ and $\bar\bA(\Delta\bB - \Delta\bB\bR_{\bar\bB}\bar\bB) = \bm 0 \Rightarrow \Delta\bB = \Delta\bB\bR_{\bar\bB}\bar\bB$.

$T_{(\bar\bA, \bar\bB)} S \subset \mathrm{ker} D\bR(\bar\bA, \bar\bB)$ is trivial; we show the opposite inclusion. For $(\Delta\bA, \Delta\bB) \in \mathrm{ker} D\bR(\bar\bA, \bar\bB)$, we have $(\Delta\bA, \Delta\bB) = (\bar\bA\bK, -\bK\bar\bB)$ for some $\bK$. Define a smooth curve
\begin{equation*}
    \gamma(t) \triangleq (\bar\bA(\bI_r + t\bK), (\bI_r + t\bK)^{-1}\bar\bB).
\end{equation*}
For all $t$ sufficiently close to 0, $\bI_r + t\bK$ is invertible and $\gamma(t) \in S$. Moreover,
\begin{equation*}
    \gamma'(0) = (\bar\bA\bK, -\bK\bar\bB).
\end{equation*}
Hence, every $(\Delta\bA, \Delta\bB) \in \mathrm{ker} D\bR(\bar\bA, \bar\bB)$ is realized as the velocity of a curve in $S$, or $(\Delta\bA, \Delta\bB) \in T_{(\bar\bA, \bar\bB)} S$. We obtain
\begin{equation*}
    T_{(\bar\bA, \bar\bB)} S = \mathrm{ker} D\bR(\bar\bA, \bar\bB).
\end{equation*}
 
The Hessian of $\cE$ at $(\bar\bA, \bar\bB)$ is computed:
\begin{align*}
    &\nabla^2\cE(\bar\bA, \bar\bB)[(\Delta\bA_0, \Delta\bB_0), (\Delta\bA_1, \Delta\bB_1)]\\
    &= \angleb{\Delta\bA_0\bar\bB + \bar\bA\Delta\bB_0, \Delta\bA_1\bar\bB + \bar\bA\Delta\bB_1} + \angleb{\bar\bA\bar\bB - \bY, \Delta\bA_0\Delta\bB_1+\Delta\bA_1\Delta\bB_0}.
\end{align*}
Without loss of generality, take $\bar\bA = \bU_r\bm\Sigma_r^{1/2}$ and $\bar\bB = \bm\Sigma_r^{1/2}\bV_r^\top$. Note that any other $(\hat\bA, \hat\bB) \in S$ must satisfy
\begin{equation*}
    (\hat\bA, \hat\bB) = (\bar\bA\bQ, \bQ^{-1}\bar\bB) \triangleq \phi_\bQ(\bar\bA, \bar\bB)
\end{equation*}
for $\bQ \in \mathsf{GL}(r)$. From $\bar\bA\Delta\bB + \Delta\bA\bar\bB = \bm 0$, we can check that
\begin{equation*}
    (\Delta\bA, \Delta\bB) \in T_{(\bar\bA, \bar\bB)} S \quad \Leftrightarrow \quad \phi_\bQ(\Delta\bA, \Delta\bB) \in T_{\phi_\bQ(\bar\bA, \bar\bB)} S.
\end{equation*}
This implies
\begin{equation*}
    \nabla^2\cE(\bar\bA, \bar\bB)[(\Delta\bA_0, \Delta\bB_0), (\Delta\bA_1, \Delta\bB_1)] = \nabla^2\cE(\phi_\bQ(\bar\bA, \bar\bB))[\phi_\bQ(\Delta\bA_0, \Delta\bB_0), \phi_\bQ(\Delta\bA_1, \Delta\bB_1)],
\end{equation*}
and thus
\begin{equation*}
    (\Delta\bA, \Delta\bB) \in \mathrm{ker}\nabla^2\cE(\bar\bA, \bar\bB) \quad \Leftrightarrow \quad \phi_\bQ(\Delta\bA, \Delta\bB) \in \mathrm{ker}\nabla^2\cE(\hat\bA, \hat\bB).
\end{equation*}
Hence, it suffices to show $\mathrm{ker}\nabla^2\cE(\bar\bA, \bar\bB) = T_{(\bar\bA, \bar\bB)} S$ at a single point $(\bar\bA, \bar\bB) \in S$ to show that $\mathrm{ker}\nabla^2\cE(\bM, \bN) = T_{(\bM, \bN)} S$ for any $(\bM, \bN) \in S$. Moreover, as the critical set $S$ is an analytic manifold of the analytic $\cL$, we have $T_{(\bM, \bN)} S \subset \mathrm{ker}\nabla^2\cE(\bar\bA, \bar\bB)$~\citep[Remark 2.9]{feehan2020morse}, hence what remains is to show $\mathrm{ker}\nabla^2\cE(\bar\bA, \bar\bB) \subset T_{(\bM, \bN)} S$.

Now return to the specific point $\bar\bA = \bU_r\bm\Sigma_r^{1/2}$ and $\bar\bB = \bm\Sigma_r^{1/2}\bV_r^\top$. If
\begin{equation*}
    (\Delta\bA_0, \Delta\bB_0) \in \mathrm{ker}\nabla^2\cE(\bar\bA, \bar\bB),
\end{equation*}
then for all $(\Delta\bA_1, \Delta\bB_1)$, $\nabla^2\cE(\bar\bA, \bar\bB)[(\Delta\bA_0, \Delta\bB_0), (\Delta\bA_1, \Delta\bB_1)] = \bm 0$. Specifically,
\begin{align}
    \Delta\bA_1 = \bm 0 &\Rightarrow \bar\bA^\top(\Delta\bA_0\bar\bB + \bar\bA\Delta\bB_0) + \Delta\bA_0^\top(\bar\bA\bar\bB - \bY) = \bm 0,\label{eq:E_stationary_A1}\\
    \Delta\bB_1 = \bm 0 &\Rightarrow (\Delta\bA_0\bar\bB + \bar\bA\Delta\bB_0)\bar\bB^\top + (\bar\bA\bar\bB - \bY)\Delta\bB_0^\top = \bm 0.\label{eq:E_stationary_B1}
\end{align}

There exists $\bC_r, \bD_r \in \sR^{r\times r}$, $\bC_\perp \in \sR^{(m-r)\times r}$, and $\bD_\perp \in \sR^{r\times (n-r)}$ such that
\begin{equation*}
    \Delta\bA_0 = \bU_r\bC_r + \bU_\perp\bC_\perp, \quad \Delta\bB_0 = \bD_r\bV_r^\top + \bD_\perp\bV_\perp^\top.
\end{equation*}

Plugging into \Cref{eq:E_stationary_A1,eq:E_stationary_B1}, we obtain
\begin{align*}
    \bm\Sigma_r^{1/2}(\bC_r\bm\Sigma_r^{1/2} + \bm\Sigma_r^{1/2}\bD_r)\bV_r^\top + \bm\Sigma_r\bD_\perp\bV_\perp^\top - \bC_\perp^\top\bm\Sigma_\perp\bV_\perp^\top = \bm 0,\\
    \bU_r(\bC_r\bm\Sigma_r^{1/2} + \bm\Sigma_r^{1/2}\bD_r)\bm\Sigma_r^{1/2} + \bU_\perp\bC_\perp\bm\Sigma_r - \bU_\perp\bm\Sigma_\perp\bD_\perp^\top = \bm 0,
\end{align*}
which reduces to 
\begin{align}
    \bC_r\bm\Sigma_r^{1/2} + \bm\Sigma_r^{1/2}\bD_r &= \bm 0,\label{eq:E_stationary_1}\\
    \bm\Sigma_r\bD_\perp - \bC_\perp^\top\bm\Sigma_\perp &= \bm 0,\label{eq:E_stationary_2}\\
    \bC_\perp\bm\Sigma_r - \bm\Sigma_\perp\bD_\perp^\top &= \bm 0.\label{eq:E_stationary_3}
\end{align}

Since no singular values of $\bm\Sigma_r$ equals with that of $\bm\Sigma_\perp$, \Cref{eq:E_stationary_2,eq:E_stationary_3} implies that $\bC_\perp\bm\Sigma_r^2 = \bm\Sigma_\perp^2 \bC_\perp$ thus $\bC_\perp = \bm 0$; $\bD_\perp = \bm 0$ also holds. Moreover, take $\bK = \bm\Sigma_r^{-1/2}\bC_r$; we see that
\begin{align*}
    \bar\bA\bK &= \bU_r\bC_r = \Delta\bA_0,\\
    \bK\bar\bB &= \bm\Sigma_r^{-1/2}\bC_r\bm\Sigma_r^{1/2}\bV_r^\top = \bm\Sigma_r^{-1/2}(-\bm\Sigma_r^{1/2}\bD_r)\bV_r^\top = -\Delta\bB_0.
\end{align*}

Therefore, $\mathrm{ker}\nabla^2\cE(\bar\bA, \bar\bB)$ is a subset of $T_{(\bar\bA, \bar\bB)} S$. This argument analogously holds at any zero of $\bR$, or any $(\bM, \bN) \in S$, so $\cE$ is Morse-Bott.

\paragraph{Finite-time convergence of $\cT$.} We further show that if SpecGF with $\cT$ is well-defined---the gradients of $\cL$ with respect to both $\bA$ and $\bB$ are never rank-deficient---and converges, then the convergence should happen at a finite time.

\begin{proposition}\label{prop:specgf_t_finite_conv}
    Assume that SpecGF with $\cT$ is analytic for all time $t\geq 0$. If SpecGF with $\cT$ converges, then it converges finitely to a critical point.
\end{proposition}

\begin{proof}
    Let
    \begin{equation*}
        \Delta(t) \triangleq \cL(\bA(t), \bB(t)) - \cL(\bA(\infty), \bB(\infty)) \triangleq \cL(\bA(t), \bB(t)) - \cL_\infty \geq 0.
    \end{equation*}
    For each (open) neighborhood $U$ of $(\bA(\infty), \bB(\infty))$, there exists a time $t_U \geq 0$ such that for all $t \geq t_U$, $(\bA(t), \bB(t)) \in U$. Apply \Cref{lem:Loja_inequ}: there are a neighborhood $W$ of $(\bA(\infty), \bB(\infty))$, constants $C > 0$ and $b \in [\frac12, 1)$ such that for all $x \in W$,
    \begin{equation*}
        |\cL(\bA(t), \bB(t)) - \cL_\infty|^b \leq C\|\nabla \cL(\bA(t), \bB(t))\|.
    \end{equation*}

    Observe that
    \begin{align*}
        \Delta'(t) &= \frac{\mathrm{d}}{\mathrm{d}t} \cL(\bA(t), \bB(t)) = -\roundb{\|\nabla_\bA \cL\|_* + \|\nabla_\bB \cL\|_*} \leq -\roundb{\|\nabla_\bA \cL\|_\mathrm{F} + \|\nabla_\bB \cL\|_\mathrm{F}}\\
        &\leq -\|\nabla \cL(\bA(t), \bB(t))\|_\mathrm{F}\\
        &\leq -\frac{1}{C}\roundb{\cL(\bA(t), \bB(t)) - \cL_\infty}^b = -\frac{1}{C}\Delta(t)^b.
    \end{align*}

    At any time $t \geq t_W$, if $\Delta(t) = 0$, done. Otherwise,
    \begin{align*}
        &\frac{\mathrm{d}}{\mathrm{d}t} \Delta(t)^{1-b} = (1-b)\Delta(t)^{-b}\Delta'(t) \leq -\frac{1-b}{C}\\
        \Rightarrow &\Delta(t)^{1-b} - \Delta(t_W)^{1-b} \leq -\frac{1-b}{C}(t-t_W).
    \end{align*}
    Taking $T = t_W + \frac{C}{1-b}\Delta(t_W)^{1-b}$ gives $\Delta(T) \leq 0$, hence $\Delta(T) = 0$.
\end{proof}

\subsection{SpecGF with \texorpdfstring{$\ell_2$}{l2} Regularization}\label{sec:L2reg}

With such regularization, the global convergence of SpecGF immediately follows. The setup of regularized SpecGF becomes
\begin{itemize}
    \item The target matrix $\bY \in \sR^{m\times n}$ yields a compact SVD $\bY = \bU \bm \Sigma \bV^\top$ where $\bm \Sigma = \Diag(\sigma_1, \sigma_2, \cdots, \sigma_{r^\star})$ with $\sigma_1 \geq \cdots \geq \sigma_{r^\star} > 0$.
    \item The loss function $\cL_\lambda: \sR^{m\times r} \times \sR^{r\times n} \to \sR$ for a fixed $\lambda > 0$ is defined as
    \begin{equation}\label{eq:reg_SpecGF_Tb}
        \cL_\lambda(\bA, \bB) \triangleq \frac{1}{2}\normF{\bA \bB - \bY}^2 + \frac{\lambda}{2}\normF{\bA}^2 + \frac{\lambda}{2}\normF{\bB}^2.
    \end{equation}
    We may abuse the notation and simply state $\cL_\lambda(t)$. $\cL_\lambda$ is also analytic in $\bA$ and $\bB$.
    \item The time derivative of $\bA$ and $\bB$ is governed by either $\cT$ or $\cT_\beta$ applied to the gradient:
    \begin{equation*}
        \dot{\bA}(t) = -\cT\roundb{\nabla_\bA \cL_\lambda(t)}, \quad \dot{\bB}(t) = -\cT\roundb{\nabla_\bB \cL_\lambda(t)}.
    \end{equation*}
\end{itemize}

The following statements are true, because the corresponding statements do not utilize the structure of the loss; only the analyticity of $\cT_\beta$ and monotonicity of the loss matter, if any of those two is required.
\begin{enumerate}
    \item \Cref{lem:SpecGF_T_dec,lem:SpecGF_Tb_dec} readily follows for the regularized SpecGF with $\cT_\beta$. Moreover,
    \begin{equation*}
        \normF{\bA(t)}^2 + \normF{\bB(t)}^2 \leq \frac{2}{\lambda}\cL_\lambda(t) \leq \frac{2}{\lambda}\cL_\lambda(0) = \frac{1}{\lambda}\normF{\bY}^2 + \normF{\bA(0)}^2 + \normF{\bB(0)}^2.
    \end{equation*}
    Thus, $\bA(t)$ and $\bB(t)$ are both bounded for all $t \geq 0$.
    \item \Cref{eq:F_dec} in \Cref{prop:infdiv_conv} also holds for $\cL_\lambda$ (with the same constant), implying that regularized SpecGF with $\cT_\beta$ should either converge or diverge to infinity. However, the latter cannot happen; hence, regularized SpecGF with $\cT_\beta$ always converges.
    \item \Cref{prop:fixed_point_avoid} is applicable for the regularized SpecGF with $\cT_\beta$. That is, regularized SpecGF with $\cT_\beta$ avoids strict saddle points almost surely as SpecGF with $\cT_\beta$ does.
    \item \Cref{thm:SpecGF_Tb_stable} readily follows for the regularized SpecGF with $\cT_\beta$. That is, every global minimum of regularized SpecGF with $\cT_\beta$ is Lyapunov stable.
\end{enumerate}
It remains to characterize the critical points of regularized SpecGF with $\cT_\beta$, especially when LoRA rank $r$ is smaller than the target rank $r^\star$. Although it is specified in \citet{udell2016generalized}, we reproduce their result.

\begin{proposition}[\citep{udell2016generalized}]\label{prop:reg_PCA} Set $r \leq r^\star$. Assume that all singular values of $\bY$ are distinct and $\lambda$ is strictly smaller than the smallest singular value of $\bY$.
    \begin{enumerate}[label={(\roman*)}, ref=\thetheorem.(\roman*), noitemsep]
        \item Let $\Omega \subset [r^\star]$ with $|\Omega| \leq r$. Every stationary point of \Cref{eq:reg_SpecGF_Tb} is of the form
        \begin{equation*}
            \bA = \bU_\Omega\roundb{\bm\Sigma_\Omega - \frac\lambda2 \bI}^{\frac{1}{2}}, \quad \bB = \roundb{\bm\Sigma_\Omega - \frac\lambda2 \bI}^{\frac{1}{2}}\bV_\Omega^\top.
        \end{equation*}
        up to orthogonal transformation by $\bR \in \mathsf{O}(r)$. We denote the submatrix of $\bU$ with columns indexed by $\Omega$ as $\bU_\Omega$, and similarly for $\bm\Sigma$ and $\bV$. \label{prop:reg_PCA_critical}
        \item Every critical point except global minima has a descent direction. \label{prop:SSP_characterization_reg}
    \end{enumerate}
\end{proposition}

\Cref{prop:SSP_characterization_reg} implies that every local minimum is a \textit{balanced} global minimum, and every saddle point is a strict saddle point. Therefore, the a.s. global convergence of regularized SpecGF with $\cT_\beta$ towards global minima is guaranteed.

\begin{proof}[Proof of \Cref{prop:reg_PCA}]
    We omit the proof of \Cref{prop:reg_PCA_critical}; readers may find the proof in \citet[Section 2.3]{udell2016generalized}. Though, we note that at every critical point, $\bA^\top \bA = \bB\bB^\top$ should hold from the first-order optimality condition.
    
    Every critical point, indexed by $\Omega$, should be of the form
    \begin{equation*}
        \bA\bB = \sum_{i\in\Omega} d_i\bu_i\bv_i^\top,
    \end{equation*}
    where $d_i = \sigma_i - \lambda > 0$, $\bu_i$ and $\bv_i$ are the $i$-th column of $\bU$ and $\bV$, respectively. We use the representative element from each orbit described above, so each column of $\bA$ is $\bu_i\sqrt{d_i}$ and each row of $\bB$ is $\sqrt{d_i}\bv_i^\top$.
    
    If a critical point is not the global minimum, then $\sigma_j > \sigma_i$ for some $i \in \Omega$ and $j \not\in \Omega$. With small $\epsilon > 0$, form $\Tilde\bA$ by replacing the column $\bu_i\sqrt{d_i}$ by $(\bu_i + \epsilon \bu_j)\sqrt{d_i}$; form $\Tilde\bB$ by replacing the column $\sqrt{d_i}\bv_i^\top$ by $\sqrt{d_i}(\bv_i + \epsilon \bv_j)^\top$. One can directly compute
    \begin{equation*}
        \lambda\squarb{\roundb{\|\Tilde\bA\|_\mathrm{F}^2 + \|\Tilde\bB\|_\mathrm{F}^2} - \roundb{\normF{\bA}^2 + \normF{\bB}^2}} = 2\lambda d_i\epsilon^2
    \end{equation*}
    and
    \begin{equation*}
        \normF{\Tilde\bA \Tilde\bB - \bY}^2 - \normF{\bA \bB - \bY}^2 = 2\epsilon^2 d_i(d_i-\sigma_j) + \epsilon^4 d_i^2.
    \end{equation*}
    Hence, the net change in the objective is
    \begin{equation*}
        2\lambda d_i\epsilon^2 + 2\epsilon^2 d_i(d_i-\sigma_j) + \epsilon^4 d_i^2 = 2\epsilon^2 d_i(\sigma_i-\sigma_j) + \epsilon^4 d_i^2,
    \end{equation*}
    which can be made (strictly) negative for small $\epsilon$.
\end{proof}
\newpage

\section{Experiments}\label{sec:add_expm}

\subsection{Matrix Factorization Experiments Under Various Settings}~\label{sec:theory_expm}

In this section, we present additional experiments that vary the training setting of \Cref{fig:toy_observation}. In \Cref{fig:specgd_varying_mu}, we vary the momentum parameter $\mu$ in \Cref{eq:muon_update}, where $\mu = 0$ corresponds to SpecGD, while applying exact orthogonalization via $\cT$ at every update.
We observe that the singular values of the product matrix exhibit consistent uniform growth across all choices of $\mu$, indicating that this behavior is robust to the strength of momentum under exact orthogonalization.

In \Cref{fig:muon_varying_mu}, we consider the same range of momentum parameters but replace the exact orthogonalization operator with Newton–Schulz iterations; this setting is identical to the practical Muon optimizer used in real training scenarios. Despite the use of approximate orthogonalization, we again observe uniform growth of singular values, demonstrating that the phenomenon persists under practical implementations of Muon.

In \Cref{fig:specgd_varying_r}, we vary the LoRA rank from $1$ to $4$, which is strictly smaller than the target rank $5 = \mathrm{rank}(\bSigma)$. In all cases, the SpecGD iteration converges, reaching plateaus at the top-$r$ singular values of $\bSigma$, and the resulting solution corresponds to the best rank-$r$ approximation of $\bY$.

Finally, we increase the number of LoRA factors (depth), by reparameterizing the weight $\bW$ as $\bW_L \bW_{L-1} \cdots \bW_1$. All intermediate factors $\bW_l$ lie in $\sR^{r\times r}$, with the first factor $\bW_L$ in $\sR^{m\times r}$, and the last factor $\bW_1$ in $\sR^{r\times n}$. Only the first factor $\bW_L$ is initialized at zero, while the remaining factors are initialized as i.i.d. Gaussian matrices scaled by $\gamma$. In \Cref{fig:specgd_varying_depth}, we observe that the uniform growth of singular values persists as depth increases, and \Cref{fig:specgd_varying_depth_one_over_n} further reveals that the growth rate scales as $\Theta(t^L)$ for depth $L$.

\begin{figure}[ht]
    \centering
    \includegraphics[width=1\linewidth]{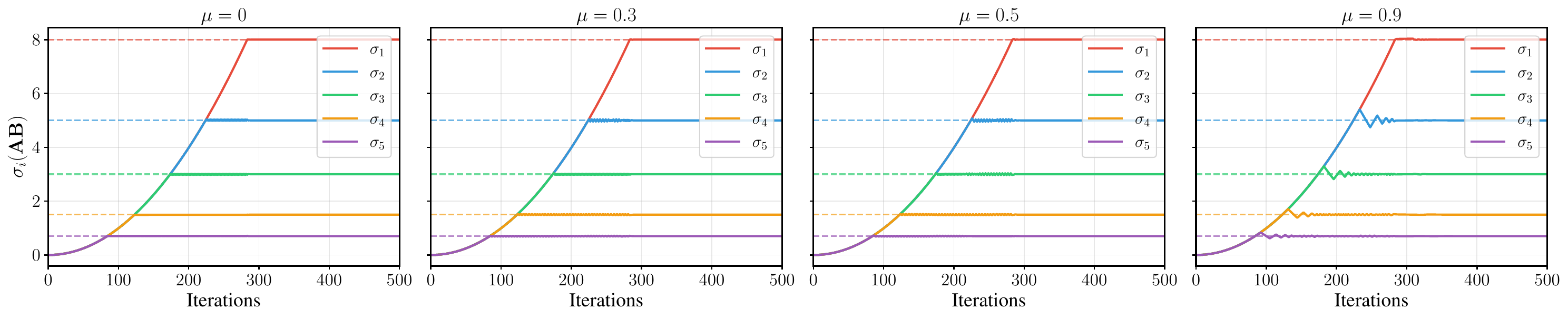}
    \caption{Comparison of singular value evolutions across $\mu \in\{0,0.3,0.5,0.9\}$ from exact orthogonalization $\cT$. Notice the consistent uniform growth.}
    \label{fig:specgd_varying_mu}
\end{figure}

\begin{figure}[ht]
    \centering
    \includegraphics[width=1\linewidth]{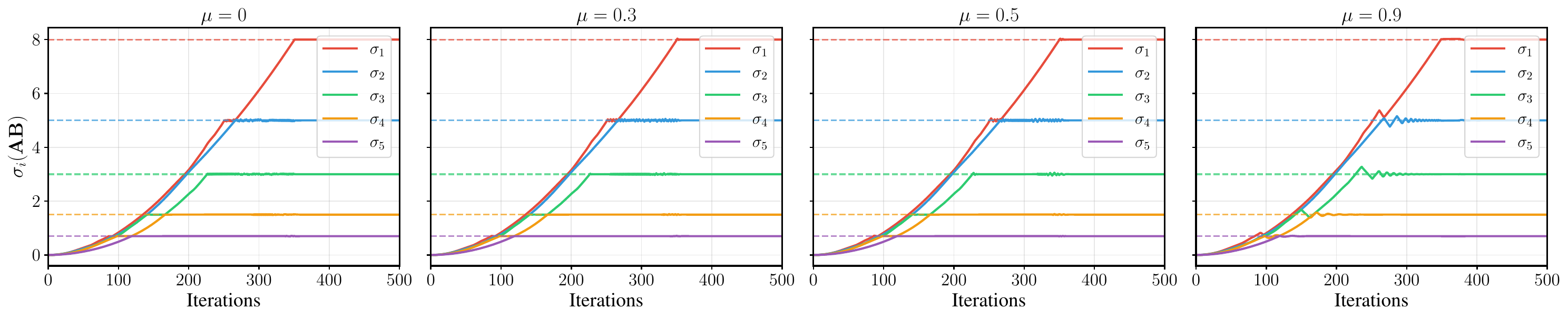}
    \caption{Comparison of singular value evolutions across $\mu \in\{0,0.3,0.5,0.9\}$ from Newton-Schulz iterations. Notice the consistent uniform growth.}
    \label{fig:muon_varying_mu}
\end{figure}

\begin{figure}[ht]
    \centering
    \includegraphics[width=1\linewidth]{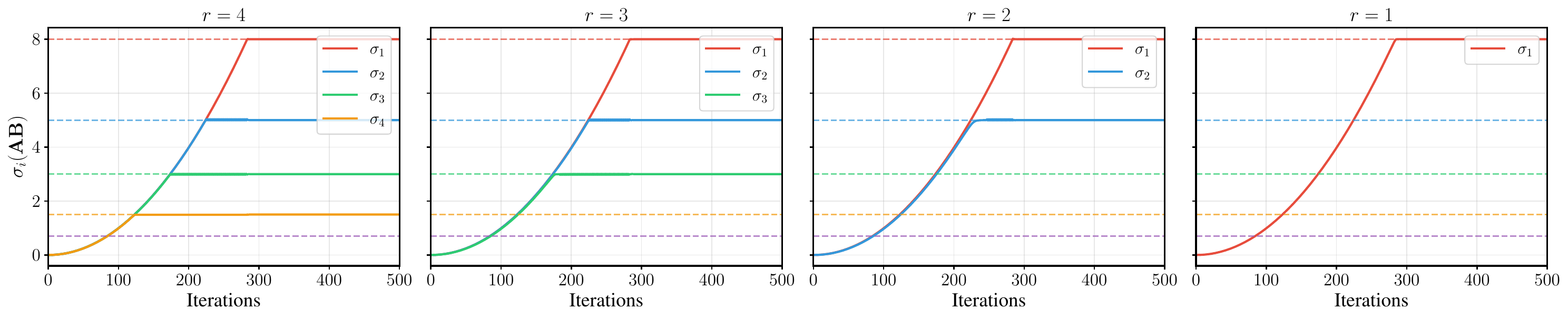}
    \caption{Comparison of singular value evolutions across LoRA rank $r \in \{1,2,3,4\}$. Notice that SpecGF converges toward best rank-$r$ approximation.}
    \label{fig:specgd_varying_r}
\end{figure}

\begin{figure}[ht]
    \centering
    \includegraphics[width=1\linewidth]{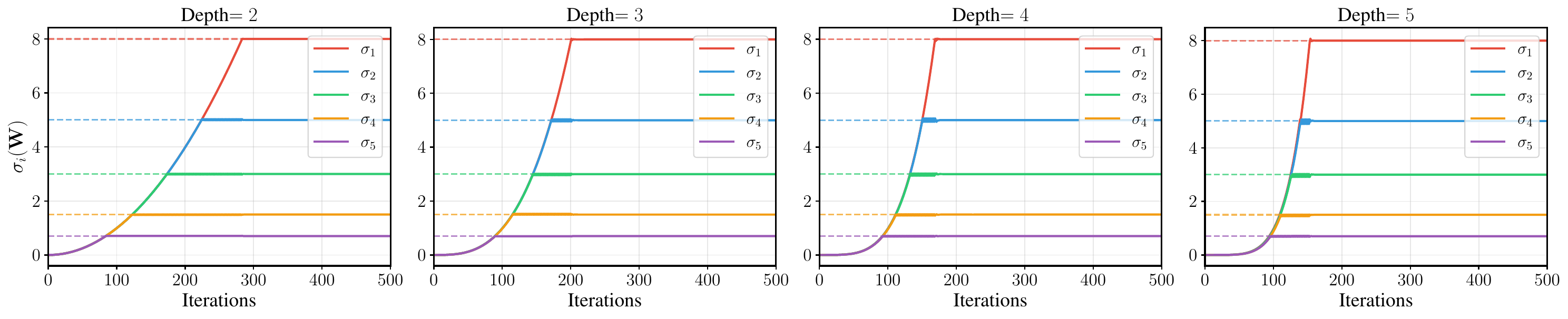}
    \caption{Comparison of singular value evolutions across the depth $L \in \{2,3,4,5\}$. Notice the consistent uniform growth.}
    \label{fig:specgd_varying_depth}
\end{figure}

\begin{figure}[H]
    \centering
    \includegraphics[width=1\linewidth]{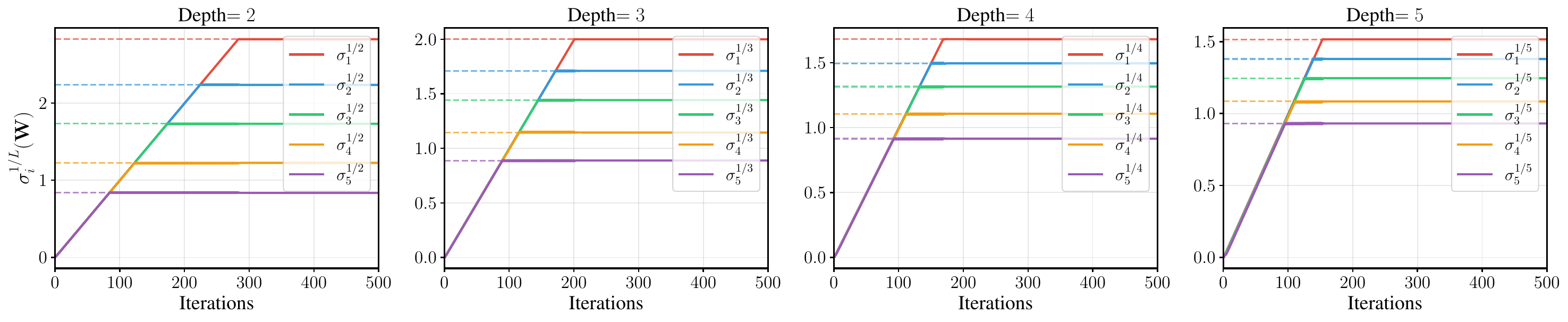}
    \caption{Comparison of the $L$-th root of the singular value evolutions across the depth $L \in \{2,3,4,5\}$. Notice the consistent uniform \emph{linear} growth.}
    \label{fig:specgd_varying_depth_one_over_n}
\end{figure}

\subsection{LLM Experiments}~\label{sec:LLM_expm}
\subsubsection{Experimental Details}
We fine-tuned a pre-trained RoBERTa-base model (125M)~\citep{liu2019robertarobustlyoptimizedbert} on the SST-2 dataset~\citep{wang2018glue} using LoRA adapters applied to the query and key matrices. We set the LoRA rank to $r=8$ and the scaling factor to $\alpha=16$.  We trained the model for 10 epochs with a batch size of 16 using AdamW and Muon optimizers, with a weight decay of 0.01. The learning rate was set to $10^{-4}$ and scheduled using a cosine decay. For larger scale experiments, we train LLaMA-3.2-1B~\cite{grattafiori2024llama} on Alpaca dataset~\citep{taori2023stanford} for supervised fine-tuning. As in RoBERTa-base, we also use the same LoRA configuration, i.e., rank $r = 8$ and the scaling factor $\alpha = 16$. We train $3$ epochs with a batch size $128$ using AdamW and Muon optimizers, with weight decay 0.01. We use cosine learning rate scheduling with $4\%$ warmup steps of total training iterations. 

\subsubsection{Singular Value Dynamics}
We provide the full training dynamics of the singular values of query $Q$ and value $V$ matrices where LoRA is applied. Each line in each subplot stands for a single singular value.

\Cref{fig:roberta_all_sval_v1,fig:roberta_all_sval_v2} show the singular value evolutions for all LoRA $\bA\bB$ layers in RoBERTa-Base trained under Muon and AdamW, respectively; \Cref{fig:llama_all_sval_v1,fig:llama_all_sval_v2} show those in LLaMA-3.2-1B trained under Muon and AdamW, respectively.

The singular values of all LoRA $\bA\bB$ layers trained under Muon~(\Cref{fig:llama_all_sval_v1,fig:roberta_all_sval_v1}) consistently and nearly uniformly evolve. On the other hand, those of all LoRA $\bA\bB$ layers trained under AdamW~(\Cref{fig:llama_all_sval_v2,fig:roberta_all_sval_v2}) vary from matrix to matrix.

\begin{figure*}[hbtp!]
	\centering
	\begin{subfigure}[b]{0.49\textwidth}
        \includegraphics[width=\linewidth]{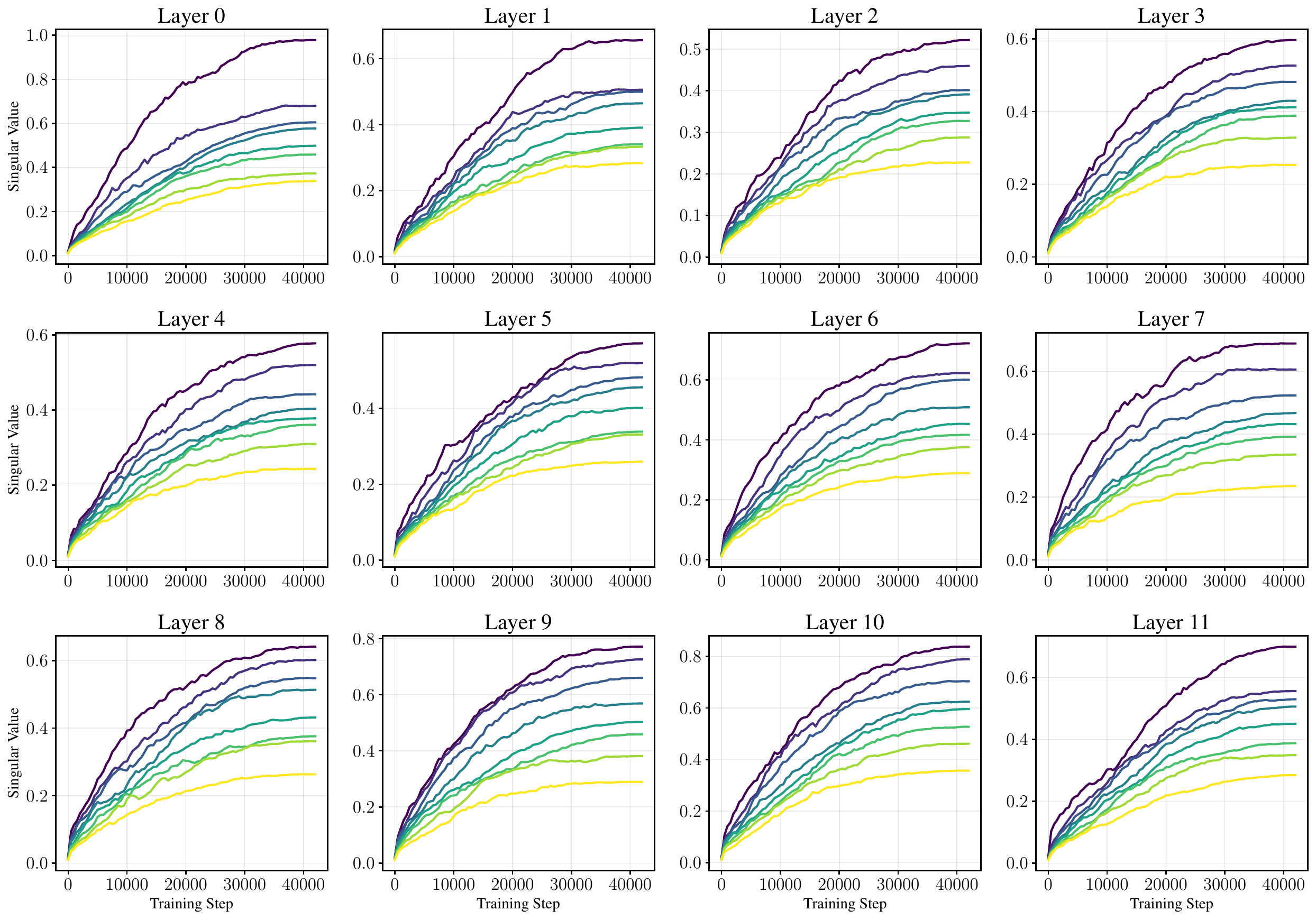}
        \caption{$Q$ matrix}
	\end{subfigure}
    \begin{subfigure}[b]{0.49\textwidth}
        \includegraphics[width=\linewidth]{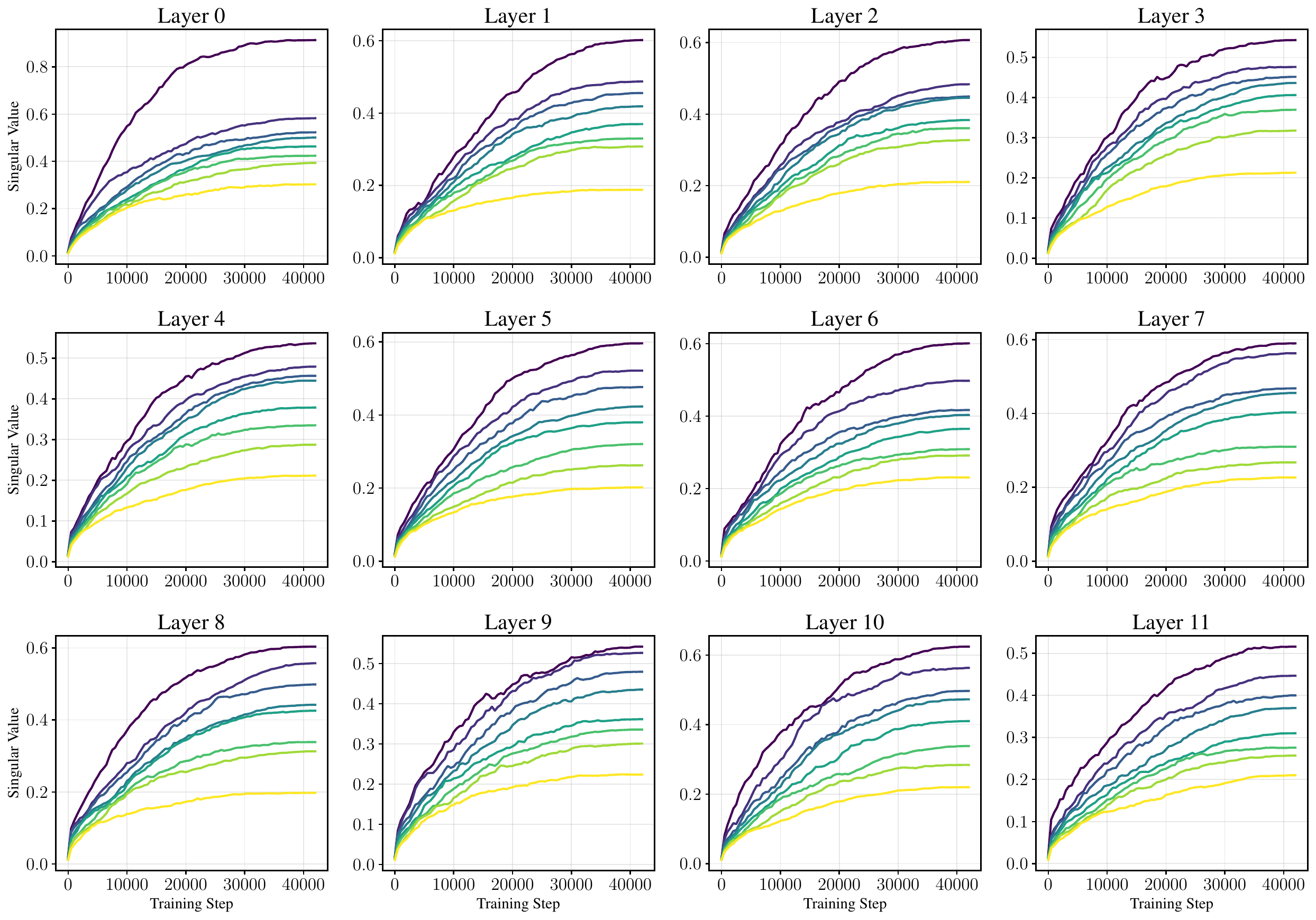}
        \caption{$V$ matrix}
	\end{subfigure}
    \caption{Singular values for all LoRA $\bA\bB$ layers in RoBERTa-Base trained with Muon.}
	\label{fig:roberta_all_sval_v1}
\end{figure*}

\begin{figure*}[hbtp!]
	\centering
	\begin{subfigure}[b]{0.49\textwidth}
        \includegraphics[width=\linewidth]{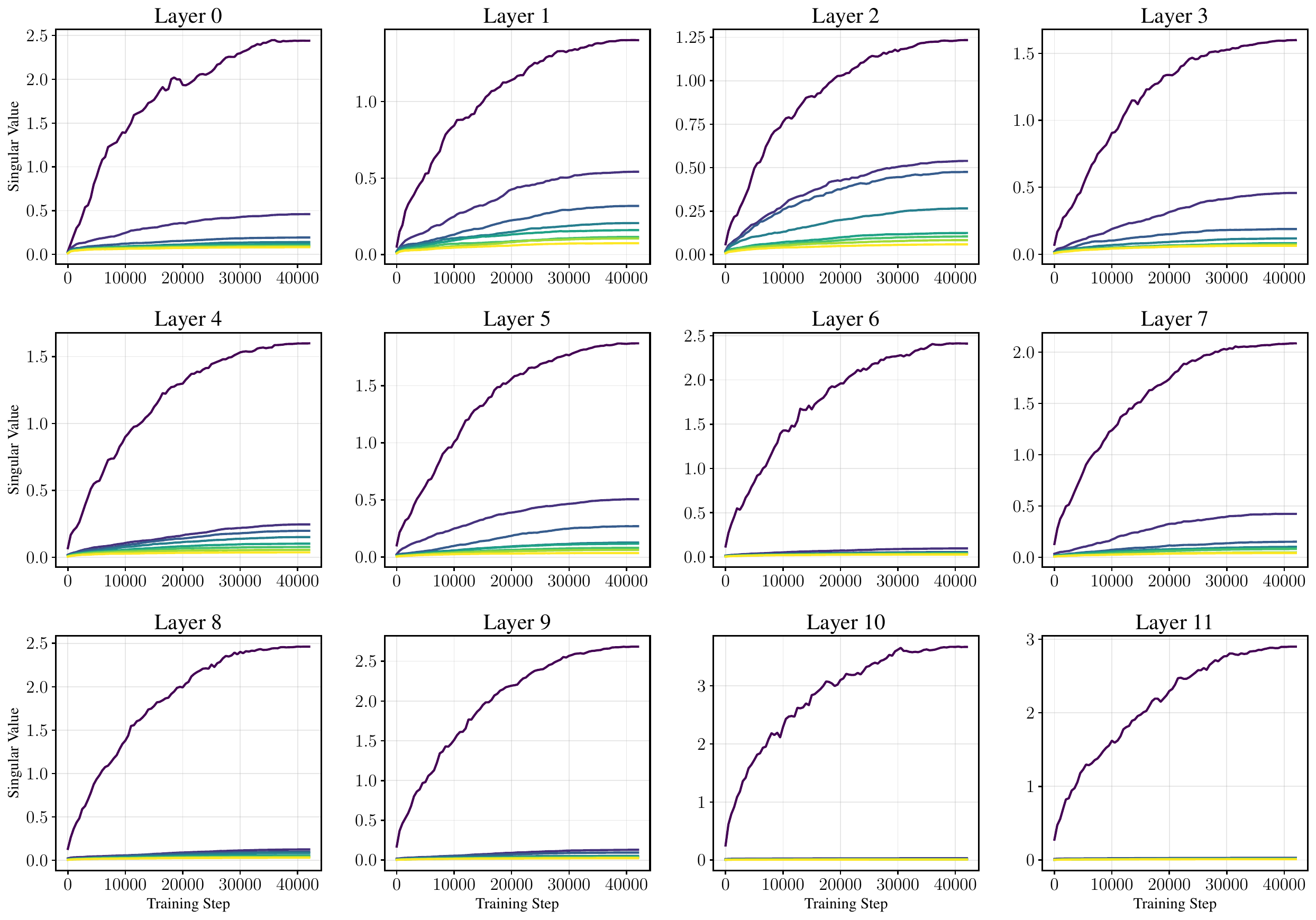}
        \caption{$Q$ matrix}
	\end{subfigure}
    \begin{subfigure}[b]{0.49\textwidth}
        \includegraphics[width=\linewidth]{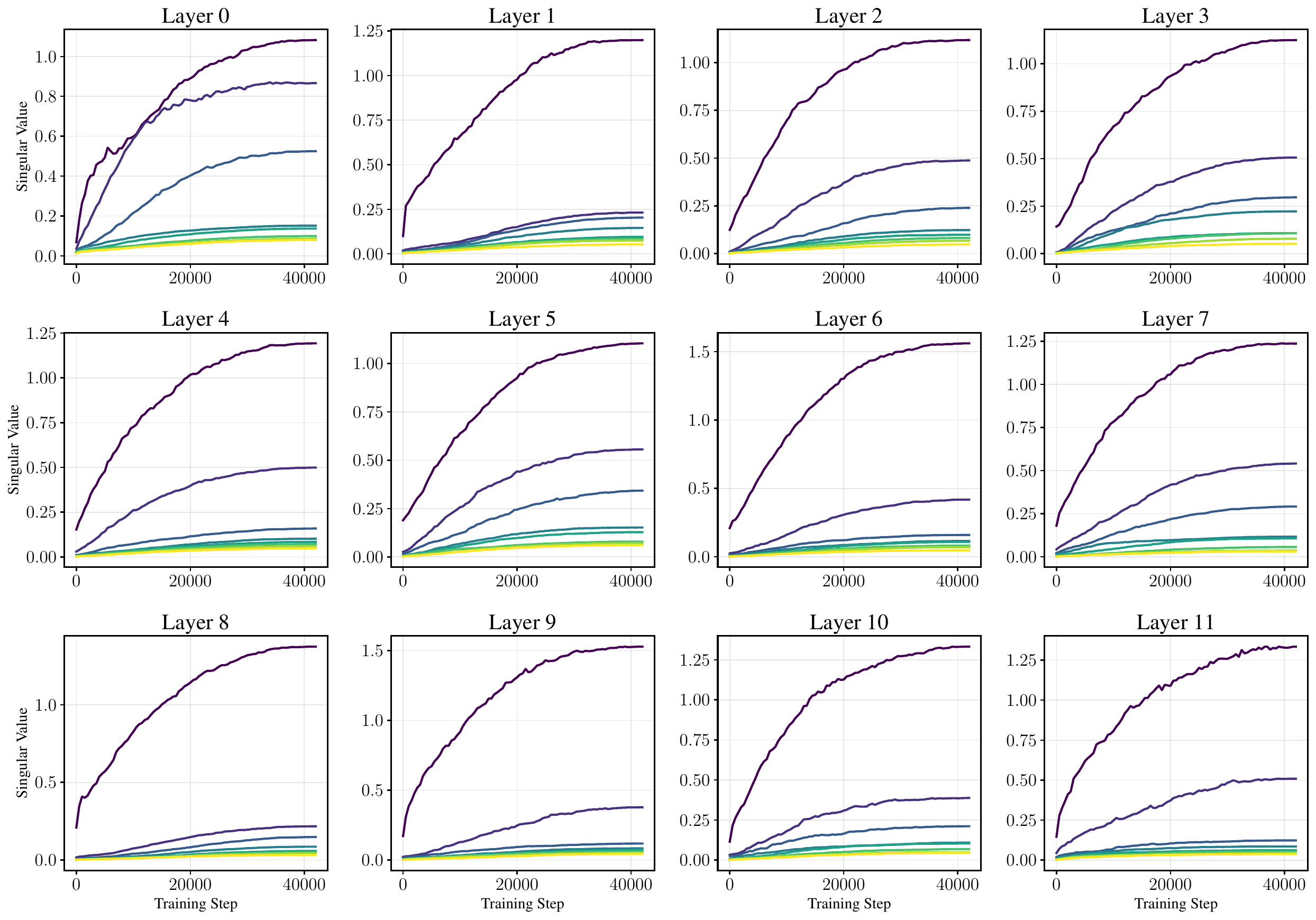}
        \caption{$V$ matrix}
	\end{subfigure}
    \caption{Singular values for all LoRA $\bA\bB$ layers in RoBERTa-Base trained with AdamW.}
	\label{fig:roberta_all_sval_v2}
\end{figure*}

\begin{figure*}[hbtp]
	\centering
	\begin{subfigure}[b]{0.49\textwidth}
        \includegraphics[width=\linewidth]{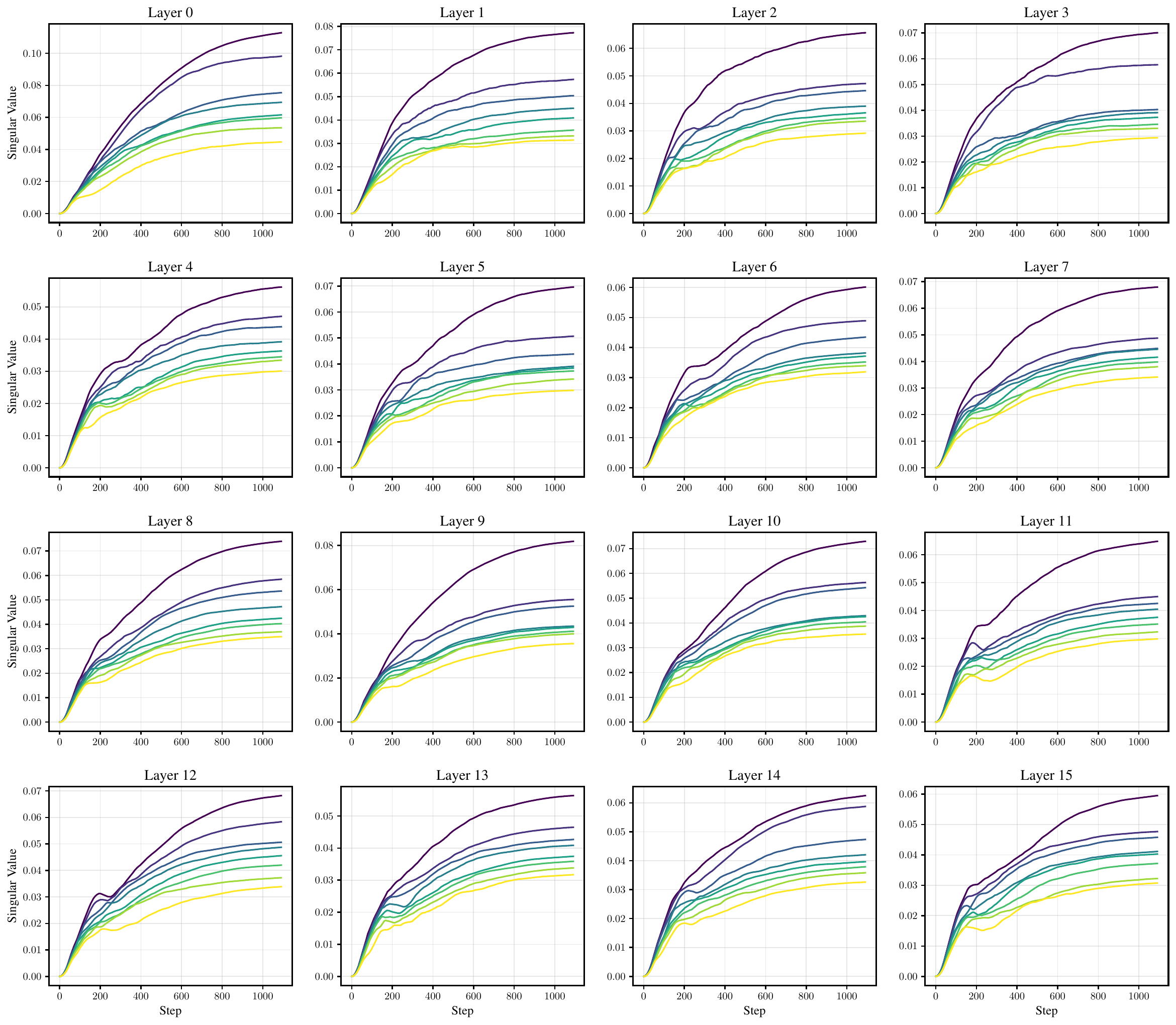}
        \caption{$Q$ matrix}
	\end{subfigure}
    \begin{subfigure}[b]{0.49\textwidth}
        \includegraphics[width=\linewidth]{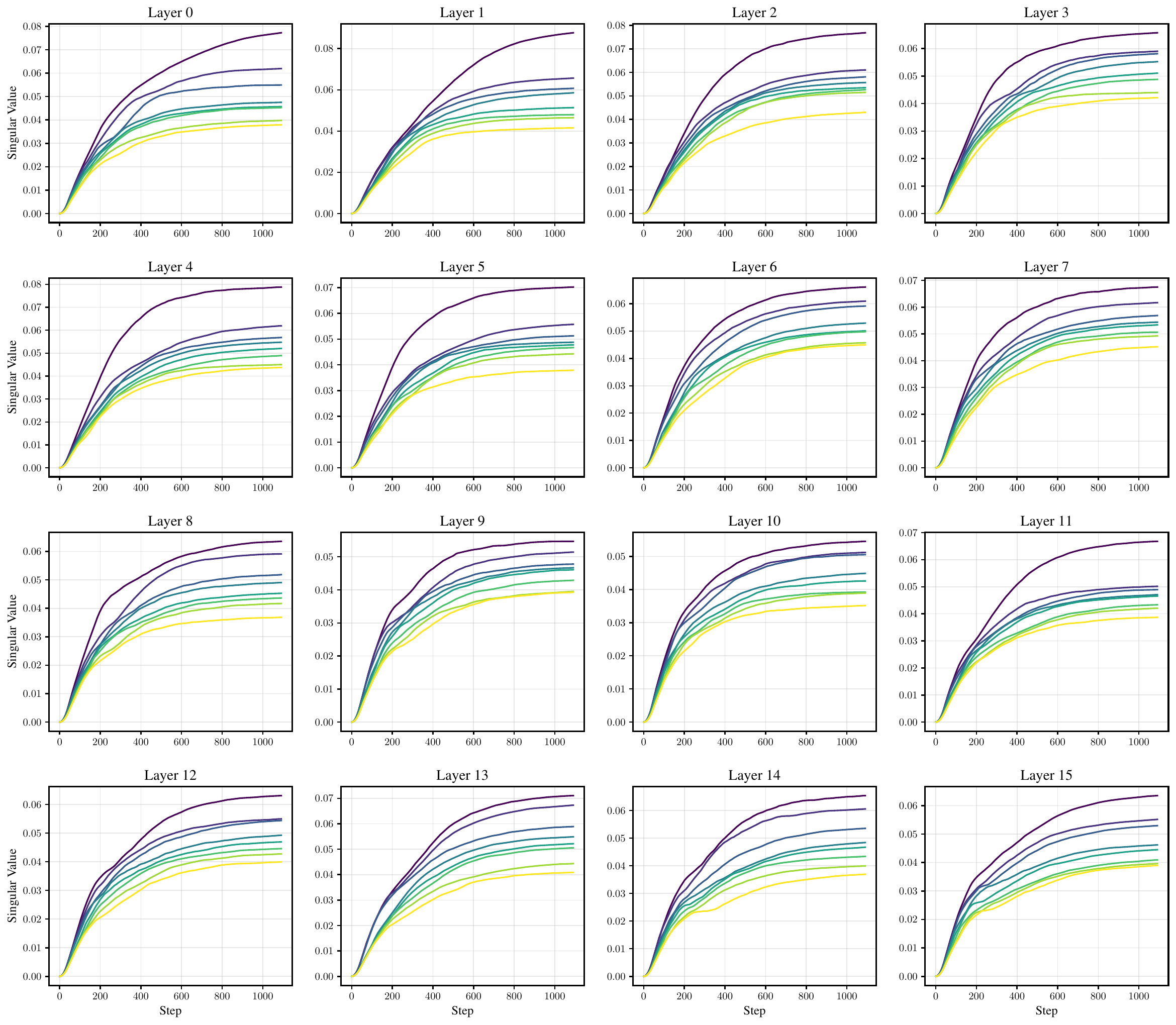}
        \caption{$V$ matrix}
	\end{subfigure}
    \caption{Singular values for all LoRA $\bA\bB$ layers in LLaMA-3.2-1B trained with Muon.}
	\label{fig:llama_all_sval_v1}
\end{figure*}

\begin{figure*}[hbtp!]
	\centering
	\begin{subfigure}[b]{0.49\textwidth}
        \includegraphics[width=\linewidth]{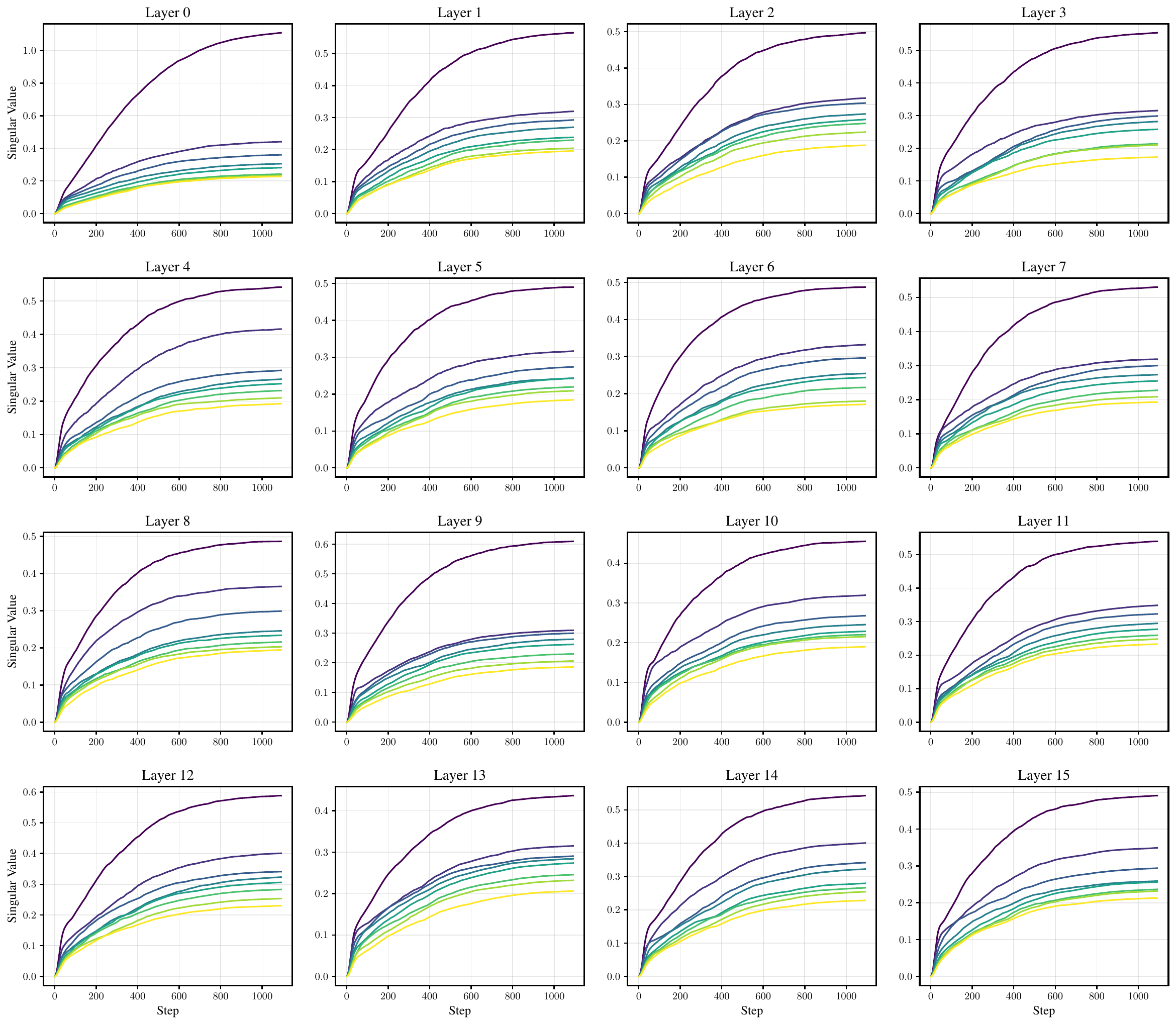}
        \caption{$Q$ matrix}
	\end{subfigure}
    \begin{subfigure}[b]{0.49\textwidth}
        \includegraphics[width=\linewidth]{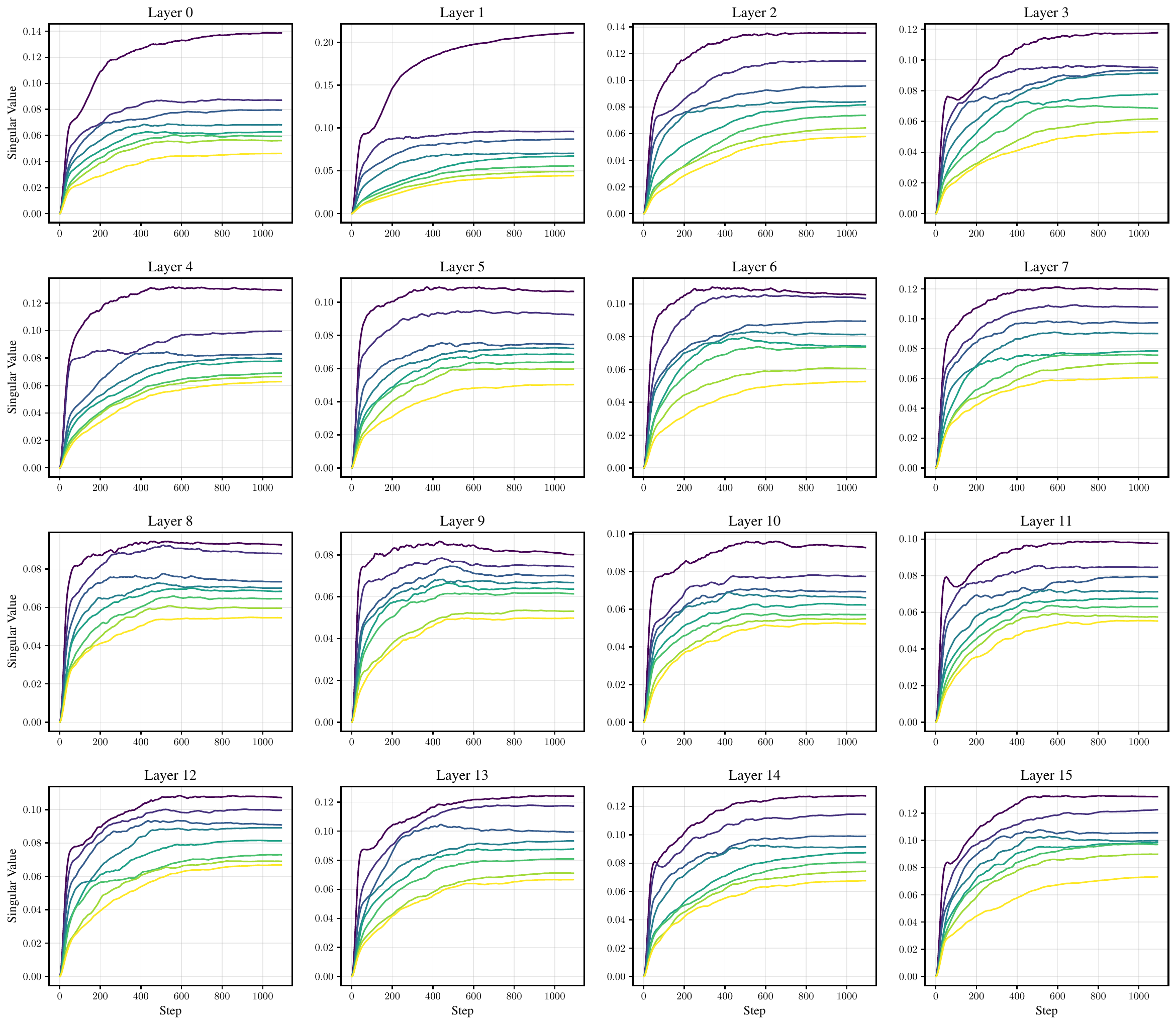}
        \caption{$V$ matrix}
	\end{subfigure}
    \caption{Singular values for all LoRA $\bA\bB$ layers in LLaMA-3.2-1B trained with AdamW.}
	\label{fig:llama_all_sval_v2}
\end{figure*}

\clearpage

\subsubsection{Effective Rank}

We provide the evolution of the effective rank of query $Q$ and value $V$ matrices where LoRA is applied. Blue and red lines are the effective rank of the matrix trained under Muon and Adam W, respectively. 

\Cref{fig:roberta_all_erank} shows the effective rank evolutions for all LoRA $\bA\bB$ layers in LLaMA-3.2-1B; \Cref{fig:llama_all_sval_v4} shows those in RoBERTa-Base. The effective rank of matrices trained under Muon maintains a stable value, close to the LoRA rank $r=8$. That of matrices trained under AdamW exhibits non-trivial evolution.

\begin{figure*}[hbtp]
	\centering
	\begin{subfigure}[b]{0.48\textwidth}
        \includegraphics[width=\linewidth]{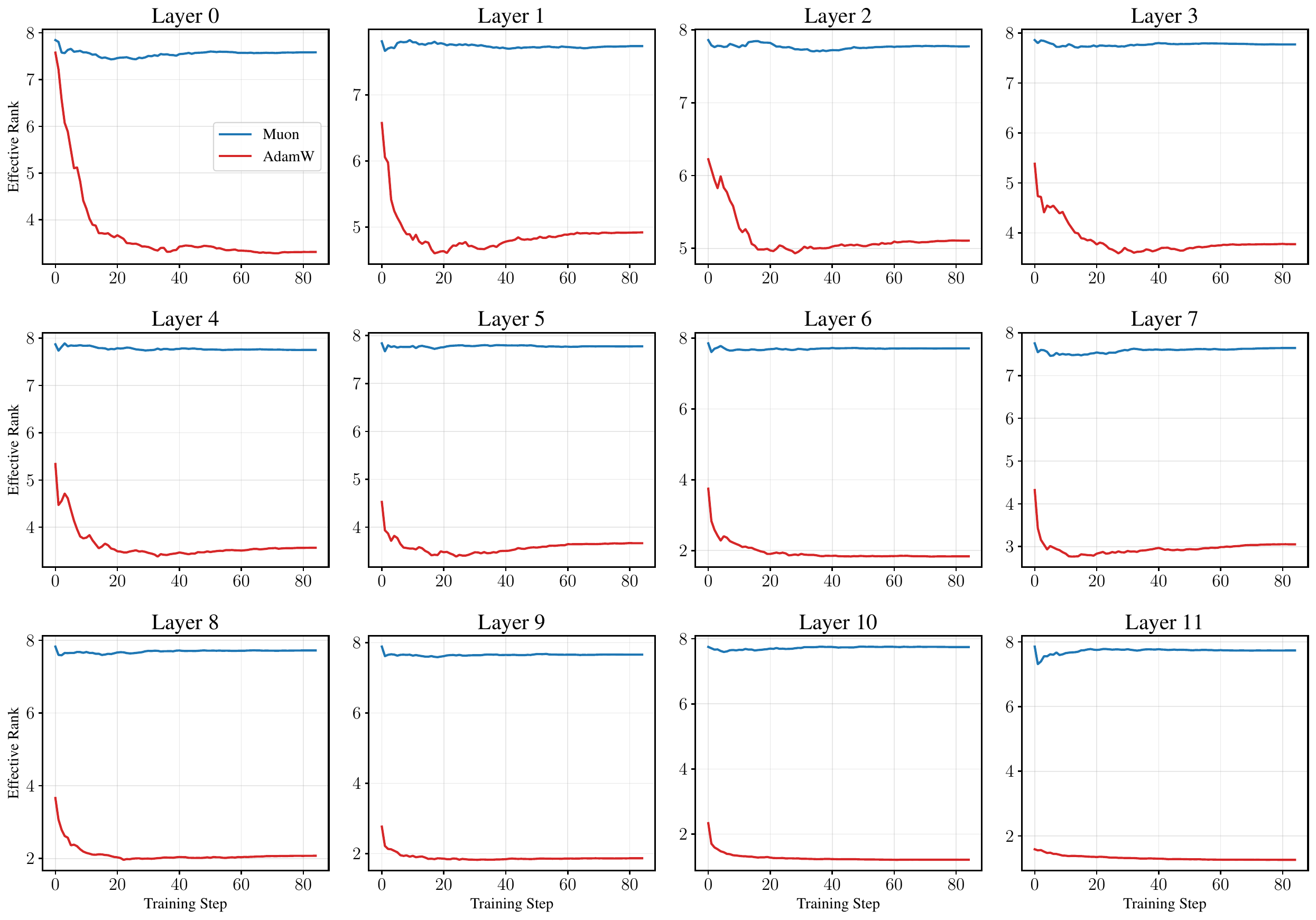}
        \caption{$Q$ matrix}
	\end{subfigure}
    \begin{subfigure}[b]{0.48\textwidth}
        \includegraphics[width=\linewidth]{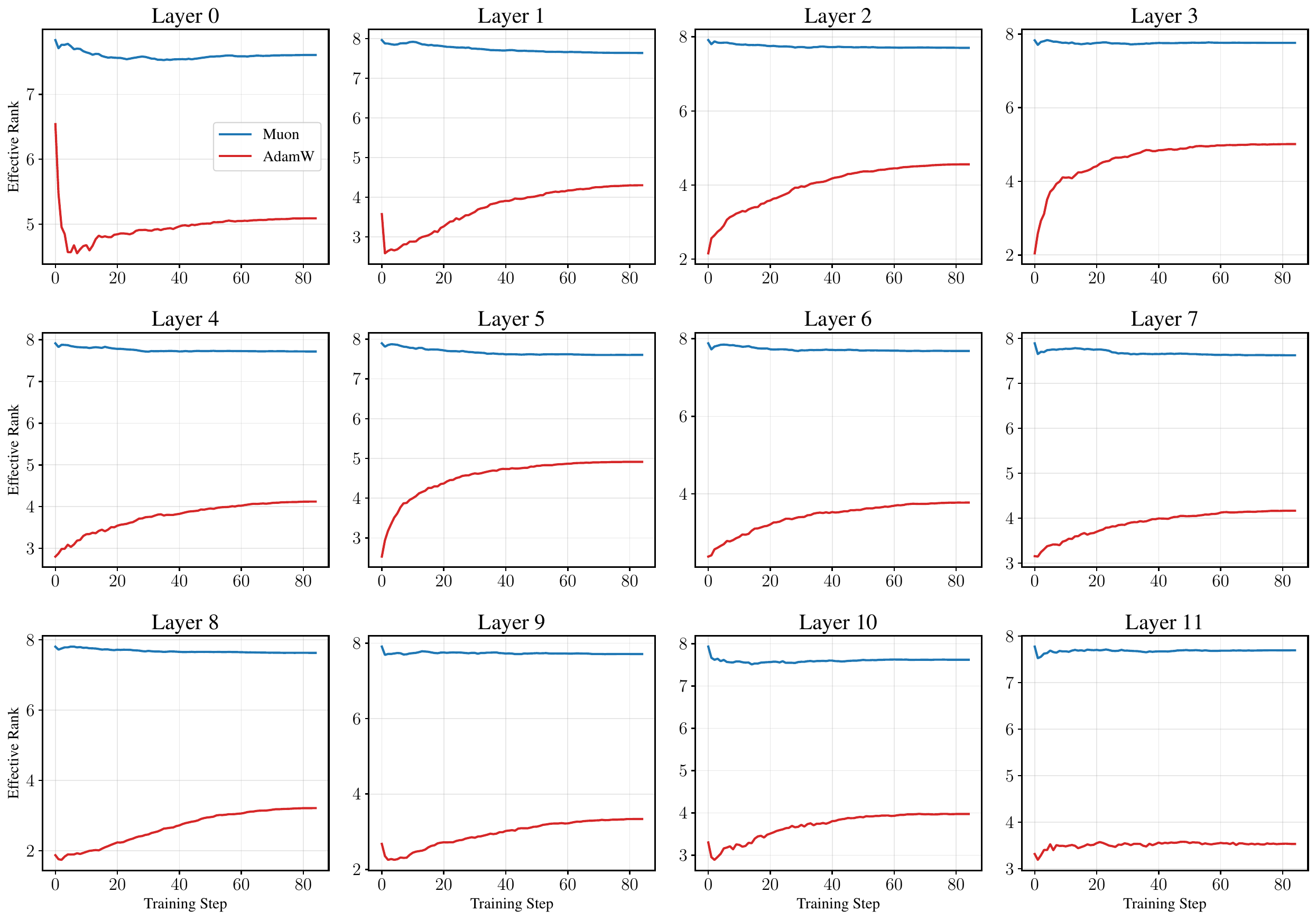}
        \caption{$V$ matrix}
	\end{subfigure}
    \caption{Effective rank of the LoRA $\bA\bB$ layers for the $Q$ and $V$ matrices in RoBERTa-Base.}
    \label{fig:roberta_all_erank}
\end{figure*}

\begin{figure*}[hbtp!]
	\centering
	\begin{subfigure}[b]{0.48\textwidth}
        \includegraphics[width=\linewidth]{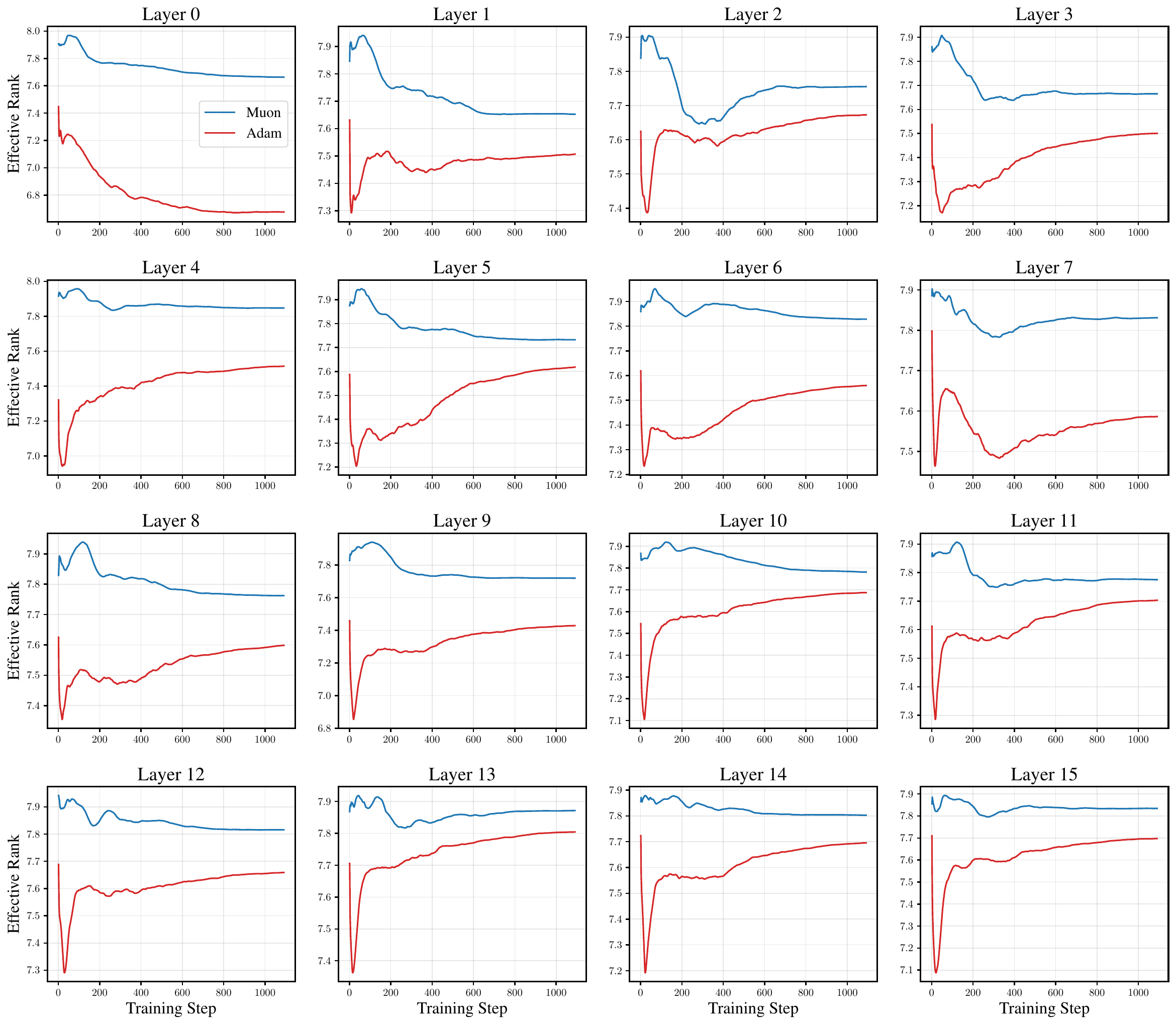}
        \caption{$Q$ matrix}
	\end{subfigure}
    \begin{subfigure}[b]{0.48\textwidth}
        \includegraphics[width=\linewidth]{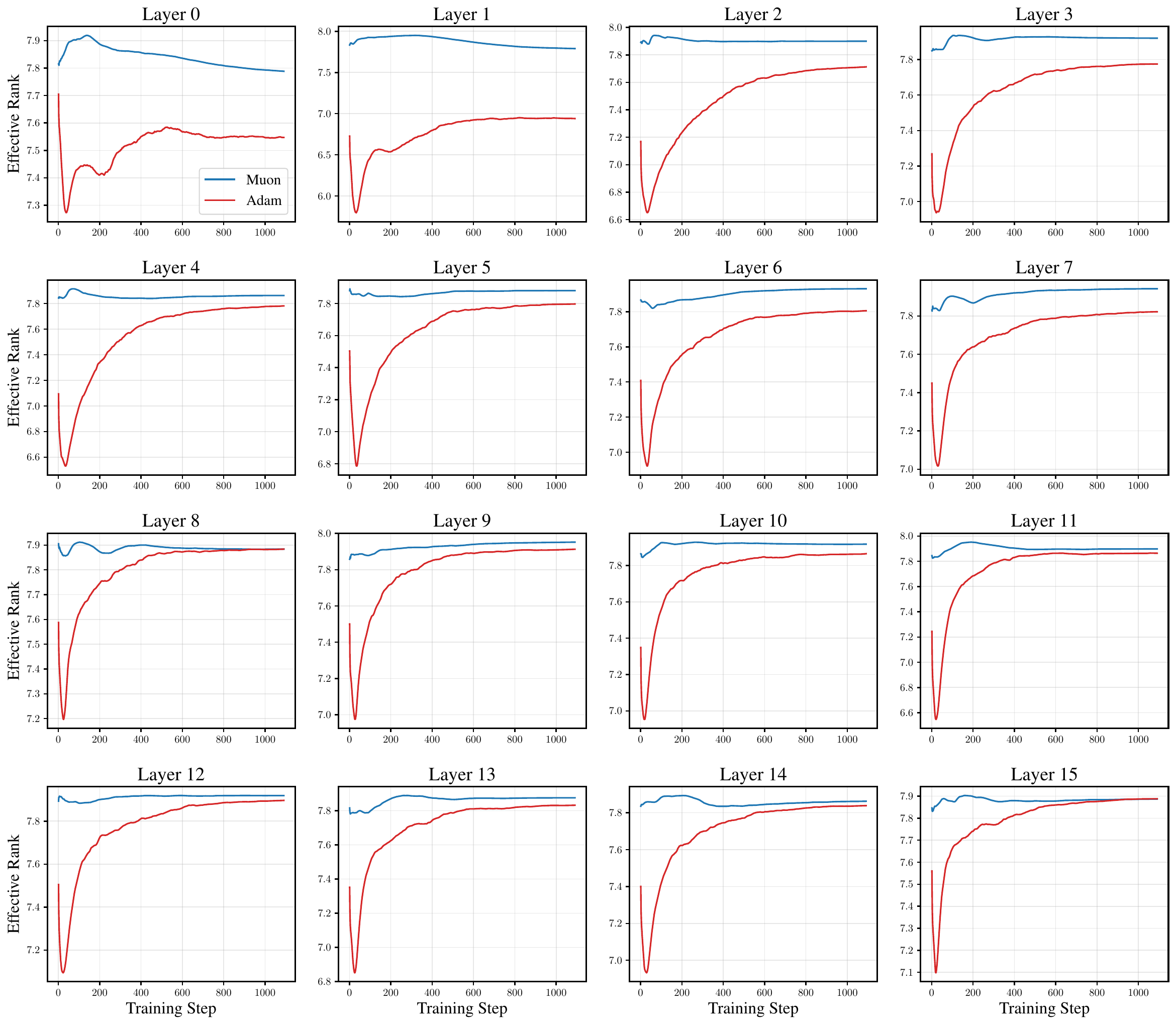}
        \caption{$V$ matrix}
	\end{subfigure}
    \caption{Effective rank of the LoRA $\bA\bB$ matrices for the $Q$ and $V$ in LLaMA-3.2-1B.}
	\label{fig:llama_all_sval_v4}
\end{figure*}


\end{document}